\documentclass{article} 

\usepackage{amsmath,amsfonts,bm}

\def\eqref#1{equation~\ref{#1}}

\def\1{\bm{1}}

\DeclareMathAlphabet{\mathsfit}{\encodingdefault}{\sfdefault}{m}{sl}
\SetMathAlphabet{\mathsfit}{bold}{\encodingdefault}{\sfdefault}{bx}{n}

\usepackage{hyperref}
\usepackage{url}
\usepackage{tikz} 
\usepackage{bbm} 
\usepackage{algorithm}
\usepackage{algpseudocode}
\usepackage{fullpage} 
\usepackage{natbib} 

\algblock{ParFor}{EndParFor}
\algnewcommand\algorithmicparfor{\textbf{parfor}}
\algnewcommand\algorithmicpardo{\textbf{do}}
\algnewcommand\algorithmicendparfor{\textbf{end\ parfor}}
\algrenewtext{ParFor}[1]{\algorithmicparfor\ #1\ \algorithmicpardo}
\algrenewtext{EndParFor}{\algorithmicendparfor}

\newcommand{\pp}[1]{\vspace{6pt}\noindent\textbf{#1}} 

\usepackage{mathtools} 
\usepackage{cleveref} 
\usepackage{nicefrac} 
\usepackage{amsmath, amssymb, amsthm}
\usepackage{wrapfig} 
\usepackage{subcaption} 
\usepackage{stmaryrd}
\usepackage{enumitem} 

\newcommand{\gfn}{\mathfrak{g}}

\newtheorem{theorem}{Theorem}[section] 
\newtheorem{proposition}[theorem]{Proposition}
\newtheorem{corollary}[theorem]{Corollary}
\newtheorem{lemma}[theorem]{Lemma} 

\theoremstyle{definition} 

\newtheorem{challenge}[theorem]{Challenge}
\newtheorem{definition}[theorem]{Definition}
\newtheorem{example}[theorem]{Example} 
\newtheorem{remark}[theorem]{Remark} 
\newtheorem{assumption}[theorem]{Assumption}

\title{Learning a Mixture of GFlowNets}

\author{Tiago da Silva \\
MBZUAI, UAE \\
\texttt{tiago.dasilva@mbzuai.ac.ae} \\
\and 
Amauri H. Souza \\
Avra, Brazil \& Federal Institute of Ceará, Brazil \\
\texttt{amauri.holanda@avra.ai} 
\and 
Salem Lahlou \\
MBZUAI, UAE \\
\texttt{salem.lahlou@mbzuai.ac.ae}
}

\begin{document}

\maketitle

\begin{abstract}
    Learning an ensemble of GFlowNets to sample from a discrete target distribution has become a common approach for achieving better state space exploration and convergence than that of a monolithic sampler.
    However, these methods often add a substantial runtime overhead to the base model, and their conceptual connection remains elusive.
    To address this, we first propose a general-purpose theoretical framework for describing a mixture of GFlowNets, which we specialize into continuously (CI) and discretely indexed (DI) collections.
    On the one hand, we show CI GFlowNets can be interpreted through the lens of a \emph{random features expansion}, provably boosting the sampler's expressivity in graph-structured tasks and reducing learning instability via spectral shifting. 
    On the other hand, we demonstrate DI GFlowNets encompass prior approaches for GFlowNet training and provide the foundation for the newly proposed \emph{Stratum-Conditioned} (SC) GFlowNets.  
    This method, which is inspired by the Doob's $h$-transform of Markov chains, decomposes the state space according to a prescribed modular function and restricts each component to sample from a distinct subset of it. 
    Importantly, SC GFlowNets support centralized and component-wise embarrassingly parallel training, 
    and we show both of them significantly speed up learning convergence and mode coverage without introducing any non-negligible extra computation. 
    \looseness=-1 
\end{abstract}

\section{Introduction}

Suppose we are given an unnormalized distribution $R$ over a discrete space $\mathcal{X}$. 
A common problem in both statistical inference and probabilistic machine learning is to draw samples from $\mathcal{X}$ distributed according to $R$ \citep{bishop2007, neal2011mcmc, blei2017}.
When $\mathcal{X}$ is intractably large---e.g., the space of graphs containing $N$ nodes or of subsets of a set with $N$ elements for largish $N$---addressing this challenge often 
requires the use of approximate algorithms such as Markov chain Monte Carlo (MCMC) or variational inference (VI).
Notably, MCMC suffers from slow convergence and correlated samples, especially in discrete spaces, and its correctness cannot be tractably assessed in general \citep{mcbook}.
VI, in contrast, reframes the sampling problem as that of finding the best approximation to $R$ within a 
tractable family, thereafter enabling the generation of independent samples from a measurably correct distribution \citep{malkin2023gflownets, zimmermann2023a}.  
\looseness=-1 


As a major milestone of VI for discrete spaces, GFlowNets \citep{bengio2021, theory} learn a sequential decision process that iteratively constructs each $x \in \mathcal{X}$ from an initial (empty) state in proportion to $R(x)$. 
Since their conception, they have been successfully deployed into applications ranging from causal discovery \citep{deleu2022bayesian, deleu2023joint} to LLM fine-tuning \citep{hu2023amortizing, venkatraman2024amortizingintractableinferencediffusion} and phylogenetic inference \citep{zhou2024phylogfn}.
Despite their potential, they suffer from insufficient state space exploration \citep{kim2025adaptive, madan2025towards}, slow training convergence \citep{liu2024generativemarginalizationmodels}, and learning instability (\Cref{fig:hvariances}), often failing to simulate even simple multi-modal distributions defined over relatively small $\mathcal{X}$ (\Cref{fig:randomwalksa}). \looseness=-1 


\begin{wrapfigure}[9]{l}{.4\textwidth}
    \vspace{-12pt} 
    \includegraphics[width=.8\linewidth, page=1]{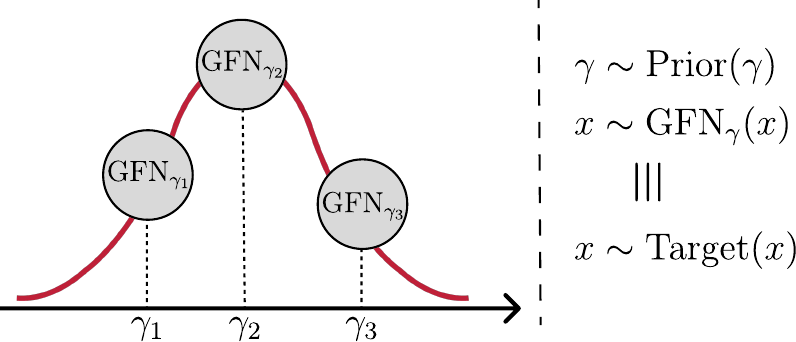}
    \caption{An ensemble of GFlowNets that, when averaged, matches the target.}
    \label{fig:theorysa} 
\end{wrapfigure}
To mitigate this issue, recent works have proposed learning an \emph{ensemble} of GFlowNets in which each component focuses on a nearly-unique subset of the state space \citep{yu2026exploringmultiplehighscoringsubspaces, yu2026partialgflownetacceleratingconvergence, dallantonia2025boostedgflownetsimprovingexploration}; in doing so, they incentivize the exploration of unvisited high-probability regions and accelerate learning convergence.
However, they often incur a substantial computational overhead---requiring several neural networks to be 
separately trained---and, despite being similar, lack a shared conceptual foundation.
In this scenario, our first contribution is to introduce a unifying theory 
to describe a mixture of GFlowNets (\Cref{fig:theorysa}), shedding light on the applicability and limitations of such an approach.
Drawing on this framework, we propose distinct algorithms that we show to both stabilize and accelerate learning convergence.   

In this context, we separate our presentation into continuously (CI) and discretely indexed (DI) mixtures of GFlowNets. 
On the one hand, CI GFlowNets learn an uncountable collection of samplers, each of which is associated to a real-valued vector whose value serves as an extra input to the underlying neural network parameterizing the sampler.
We demonstrate this formulation is connected to the theory of random features for graph learning \citep{sato2021random, abboud2021surprising}, providing a computationally simple approach to provably boost the sampler's expressivity in graph-structured tasks (\Cref{sec:continuousaa,sec:continuousa}).
In addition, we show CI GFlowNets stabilize learning by shifting the spectrum of the loss function's Hessian, reducing the dependence of the learned model on exogenous factors such as the initial value of the pseudorandom number generator (Proposition~\ref{prop:regularizationsa}). 

On the other hand, DI GFlowNets consist of a finite (or countably infinite) mixture of samplers.
Each sampler learns in proportion to a \emph{re-weighted} target; for the $k$th component, this is $x \mapsto R(x) q(k|x)$ for a \emph{posterior mixing distribution} $q(\cdot | x)$ (see \Cref{sec:discretesa}).
In this sense, DI GFlowNets support two learning modalities: centralized, with an embedding of $k$ being used as extra input to a component-wise shared model, and embarrassingly parallel \citep{neiswanger2014}, with distinct components being independently trained across non-communicating processes. 
We prove this framework encompasses two prior methods, Subgraph Asynchronous Learning \citep{silva2025generalization} and Boosted GFlowNets \citep{dallantonia2025boostedgflownetsimprovingexploration}, unifying otherwise disparate algorithms. 
\looseness=-1 

As the simplest non-trivial instantiation of DI GFlowNets, we propose setting $q(\cdot|x)$ as a sparse distribution assigning zero probability to all but a few components. 
This effectively corresponds to defining a \emph{covering} of the state space. 
We show, however, that conditioning a sampler to exclusively explore an arbitrary subset of $\mathcal{X}$ is NP-complete in general.
We consequently introduce \emph{Stratum-Conditioning} (SC) as a tractable approach for learning DI GFlowNets.
Simply put, SC partitions the state space by binning the outputs of a modular and real-valued function defined on $\mathcal{X}$, with each bin being called a \emph{stratum} (hence the name). 
This function can be set to the trajectory length required for reaching $x \in \mathcal{X}$, for instance, or as the distance between $x \in \mathcal{X}$ and a chosen reference point according to a prescribed metric (\Cref{sec:scgflownetsa}). 
We demonstrate the set of samplers learned in this way are connected via Doob's $h$-transform of a base GFlowNet (\Cref{sec:doobsa}).  
\looseness=-1 

Our experiments highlight SC GFlowNets substantially improve both learning convergence and state space exploration across diverse benchmark tasks.
Importantly, contrarily to recent algorithmic advances in the GFlowNet literature, SC \emph{adds no computational overhead} to training, being as fast to train as a single-component sampler (see \Cref{fig:runtimesa} and \Cref{fig:compssetsa,fig:compshypergridsa,fig:compsdiffusionsa}). 
As with most of machine learning research, which historically started from simple algorithms that have been incrementally enhanced over time, we believe SC GFlowNets set a minimal baseline for learning a mixture of GFlowNets. 
In summary, our contributions are as follows. 
\looseness=-1 

\begin{enumerate}[leftmargin=12pt]
    \item We introduce an abstract framework for learning an ensemble of GFlowNets, demonstrating its usefulness in providing an unifying view of prior techniques and in designing novel algorithms. 
    \item We define 
    continuously (CI) and discretely indexed (DI) GFlowNets, outlining their empirical and theoretical advantages over a single-component sampler.
    \item We propose Stratum-Conditioned (SC) GFlowNets as an effective form of learning DI GFlowNets, supporting both centralized and parallel training. 
    \item We show that SC GFlowNets speed up convergence and mode discovery across benchmark tasks. \looseness=-1 
\end{enumerate}

\section{Preliminaries} \label{sec:preliminariesa} 

\pp{Definitions.} 
We let $\mathcal{X}$ and $\mathcal{S}$ be finite spaces, with $\mathcal{S}$ corresponding to \emph{partial realizations} of $\mathcal{X}$.
We also let $s_{o} \in \mathcal{S}$ be the \emph{initial state} from which any $x \in \mathcal{X}$ can be constructed. 
An example the reader should keep in mind is when $\mathcal{X}$ is the set of subsets of size $K$ of a given finite set and $\mathcal{S}$ is the set of subsets of size less than $K$, with $s_{o} = \emptyset$. 
We call a \emph{state graph} a directed acyclic graph having $s_{o}$ as the only state without incoming edges and $x \in \mathcal{X}$ as the only states without outgoing edges, which we call \emph{terminal}, and such that there is a path from $s_{o}$ to every $x \in \mathcal{X}$ containing only states from $\mathcal{S}$. 
\looseness=-1 

Notationally, we let $[N] = \{1, \dots, N\}$ for each positive integer $N \in \mathbb{N}$ and $s \rightsquigarrow s'$ be the set of trajectories $\tau = \{(s, s_{1}), (s_{1}, s_{2}), \dots, (s_{d}, s')\}$ from $s \in \mathcal{S}$ to $s' \in \mathcal{S} \cup \mathcal{X}$. 


\pp{GFlowNets.} 
A GFlowNet learns to stochastically construct $x \in \mathcal{X}$ from $s_{o}$ by following a policy $p_{F} \colon \mathcal{S} \times (\mathcal{S} \cup \mathcal{X}) \rightarrow [0, 1]$ such that $p_{F}(s, s') > 0$ if and only if $s \rightarrow s'$ in the state graph. 
The objective is for the marginal distribution of $p_{F}$ over $\mathcal{X}$ to match the unnormalized distribution $R \colon \mathcal{X} \rightarrow \mathbb{R}_{+}$ from which we aim to sample, i.e., 
\begin{equation} \label{eq:marginalsa} 
    m_{F}(x) \coloneqq \sum_{\tau \in s_{o} \rightsquigarrow x} \prod_{(s_{i}, s_{i + 1}) \in \tau} p_{F}(s_{i}, s_{i + 1}) \coloneqq \sum_{\tau \in s_{o} \rightsquigarrow x} p_{F}(s_{o}, \tau) \propto R(x). 
\end{equation}
As $m_{F}(x)$ is generally intractable, we introduce a \emph{backward policy} $p_{B}(x, \tau)$ representing the probability of having followed $\tau$ given that $x$ was sampled. 
We then search for $(p_{F}, p_{B})$ satisfying a \emph{consistency condition}, with a popular choice being the 
trajectory balance \citep[TB]{malkin2022trajectory}: $Z \cdot p_{F}(s_{o}, \tau) = R(x) \cdot p_{B}(x, \tau)$ and $Z$ as a learned constant. 
In practice, $p_{F}(s, \cdot)$ is parameterized as a neural network and this search is carried out via stochastic gradient descent. 
Please refer to \Cref{sec:backgroundsa,sec:relatedsa} for a comprehensive background and recent advances on GFlowNets. \looseness=-1 

\section{Learning a Mixture of GFlowNets}

Simply put, the central challenges in achieving \Cref{eq:marginalsa} are the problem's \emph{under-determination}, owning to the existence (in general) of infinitely many pairs $(p_{F}, p_{B})$ abiding by any given consistency condition, and \emph{sample inefficiency}, requiring extensive state graph exploration for finding a resolutive $(p_{F}, p_{B})$.
We will show a mixture formulation can naturally address both of them, focusing on the latter in the main text and elaborating on the former in \Cref{sec:continuousa} of the supplement. 
\looseness=-1












Towards this objective, we develop an abstract framework for learning a mixture of GFlowNets. 
For starters, we let $\Gamma$ be an \emph{index space} and $\{p_{F}^{\smash[t]{(\gamma)}}, p_{B}^{\smash[t]{(\gamma)}}\}_{\gamma \in \Gamma}$ be the corresponding set of policies. 
Our goal is for the marginal distribution of these policies on each $x \in \mathcal{X}$, when weightedly averaged over $\gamma \in \Gamma$, to match the target $R(x)$.
Formally, we assume $\Gamma$ to be embedded in a topological measurable space equipped with a measure $\nu$ over its Borel $\sigma$-algebra.
This allows us to set aside the specific properties of $\Gamma$, while accounting for common use cases via a unified language.
\looseness=-1 

\begin{remark}[$\Gamma$ can be either discrete or continuous] 
    Common choices for $\Gamma$ and $\nu$ are $\Gamma = \mathbb{R}^{d_{\gamma}}$ with $\nu$ as the Lebesgue measure and $\Gamma = [K]$ with $\nu$ as the counting measure (\Cref{sec:continuousa,sec:discretesa}). \looseness=-1 
\end{remark} 

\pp{Designing a mixture of GFlowNets.} 
With this setup, we let $p_{M}$ be a probability measure on $\Gamma$ absolutely continuous with respect to $\nu$.
The subscript $M$ indicates $p_{M}$ is a \emph{mixing (prior)} distribution over $\Gamma$. 
Our goal is for the marginal identity analogous to \Cref{eq:marginalsa} \looseness=-1 
\begin{equation} \label{eq:marginalss} 
    \int_{\Gamma} p_{M}(\mathrm{d}\gamma) \sum_{\tau \in s_{o} \rightsquigarrow x} p_{F}^{\smash[t]{(\gamma)}}(s_{o}, \tau) \coloneqq \int_{\Gamma} p_{M}(\mathrm{d}\gamma) m_{F}^{\smash[t]{(\gamma)}}(x) \propto R(x)
\end{equation}
to be satisfied. 
As in \Cref{eq:marginalsa}, we call $m_{F}^{\smash[t]{(\gamma)}}(x)$ the marginal of $p_{F}^{\smash[t]{(\gamma)}}(s_{o}, \cdot)$ over $\mathcal{X}$. 
Under this condition, we can sample from $R$ by first drawing $\gamma \sim p_{M}(\mathrm{d}\gamma)$ and then sampling $x \sim m_{F}^{\smash[t]{(\gamma)}}(x)$. 
Nevertheless, \Cref{eq:marginalss} is doubly intractable: neither $m_{F}^{\smash[t]{(\gamma)}}$ nor its average under $\gamma \sim p_{M}$ can be computed efficiently. 
We circumvent this issue by introducing a \emph{mixing (posterior) distribution} $q_{M}(\mathrm{d} \gamma | x)$ and rewriting \Cref{eq:marginalss} as\footnote{Recall $f(\gamma) \mu(\mathrm{d}\gamma) = g(\gamma) \nu(\mathrm{d}\gamma)$ means $\int_{B} f(\gamma) \mu(\mathrm{d}\gamma) = \int_{B} g(\gamma) \nu(\mathrm{d}\gamma)$ for any measurable set $B$.}
\begin{equation} \label{eq:mmtbsa}
    p_{M}(\mathrm{d}\gamma) m_{F}^{\smash[t]{(\gamma)}}(x) \propto R(x) q_{M}(\mathrm{d} \gamma | x). 
\end{equation}
The role of $q_{M}$ in the equation above is akin to that of \emph{responsibilities} in expectation-maximization algorithms \citep{dempster1977}, representing the probability a given $x \in \mathcal{X}$ was drawn from a $\gamma$-indexed sampler. 
As in \Cref{sec:preliminariesa}, we turn \Cref{eq:mmtbsa} into a learning objective---referred to as \emph{mixed trajectory balance} (MTB)---by replacing the intractable marginal with tractable policies. \looseness=-1 

\begin{definition}[MTB condition] \label{def:mtbsa} 
    Let $\gfn_{\gamma} = \left(p_{F}^{\smash[t]{(\gamma)}}, p_{B}^{\smash[t]{(\gamma)}}\right)$ and $\{\gfn_{\gamma}\}_{\gamma \in \Gamma}$ be a $\Gamma$-indexed set of GFlowNets, and $p_{M}(\mathrm{d}\gamma)$ and $q_{M}(\mathrm{d}\gamma | x)$ be \emph{prior} and \emph{posterior} mixing distributions over $\Gamma$. 
    Then, 
    \begin{equation*}
        p_{M}(\mathrm{d}\gamma) p_{F}^{\smash[t]{(\gamma)}}(s_{o}, \tau) = \frac{R(x)}{Z} q_{M}(\mathrm{d} \gamma | x) p_{B}^{\smash[t]{(\gamma)}}(x, \tau)
    \end{equation*}
    for each $x \in \mathcal{X}$ and $\tau \in s_{o} \rightsquigarrow x$ is the \emph{mixed trajectory balance} condition, with $Z = \sum_{x \in \mathcal{X}} R(x)$. 
\end{definition}

The sufficiency of Definition~\ref{def:mtbsa} for \Cref{eq:marginalss} is a direct consequence of \cite[Proposition 1]{malkin2022trajectory}. 
As it will soon become clear, it will be 
notationally 
convenient to define 
\begin{equation*}
    Z(\mathrm{d}\gamma) = Z \cdot p_{M}(\mathrm{d}\gamma) \text{ and } R(x, \mathrm{d}\gamma) = R(x) q_{M}(\mathrm{d} \gamma | x) 
\end{equation*}
as unnormalized measures over $\Gamma$ and $\mathcal{X} \times \Gamma$.
This allows us to rewrite MTB as 
\begin{equation*}
    Z(\mathrm{d}\gamma) \cdot p_{F}^{\smash[t]{(\gamma)}}(s_{o}, \tau) = R(x, \mathrm{d}\gamma) \cdot p_{B}^{\smash[t]{(\gamma)}}(x, \tau), 
\end{equation*}
resembling the familiar TB condition.
We also deduce from this that other balance conditions, such as Detailed Balance (DB) \citep{foundations} and SubTB \citep{madan2022learninggf}, remain valid in our setting. 
Given a function $F \colon s \in \mathcal{S} \mapsto F(s, \mathrm{d}\gamma)$ mapping $s$ to a measure in $\Gamma$, for instance, the DB condition can be written as, for $s \in \mathcal{S}$, $s' \in \mathcal{S} \cup \mathcal{X}$, and $x \in \mathcal{X}$ 
\begin{equation*}
    F(s, \mathrm{d}\gamma) \cdot p_{F}^{\smash[t]{(\gamma)}}(s, s') = F(s', \mathrm{d}\gamma) p_{B}^{\smash[t]{(\gamma)}}(s', s) \text{ and } F(x, \mathrm{d}\gamma) = R(x, \mathrm{d}\gamma). 
\end{equation*}


This raises the question of how to select both $Z(\mathrm{d}\gamma)$ and $R(x, \mathrm{d}\gamma)$, which will be addressed in \Cref{sec:discretesa}. 
Under our abstract setting, however, we show both measures are intimately related. 
\looseness=-1 

\begin{proposition}[Relationship between $Z(\mathrm{d}\gamma)$ and $R(x, \mathrm{d}\gamma)$] \label{prop:relsa} 
    Let $\{\gfn_{\gamma}\}_{\gamma \in \Gamma}$ and $p_{M}$ and $q_{M}$ be as in Definition~\ref{def:mtbsa}, and $Z(\mathrm{d}\gamma)$ and $R(x, \mathrm{d}\gamma)$ be as above. 
    Then, $Z(\mathrm{d}\gamma)$ and $R(x, \mathrm{d}\gamma)$ jointly satisfy
    \begin{equation*}
        Z(\mathrm{d}\gamma) = \sum_{x \in \mathcal{X}} R(x, \mathrm{d}\gamma) \text{ and } R(x, \mathrm{d}\gamma)  = Z(\mathrm{d}\gamma) \cdot m_{F}^{\smash[t]{(\gamma)}}(x). 
    \end{equation*}
\end{proposition}

Under Proposition~\ref{prop:relsa}, $Z(\mathrm{d}\gamma)$ may be understood as the fraction of total probability mass ($Z$) being assigned to the model indexed by $\gamma$, while $x \mapsto R(x, \mathrm{d}\gamma)$ represents the (unnormalized) marginal distribution of such model over $\mathcal{X}$. 
This supports the interpretation of $\Gamma$-indexing as a \emph{soft partition} of $\mathcal{X}$: each $\gfn_{\gamma}$ primarily focuses on a (potentially) distinct subset of $\mathcal{X}$.

\pp{Learning a GFlowNet ensemble.} 
We learn a MTB-abiding $\gfn_{\Gamma} \!= \! (p_{M}, q_{M}, \{\gfn_{\gamma}\}_{\gamma \in \Gamma})$ by optimizing \looseness=-1 
\begin{equation} \label{eq:tblossa} 
    \mathcal{L}_{\mathrm{TB}}(\tau, \mathrm{d}\gamma ; \gfn_{\Gamma}) = \left( \log \frac{Z(\mathrm{d}\gamma) \cdot p_{F}^{\smash[t]{(\gamma)}}(s_{o}, \tau)}{R(x, \mathrm{d}\gamma) \cdot p_{B}^{\smash[t]{(\gamma)}}(x, \tau)} \right)^{2}   
\end{equation}
for all $\tau, \gamma$; we call $\mathcal{L}_{\mathrm{TB}}$ the \emph{TB loss function} \citep{malkin2022trajectory}. 
As an aside, we notice in Definition~\ref{def:lossesa} that similar learning objectives from \citet{foundations, deleu2024discreteprobabilisticinferencecontrol}, such as SubTB and DB, can also be seamlessly extended to our setting. 

From this, learning continues business as usual: we parameterize $p_{F}^{\smash[t]{(\gamma)}}(s, \cdot)$ and (optionally) $p_{B}^{\smash[t]{(\gamma)}}(s, \cdot)$ as neural networks receiving both $s$ and $\gamma$ as inputs and, defining exploratory (full support) distributions $\rho_{M}(\mathrm{d}\gamma)$ and $\rho^{\smash[t]{(\gamma)}}(s_{o}, \tau)$, we solve the stochastic program 
\begin{equation} \label{eq:optstochasticsa}  
    \min_{\gfn_{\Gamma}} \mathbb{E}_{B \sim \rho_{M}(\cdot), \tau \sim \rho^{\smash[t]{(\gamma)}}(s_{o}, \cdot)} \left[ \mathcal{L}_{\mathrm{TB}}\left( \tau, B ; \gfn_{\Gamma} \right) \right] 
\end{equation}
via stochastic gradient descent \citep{jordan2024muon}. 
In the following sections, we present both specific instantiations and an empirical validation of the proposed framework. 

\begin{remark}[Modularity]
    We also notice methods for 
    learning both $p_{F}$ and $\rho$ in \Cref{eq:optstochasticsa} can be directly applied to our setting \citep{kim2025adaptive, madan2025towards, dallantonia2026avoidknowdivergenttrajectory}. 
\end{remark}

\looseness=-1

\section{Continuously Indexed GFlowNets} 
\label{sec:continuousaa} 

The most straightforward form of \Cref{eq:optstochasticsa} assumes $q(\mathrm{d}\gamma|x) = p_{M}(\mathrm{d}\gamma)$, i.e., each component is equally likely to generate $x \in \mathcal{X}$. 
In this case, the ratio $\nicefrac{Z(\mathrm{d}\gamma)}{R(x, \mathrm{d}\gamma)} = \nicefrac{Z}{R(x)}$ becomes $\gamma$-independent, and our problem reduces to minimizing the familiar TB loss averaged over $\gamma \in \Gamma$ \citep{malkin2022trajectory}.   
Although deceptively simple, this approach corresponds to \emph{augmenting} our model with \emph{random features} (RFs) when $\Gamma = \mathbb{R}^{d_{\gamma}}$. 
This is a well-established procedure for boosting the expressivity of graph neural networks \citep{abboud2021surprising}, often-used 
by GFlowNets in graph-structured spaces. 
Drawing on this, we show \emph{continuously indexed} (CI) GFlowNets---as we call the RF-augmented sampler---can arbitrarily approximate any graph distribution. \looseness=-1 
\begin{remark}[CI GFlowNets as Universal Approximators]
    CI GFlowNets are strictly more expressive, in terms of which distributions they can realize, than GFlowNets (see Example~\ref{example:graphsa}). 
    In particular, when considering graph-structured tasks based on edge-additive generative processes, CI GFlowNets are \emph{universal approximators} (Proposition~\ref{prop:universalsa}). 
\end{remark}
We formalize this discussion in \Cref{sec:continuousa}, where we also demonstrate CI GFlowNets can have a smoothing effect on the loss surface \citep[as in][]{bishop1995training}, significantly reducing training instability. 


\section{Discretely Indexed GFlowNets} \label{sec:discretesa} 



To start our discussion on discrete GFlowNet mixtures, we first notice that the TB loss function in \Cref{eq:tblossa} does not impose a preference of a non-trivial solution over a collapsed sampler in which $Z(\mathrm{d}\gamma) = Z$ and $R(x, \mathrm{d}\gamma) = R(x)$ or, more generally, in which the ratio $\nicefrac{R(x, \mathrm{d}\gamma)}{Z(\mathrm{d}\gamma)}$ does not depend on $\gamma$.
Although a $\gamma$-independent target can exhibit beneficial properties, such as greater expressivity and a smoother loss surface (as in \Cref{sec:continuousaa}), this should be a design decision---achieved by setting $q_{M}(\mathrm{d}\gamma | x) = p_{M}(\mathrm{d}\gamma)$---rather than a consequence of training.
In practice, learning a non-collapsed mixture from \Cref{eq:optstochasticsa} requires constraining the search space for the mixing distributions $q_{M}$ and $p_{M}$ either via soft regularization or hard constraints.

We advocate for the latter in this section, suggesting to fix $\Gamma$ as a finite set and $q_{M}(\mathrm{d}\gamma|x)$ as a \emph{sparse distribution} assigning zero probability to most $\gamma \in \Gamma$. 
We also briefly elaborate on using a countably infinite $\Gamma$ when discussing Boosted GFlowNets \citep{dallantonia2025boostedgflownetsimprovingexploration} in \Cref{sec:boostingsa}. 
\looseness=-1 

\pp{Discretely Indexed GFlowNets.}
We refer to the corresponding sampler as Discretely Indexed (DI) GFlowNets. 
As notational shortcuts, we let $\Gamma = [K]$, and write the measures $Z$ and $R$ and $q_{M}$ as $Z(\{k\}) = Z_{k}$ and $R(x, \{k\}) = R_{k}(x)$ and $q_{M}(\{k\} | x) = q_{k}(x)$, respectively. 
In addition, we let $\mathbf{Z} = (Z_{k})_{k \in [K]}$ and $\rho_{k}$ be the exploration policy for the mixture's $k$th component.
Our objectives are to define $q_{k}$ and $\rho_{k}$; as with single-component GFlowNets, we set $\mathbf{Z}$ as a learnable parameter.   
\looseness=-1 

\looseness=-1 


\subsection{Stratum-Conditioned GFlowNets} \label{sec:scgflownetsa} 


In this context, we propose setting $q_{k}(x)$ according to a hand-crafted \emph{covering} $\{X_{k}\}_{k \in [K]}$ of $\mathcal{X}$: under this condition, $q_{k}(x)$ corresponds to a uniform distribution over the indices $k$ for which $\mathcal{X}_{k}$ contains $x$. 
We call each $\mathcal{X}_{k}$ a \emph{stratum} 
due to its construction as the pre-image of an interval with respect to a real-valued function on $\mathcal{X}$; see Remark~\ref{remark:modulosa}. 
This section formalizes the ensuing algorithm, called Stratum-Conditioned (SC) GFlowNets, explains how it can be implemented, and describes sufficient conditions for which sampling from $\mathcal{X}_{k}$ is computationally tractable.

\looseness=-1 

\pp{Stratum Conditioning.}
The first requirement for tractability is for us to know how many sets $N(x)$ in $\{\mathcal{X}_{k}\}_{k \in [K]}$ contain a given $x$, allowing for the efficient computation of $q_{k}(x)$.  
As such, we define a \emph{tractable $K$-cover} of $\mathcal{X}$, denoted as $\mathrm{TC}(\mathcal{X}, K)$. 

\begin{definition}[Tractable $K$-Cover]
    Let $K$ be a positive integer and $\mathcal{X}$ be a set. We say $\mathrm{TC}(\mathcal{X}, K) = \{\mathcal{X}_{k}\}_{k=1}^{K}$ is a \emph{Tractable $K$-Cover} of $\mathcal{X}$ if the following conditions are satisfied. 
    \begin{enumerate}[leftmargin=18pt]
        \item $\mathrm{TC}(\mathcal{X}, K)$ is an exact cover of $\mathcal{X}$, i.e., each $\mathcal{X}_{k} \subseteq \mathcal{X}$ and $\bigcup_{k=1}^{K} \mathcal{X}_{k} = \mathcal{X}$.  
        \item For each $x \in \mathcal{X}$, the quantity $N(x) \coloneqq \#\{k \in [K] \colon x \in \mathcal{X}_{k}\}$ is known\footnote{An  \emph{intractable} cover, in contrast, could be built upon an expensive multiclass classifier $f \colon \mathcal{X} \rightarrow [K]$.}. 
    \end{enumerate}
\end{definition}

The second condition allows efficient querying of $q_{k}(x)$, 
defined as $q_{k}(x) = N(x)^{-1}$ if $x \in \mathcal{X}_{k}$ and $q_{k}(x) = 0$ otherwise. 
In practice, we often let $\mathrm{TC}(\mathcal{X}, K)$ be a disjoint cover---a partition---of $\mathcal{X}$; in this case, $q_{k}(x) = 1$ for the only $k$ with $x \in \mathcal{X}_{k}$. 
A 
corollary of Proposition~\ref{prop:relsa} cleanly characterizes both $Z_{k}$ and $R_{k}(x)$ as a fraction of the probability masses assigned to $\mathcal{X}_{k}$ and $x$ under a TC. 
\looseness=-1 

\begin{corollary}[$Z_{k}$ as the mass within $\mathcal{X}_{k}$] \label{col:pmassa} 
    Let $\mathrm{TC}(\mathcal{X}, K) = \{\mathcal{X}_{k}\}_{k=1}^{K}$ and $\gfn_{[K]}$ be a GFlowNet ensemble satisfying the MTB condition in Definition~\ref{def:mtbsa}. 
    Then, denoting by $\mathbbm{1}[\cdot]$ the indicator function, \looseness=-1 
    \begin{equation*}
        Z_{k} = \sum_{x \in \mathcal{X}_{k}} \frac{R(x)}{N(x)} \text{ and } R_{k}(x) = \frac{R(x)}{N(x)} \cdot \mathbbm{1}[x \in \mathcal{X}_{k}]. 
    \end{equation*}
\end{corollary}

In particular, when $N(x)$ is constant, the learned policy $p_{F}^{\smash[t]{(k)}}$ generates samples in proportion to the restriction of $R(x)$ to $\mathcal{X}_{k}$, a property that might be of independent interest when the covering has a meaningful semantic value for the problem at hand.  

\begin{remark}[$R_{k}(x)$ for a disjoint cover]
    Under the conditions of Corollary~\ref{col:pmassa}, assume $N(x) = N$ (e.g., $N = 1$ for a disjoint cover). Then, $m_{F}^{\smash[t]{(k)}}(x) \propto R(x) \cdot \mathbbm{1}[x \in \mathcal{X}_{k}]$. 
\end{remark}

We emphasize $\{p_{F}^{\smash[t]{(k)}}\}_{k \in [K]}$ are parameterizable via a shared neural network architecture receiving that receives an embedding of $k$ as input; they need not be completely separate samplers. 
To understand the benefits of partitioning, 
we consider 
the 
\textsc{Lines} domain 
\citep{silva2026pathdependent}, described next. 
\looseness=-1 

\begin{figure}
    \centering
    \includegraphics[width=.8\linewidth]{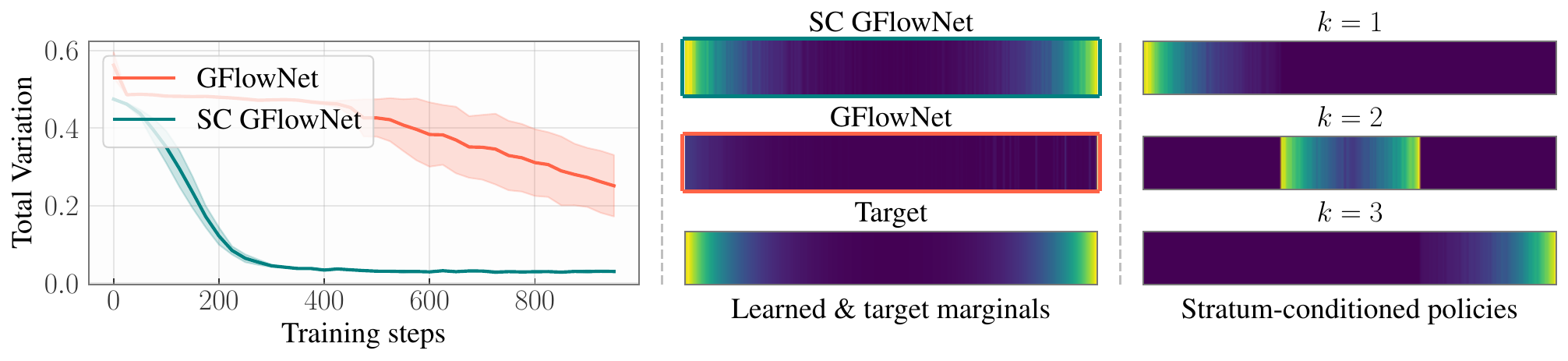}
    \caption{SC GFlowNets for the \textsc{Lines} environment ($N = 256$). (Left) We show the TV distance between the learned and target distributions, emphasizing the faster convergence promoted by stratum-conditioning. (Middle) The marginal distribution of both models after 1000 training steps, and the target. (Right) The marginal distribution of SC GFlowNets for each stratum $k \in \{1, 2, 3\}$.}
    \label{fig:linesa}
    \vspace{-12pt} 
\end{figure}

\begin{example}[\textsc{Lines} domain] \label{example:linesa}  
    We define $\mathcal{S} = \{1, \dots, N\}$ and $\mathcal{X} = \{1, \dots, N\} \times \{\top\}$ with $N \ge 2$ and $\top$ indicating the state is in $\mathcal{X}$ and not in $\mathcal{S}$. 
    The initial state is $s_{o} = 1$ and each transition corresponds to adding $1$ to the current state if it is less than $N$ or attaching $\top$ and terminating. Let 
    \begin{equation*}
        \mathcal{X}_{1} = \{1, \dots, \lfloor \nicefrac{N}{3} \rfloor\}, \, 
        \mathcal{X}_{2} = \{\lfloor\nicefrac{N}{3}\rfloor + 1, \dots, \lfloor \nicefrac{2 \cdot N}{3} \rfloor\}, \text{ and } \mathcal{X}_{3} = \{\lfloor \nicefrac{2 \cdot N}{3} \rfloor + 1, \dots, N\} 
    \end{equation*}
    be the partition. 
    We design $p_{F}^{\smash[t]{(k)}}(s, \cdot)$ by masking out the terminating transition if $s \notin \mathcal{X}_{k}$ and masking out the successor transition if $s + 1 \notin \mathcal{X}_{k}$, i.e., 
    \begin{equation*}
        p_{F}^{\smash[t]{(k)}}(s, s + 1) \propto \mathbbm{1} [s + 1 \in \mathcal{X}_{k}] \cdot 
        e^{\phi(s, k)^{\top} w_{1}} 
        \text{ and } p_{F}^{\smash[t]{(k)}}(s, (s, \top)) \propto \mathbbm{1}[s \in \mathcal{X}_{k}] \cdot 
        e^{\phi(s, k)^{\top} w_{2}}
        , 
    \end{equation*}
    with $\phi \colon [N] \times [K] \rightarrow \mathbb{R}^{d}$ and $w_{1}, w_{2} \in \mathbb{R}^{d}$ being learnable parameters. 
    \Cref{fig:linesa} shows the error curves measured in terms of total variation distance (TV) for (i) a GFlowNet and (ii) a SC GFlowNet for the target distribution $R(x) = \exp\{- (\beta N)^{-1} \cdot \min_{y \in \{3, N - 2\}} |x - y| \}$ with $N = 256$ and $\beta = 0.1$. 
    Notably, while SC GFlowNet quickly converges to the target distribution, a GFlowNet struggles to achieve a comparable goodness-of-fit, requiring over three times more training steps. \looseness=-1 
\end{example}



\pp{How to design a tractable covering?}
A pivotal question for 
SC GFlowNets is how to define the mapping $\mathrm{TC}(\mathcal{X}, K)$ for each $\mathcal{X}$ and $K$.
Although distinct problems may benefit from idiosyncratic strategies, we provide next a blueprint for designing $\mathrm{TC}(\mathcal{X}, K)$ in commonly encountered tasks.

\begin{remark}[Modulo-based $\mathrm{TC}$] \label{remark:modulosa} 
    Example~\ref{example:linesa} hints on how to construct $\mathrm{TC}$: we let $\mathcal{X}_{k}$ be the set of states whose distance to a reference point is within a prescribed interval. 
    As most state spaces of interest lack an inherent geometric structure, however, we use a \emph{modular function} as a general proxy for distance. 
    To make sense of such definition, recall a \emph{compositional} object can be described as a \emph{multisubset} of a \emph{component set} $\mathcal{C}$, i.e., each $s \in \mathcal{S} \cup \mathcal{X}$ can be identified with $s \subseteq \mathcal{C}$. 
    In Example~\ref{example:linesa}, for instance, $\mathcal{C} = \{1, \top\}$ and $s \equiv \{\!\{1, \dots, 1\}\!\}$ ($s$ times). 
    In this scenario, let $\psi$ be a bounded modular function with $\psi(\{c\}) \ge 0$ for all $c \in \mathcal{C}$. 
    We let \looseness=-1 
    \begin{equation*}
        \mathcal{X}_{k} = \left\{ x \in \mathcal{X} \colon \frac{k - 1}{K} (\psi(\mathcal{C}) + 1) \le \psi(x) < \frac{k}{K} ( \psi(\mathcal{C}) + 1) \right\}. 
    \end{equation*}
    When $\psi(s) = \#s$, we call $\mathrm{TC}(\mathcal{X}, K)$ to be a \emph{length-based covering}. 
    In general, we let $\psi(s) = \# (s \cap r)$ for a given \emph{reference set} $r \subseteq \mathcal{C}$; if $r = \emptyset$, the SC GFlowNet reduces to a GFlowNet. 
    \looseness=-1 
\end{remark}


Importantly, while easily solved for some $\psi$ such as the length-based one in Example~\ref{example:linesa}, we show next 
that the \emph{decision problem} of inferring whether $\mathcal{X}_{k}$ is \emph{reachable} from a state in $\mathcal{S}$ is NP-complete in general.
This is required for masking out from $p_{F}^{\smash[t]{(k)}}$ transitions leading to states for which $\mathcal{X}_{k}$ becomes unreachable, ensuring the sampler does not become deadlocked in a zero-probability region.
\looseness=-1 

\begin{proposition}[Sampling from a general TC is NP-complete] \label{prop:npsa} 
    Let $\mathrm{TC}(\mathcal{X}, K)$ be a modulo-based TC 
    built upon a modular function $\psi$ and $\mathcal{X}_{k} = \{x \in \mathcal{X} \colon \psi(x) \in [a_{k}, b_{k})\}$ for numbers $\{(a_{k}, b_{k})\}_{k=1}^{K}$. 
    Let $s \in \mathcal{S} \subset 2^{\mathcal{C}}$.
    Given $k \in [K]$, we consider the following decision problem, 
    \begin{equation} \label{eq:decisionsa} 
        \exists s' \in 2^{\mathcal{C} \setminus s} \text{ s.t. } s \cup s' \in \mathcal{X} \text{ and } \psi(s \cup s') \in [a_{k}, b_{k}), 
    \end{equation}
    i.e., whether there is a terminal $x \in \mathcal{X}_{k}$ reachable from $s \in \mathcal{S}$. Then, if $\psi$ is an arbitrary modular function evaluable in polynomial time on the size of its input, \Cref{eq:decisionsa} is NP-complete. 
\end{proposition}

We demonstrate Proposition~\ref{prop:npsa} by reducing \Cref{eq:decisionsa} to the \texttt{SubsetSum} problem, which is NP-complete \citep{clrs2009}. 
Rather than a strict limitation of SC GFlowNets, however, Proposition~\ref{prop:npsa} should be read as a guiding design principle for $\mathrm{TC}(\mathcal{X}, K)$, suggesting there are partitions for which sampling cannot be efficiently implemented.
In contrast, drawing on Example~\ref{example:linesa}, we show there is a specific class of functions $\psi$ for which \Cref{eq:decisionsa} is decidable in polynomial time. 
\looseness=-1 

\begin{proposition}[Efficient sampling from symmetric junta-based TC] \label{prop:npsaa} 
    We call a function $\psi \colon 2^{\mathcal{C}} \rightarrow \mathbb{R}$ a \emph{symmetric junta} if there are a \emph{reference set} $r \subseteq 2^{\mathcal{C}}$ and a function $\xi \colon \mathbb{Z} \rightarrow \mathbb{R}$ for which $\psi(x) = \xi(\# (x \cap r))$ for all $x \subseteq \mathcal{C}$.
    We assume $\psi$ can be evaluated in polynomial time. 
    Then, if \looseness=-1 
    \begin{equation*}
        \mathcal{X}_{k} = \{x \in \mathcal{X} \colon \psi(x) \in [a_{k}, b_{k}) \} 
    \end{equation*}
   for $k \in [K]$ is our $\mathrm{TC}(\mathcal{X}, K)$, the decision problem in \Cref{eq:decisionsa} can be resolved in polynomial time with respect to $\# \mathcal{C}$.
   This remains true when we replace $2^{\mathcal{C}}$ by the bounded multisubsets of $\mathcal{C}$. \looseness=-1  
\end{proposition}

As we will see in \Cref{sec:experimentsa,sec:experimentsaa}, symmetric junta-based TCs can be applied to a broad range of problems in the GFlowNet literature.
Building on the solution to \Cref{eq:decisionsa}, we can efficiently mask out undesirable transitions from each policy. 
In practice, this operation requires a negligible amount of arithmetic operations, adding an imperceptible extra cost to the sampling process. 
In Example~\ref{example:linesa}, as an illustration, we solely verify whether $s + 1$ is within the boundaries of $\mathcal{X}_{k}$. 
\looseness=-1  


\subsection{Complementary Characterization of SC GFlowNets} \label{sec:doobsa} 

\pp{Connection to Doob's $h$-transform.}\footnote{
This section is fairly technical and can be skipped by readers solely interested in our method.}
A natural question raised by our exposition is whether there is a relationship between each $p_{F}^{\smash[t]{(k)}}$ for $k \in [K]$. 
We show in this section that, given a solution $p_{F}$ to the original problem of sampling from $x \in \mathcal{X}$ in proportion to $R(x)$ and \emph{any} partition $\{\mathcal{X}_{k}\}_{k=1}^{K}$, a solution $p_{F}^{\smash[t]{(k)}}$ to the optimization problem in \Cref{eq:optstochasticsa} can be derived from a Doob's $h$-transform of $p_{F}$ when $\mathrm{TC}$ is a disjoint cover. 
To understand this, we 
review 
Doob's $h$-transform of 
Markov chains.
\looseness=-1 

\begin{definition}[Doob's $h$-transform] \label{def:doobsa} 
    Let $\{S_{t}\}_{t \ge 1}$ be a Markov chain with state space $\bar{\mathcal{S}} \coloneqq \mathcal{S} \cup \mathcal{X}$ and transition kernel $\kappa \colon \bar{\mathcal{S}} \times \bar{\mathcal{S}} \rightarrow [0, 1]$. 
    Let also $h \colon \bar{\mathcal{S}} \rightarrow \mathbb{R}$ be a \emph{harmonic function}, i.e., $h(s) = \mathbb{E}_{s' \sim \kappa(s, \cdot)} \left[ h(s') \right]$. 
    We call $\kappa^{\smash[t]{(h)}}(s, s') = \kappa(s, s') \cdot \nicefrac{h(s')}{h(s)}$ the \emph{Doob's $h$-transform} of $\kappa$. 
\end{definition}

We refer the reader to \citep[Chapter III]{rogers2000} for a comprehensive overview of Definition~\ref{def:doobsa} and its applications to stochastic processes. 
This section shows that there is a harmonic function $h_{k}$ for which $p_{F}^{\smash[t]{(k)}}$ is the Doob's $h_{k}$-transformation of $p_{F}$ for each $k \in [K]$.
To see this, we 
describe the generative process of a GFlowNet as 
a Markov chain having $\mathcal{X}$ as the absorbing set. \looseness=-1 
\looseness=-1

\begin{proposition}[GFlowNet as an absorbing Markov chain] \label{prop:markovsa} 
    Let $p_{F}$ be a GFlowNet's policy abiding by any balance condition and $\kappa_{F}$ be a Markov kernel on $\bar{\mathcal{S}}$ defined as $\kappa_{F}(s, s') = p_{F}(s, s')$ if $s \in \mathcal{S}$ and $\kappa_{F}(x, s') = \mathbbm{1}[s' = x]$ if $x \in \mathcal{X}$ for each $s' \in \bar{\mathcal{S}}$. 
    Let $\{S_{t}^{\smash[t]{(f)}}\}_{t \ge 1}$ be a Markov chain following $\kappa_{F}$, and let $T_{\mathcal{X}} = \min \{t \colon S_{t}^{\smash[t]{(f)}}\in \mathcal{X}\}$ be its hitting time on $\mathcal{X}$. 
    When $p_{F}$ satisfies the TB, 
    \begin{equation*}
        \mathbb{P}_{\kappa_{F}}\left[ T_{\mathcal{X}} < \infty | s_{o} \right] = 1 \text{ and } \mathbb{P}_{\kappa_{F}}[ S_{T_{\mathcal{X}}}^{\smash[t]{(f)}} = x \in \mathcal{X} | s_{o} ] = \frac{R(x)}{\sum_{x \in \mathcal{X}} R(x)} 
    \end{equation*}
    with $\mathbb{P}_{\kappa_{F}}$ being the probability measure induced by the transition kernel $\kappa_{F}$. 
    Equivalently, let $p_{B}$ be a GFlowNet's backward policy and $\kappa_{B}$ a Markov kernel on $\bar{\mathcal{S}}$ such that $\kappa_{B}(s, s') = p_{B}(s, s')$ for $s \in \bar{\mathcal{S}} \setminus \{s_{o}\}$ and $s' \in \bar{\mathcal{S}}$ and $\kappa_{B}(s_{o}, s) = \mathbbm{1}[s = s_{o}]$. 
    As such, let $\{S_{t}^{\smash[t]{(b)}}\}_{t \ge 1}$ follow $\kappa_{B}$; if 
    \begin{equation*} 
        T_{\{s_{o}\}} = \min \{t \colon S_{t}^{\smash[t]{(b)}} = s_{o}\}, \text{ then } \mathbb{P}_{\kappa_{B}} \left[ T_{\{s_{o}\}} < \infty | x\right] = 1 \text{ for every } x \in \mathcal{X}.
    \end{equation*} 
\end{proposition}

Under Proposition~\ref{prop:markovsa}, we show that there is a transition kernel $\kappa_{F}^{\smash[t]{(k)}}$ for which (i) $\mathcal{X}_{k}$ is the only absorbing set, (ii) each $x \in \mathcal{X}_{k}$ is an absorbing state being arrived at with probability proportional to $R(x)$, (iii) and that $\kappa_{F}^{\smash[t]{(k)}}$ is a Doob's transformation of $\kappa_{F}$. 
For this, 
let $\{S_{t}^{\smash[t]{(f)}}\}_{t \ge 1}$ be as above and \looseness=-1 
\begin{equation*}
    T_{k} = \min \left\{ t \colon S_{t}^{\smash[t]{(f)}} \in \mathcal{X}_{k} \right\}
\end{equation*}
be the hitting time on $\mathcal{X}_{k}$. 
We 
define $h_{k}(s) = \mathbb{P}_{\kappa_{F}}\left[ T_{k} = \min_{k'} T_{k'} | s \right]$ 
as the probability that the Markov chain $\{S_{t}^{\smash[t]{(f)}}\}_{t \ge 1}$ reaches $\mathcal{X}_{k}$ before any other $\mathcal{X}_{k'}$ when starting at $s$. 
The next lemma shows $h_{k}$ is a harmonic function in the sense of Definition~\ref{def:doobsa}. 

\begin{lemma} \label{lemma:harmonicsa} 
    The functions $h_{k} \colon \bar{\mathcal{S}} \rightarrow \mathbb{R}$, $k \in [K]$, are harmonic in $\mathcal{S}$ with respect to $\kappa_{F}$. 
    They also satisfy $h_{k}(x) = \mathbbm{1}[x \in \mathcal{X}_{k}]$ for $x \in \mathcal{X}$ and, defining $Z_{k}$ and $Z$ as in Corollary~\ref{col:pmassa}, $h_{k}(s_{o}) = \nicefrac{Z_{k}}{Z}$. 
    
    \looseness=-1 
\end{lemma}

Lemma~\ref{lemma:harmonicsa} ensures Doob's $h_{k}$-transform of $\kappa_{F}$ is well-defined. 
We show in the next proposition that the kernels resulting from such transform sample each $x \in \mathcal{X}_{k}$ in proportion to $R(x)$. 
\looseness=-1 

\begin{proposition}[$p_{F}^{\smash[t]{(k)}}$ is a Doob's $h_{k}$-transform of $p_{F}$] \label{prop:doobsaa} 
    Let $\kappa_{F}$ be as in Proposition~\ref{prop:markovsa} and $h_{k} \colon \bar{\mathcal{S}} \rightarrow \mathbb{R}$ be the harmonic function of Lemma~\ref{lemma:harmonicsa} for $k \in [K]$. 
    We define 
        $\kappa_{F}^{\smash[t]{(k)}}(s, s') = \kappa_{F}(s, s') \cdot \nicefrac{h_{k}(s')}{h_{k}(s)} \text{ and } \{S_{t}^{\smash[t]{(k)}}\}_{t \ge 1}$ 
    as a Markov chain following $\kappa_{F}^{\smash[t]{(k)}}$. 
    Then, 
    if $T_{m}^{\smash[t]{(k)}} = \min \left\{ t \colon S_{t}^{\smash[t]{(k)}} \in \mathcal{X}_{m} \right\}$, 
    \begin{equation*}
        \mathbb{P}_{\kappa_{F}^{\smash[t]{(k)}}}\left[ T_{k}^{(k)} < \infty = T_{m}^{(k)} \right] = 1 \text{ for } m \neq k
        \text{ and }  \mathbb{P}_{\kappa_{F}^{\smash[t]{(k)}}}\left[ S_{T_{k}^{(k)}}^{\smash[t]{(k)}} = x \right] = \frac{R(x)}{\sum_{x \in \mathcal{X}_{k}} R(x)}
    \end{equation*}
    for all $x \in \mathcal{X}_{k}$. 
    In particular, the policy function $p_{F}^{\smash[t]{(k)}}$ defined as $p_{F}^{\smash[t]{(k)}}(s, s') = \kappa_{F}^{\smash[b]{(k)}}(s, s')$ for $s \in \mathcal{S}$ and $s' \in \bar{\mathcal{S}}$ satisfies \Cref{eq:marginalss}, i.e., the marginal of $p_{F}^{\smash[t]{(k)}}$ on $\mathcal{X}$ matches $R_{k}$ up to a constant. 
    \looseness=-1 
\end{proposition}

In other words, Proposition~\ref{prop:doobsaa} states $p_{F}^{\smash[t]{(k)}}$ solves SC GFlowNet's learning problem. 


\subsection{Sampling Partitions \& Parallel Learning for SC GFlowNets}

Although Proposition~\ref{prop:doobsaa} 
clearly connects 
each $p_{F}^{\smash[t]{(k)}}$ to a base policy $p_{F}$, the quantity $h_{k}(x)$ cannot be efficiently computed in general even for known $\kappa_{F}$---let alone for the unknown solution to the GFlowNet's balance equations. 
Besides, learning $h_{k}$ from Lemma~\ref{lemma:harmonicsa} can be impractical, as it would severely increase the per-transition computational cost. 
In practice, we parameterize our policies as in Example~\ref{example:linesa} and learn them by minimizing the objective function 
\begin{equation} \label{eq:optmixsa} 
    \mathbb{E}_{k \sim \rho_{M}, \, \tau \sim \rho_{k}} \left[ \mathcal{L}_{\mathrm{TB}}\left( \tau, \{k\} ; \gfn_{[K]} \right) \right].  
\end{equation}
\looseness=-1

\pp{Active Partition Sampling.}
 \Cref{eq:optmixsa} begs the question of how to sample $k$ from $[K]$ during training. 
Intuitively, we should select most often the $k$ for which the probability mass under $\mathcal{X}_{k}$ is largest, while also serendipitously picking the $k$ such that $\mathcal{X}_{k}$ is thought to only include low-probability states. 
As $Z_{k}$ serves as a surrogate for the probability mass under $\mathcal{X}_{k}$ (Corollary~\ref{col:pmassa}), 
this intuition can be naturally operationalized by drawing $k$ according to 
\begin{equation*}
    k \sim \rho_{M} \coloneqq \mathrm{Categorical}\left( (1 - \epsilon) \cdot \nicefrac{\mathbf{Z}}{\mathbf{1}^{\intercal} \mathbf{Z}} + \epsilon \cdot K^{-1} \right) 
\end{equation*}
with $\epsilon \in [0, 1]$ and $\mathbf{1}$ as a $K$-sized vector of ones. 
We emphasize more sophisticated approaches, such as curriculum-based sampling \citep{bengio2009curriculuml, laajil2025curriculumaugmentedgflownetsmrnasequence} focused on high-loss instead of high-mass components, could also be developed; however, we leave this for future research. 
We summarize both the training and inference procedures for SC GFlowNets in Algorithms~\ref{alg:trainingsa} and~\ref{alg:inferencesa}. 
\looseness=-1 

\looseness=-1 




\pp{Embarrassingly Parallel Learning.} 
Besides being trainable as a monolithic model through active partition sampling, SC GFlowNets enable distributed learning of GFlowNets without inter-process communication. 
We call this \emph{embarrassingly parallel learning}: 
we learn $K$ independent models, each of which generates samples from $\mathcal{X}_{k}$ according to $R_{k}(x)$. 
Afterwards, during inference, we pick each model in proportion to its learned partition function to generate samples from $\mathcal{X}$. 
To 
separate SC GFlowNet from their asynchronous variants, we 
refer to the latter as \emph{ASC GFlowNets}.

\begin{figure}
    \begin{minipage}{.7\textwidth}
        \centering
        \includegraphics[width=\textwidth]{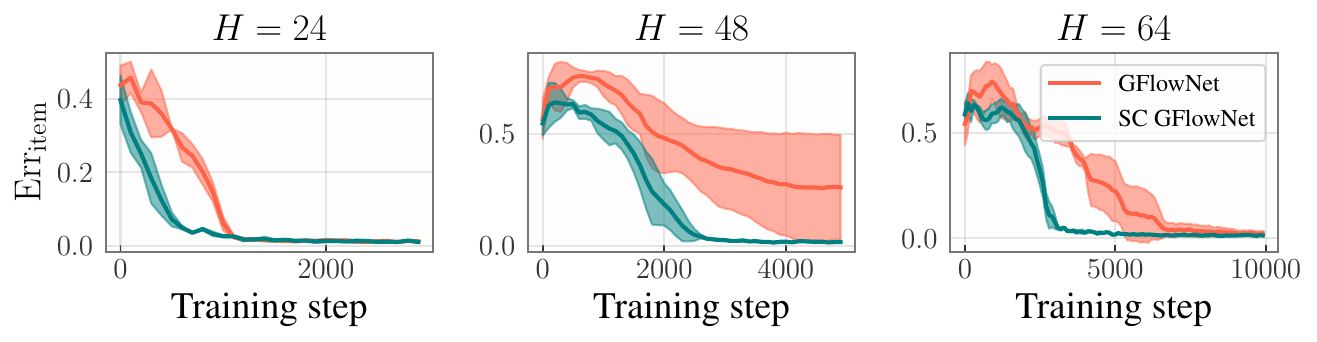}
        \caption{
            \textbf{Length-based SC GFlowNets trains faster than conventional GFlowNets} in the \textsc{Set Generation} domain, requiring less than half the number of iterations to achieve 
            the same accuracy. \looseness=-1 
        }
        \label{fig:setsconvergences}
    \end{minipage} 
    \hspace{12pt} 
    \begin{minipage}{.24\textwidth}
        \centering 
        \includegraphics[width=.9\linewidth]{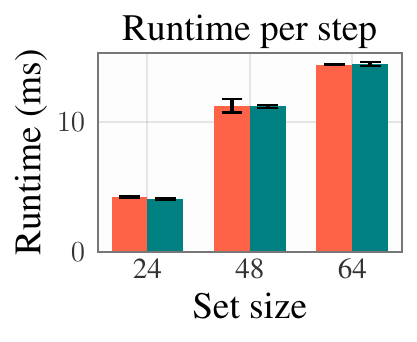}
        \caption{Average wall-clock times are similar for both samplers.} 
        \label{fig:runtimesa} 
    \end{minipage}
    \vspace{-6pt} 
\end{figure}

\begin{figure}
    \begin{subfigure}{.39\textwidth}
        \centering 
        \includegraphics[width=.95\linewidth]{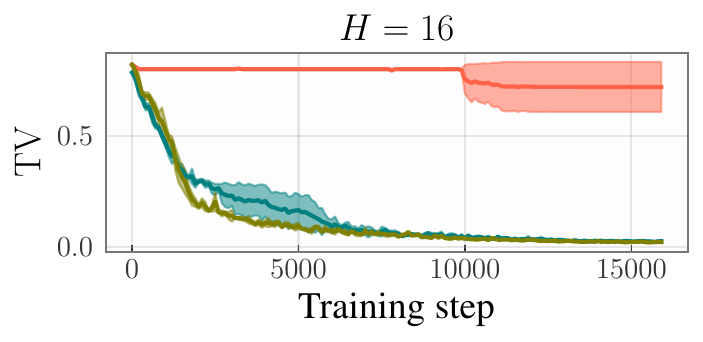}
    \end{subfigure}
    \begin{subfigure}{.39\textwidth}
        \centering 
        \includegraphics[width=.95\linewidth]{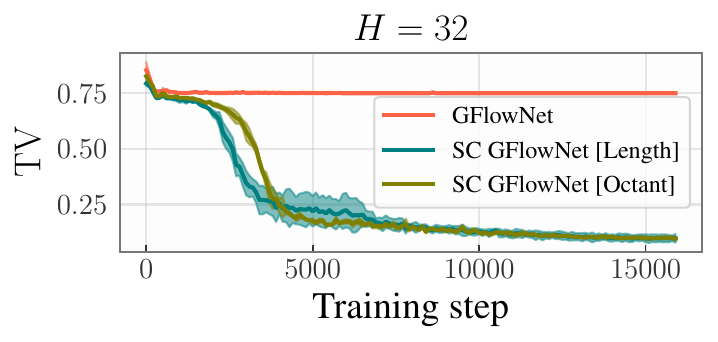}
    \end{subfigure}
    \begin{subfigure}{.19\textwidth}
        \centering 
        \includegraphics[width=.95\linewidth]{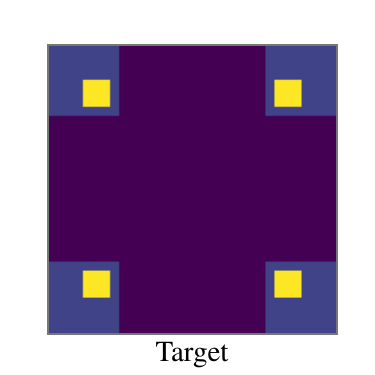}
    \end{subfigure}
    \caption{\textbf{SC GFlowNets converge substantially faster than GFlowNets} in the \textsc{Hypergrid} domain, with distinct partitioning strategies (\texttt{Length} \& \texttt{Octant}) yielding similar benefits.
    } 
    \label{fig:hypergridsaa} 
    \vspace{-12pt} 
\end{figure}

\section{Experiments} \label{sec:experimentsa}  

Our experiments 
evaluate the effectiveness of SC GFlowNets. 
We consider the tasks of \textsc{Set Generation}, \textsc{Hypergrid}, and \textsc{Ancestral Graphs}, 
representing both classical and recent GFlowNet applications with tractable assessment metrics \citep{malkin2022trajectory, malkin2023gflownets, kim2023local, kim2025adaptive}, and defer \textsc{Lazy Random Walk} 
and \textsc{Variable Selection}, along with an evaluation of ASC GFlowNets and a comparison against Boosted GFlowNets \citep{dallantonia2025boostedgflownetsimprovingexploration}, to \Cref{sec:experimentsaa}. \looseness=-1

\pp{Set generation.}
We let $\mathcal{S} = \{s \subseteq [H]\}$ for some $H \in \mathbb{N}$ and $\mathcal{X} = \mathcal{S} \times \{\top\}$ for the state space. 
We also define $u \colon [H] \rightarrow \mathbb{R}_{+}$ and let $R(x) = \prod_{e \in x} u(e)$ be the target distribution, which allows the probability $\pi_{\mathrm{item}}(e)$ of $e$ being a member of $x$ for $e \in [H]$ to be tractably computed. 
We evaluate the model 
via $\mathrm{Err}_{\mathrm{item}} = \frac{1}{2} \sum_{e \in [H]} |\hat{\pi}_{\mathrm{item}}(e) - \pi_{\mathrm{item}}(e)|$, with $\hat{\pi}_{\mathrm{item}}$ as the Monte Carlo estimate of $\pi_{\mathrm{item}}$. 
For SC GFlowNets, we use a length-based partitioning with $K = 3$. 
Results for $H \in \{24, 48, 64\}$ in \Cref{fig:setsconvergences} highlight the benefits of SC GFlowNet over its unpartitioned counterpart. 
In addition, we notice in \Cref{fig:runtimesa} the difference in runtime for both samplers is not significant. 
The main reason for this is that the cost of the extra steps required by the mixture model---active partition sampling and mask computation---is dwarfed by that of policy and reward evaluation.
We observed similar runtime patterns across all our experiments.
\looseness=-1 

\pp{Hypergrid.} 
The state space for \textsc{Hypergrid} is $\mathcal{S} = \{0, \dots, H - 1\}^{d}$ and $\mathcal{X} = \mathcal{S} \times \{\top\}$. 
The initial state is $s_{o} = \mathbf{0}$ and each transition from $s \in \mathcal{S}$ either adds $1$ to one of $s$'s coordinates, respecting the boundaries of $\mathcal{S}$, or interrupts the generative process, producing $x = (s, \top)$. 
The target distribution, described in \Cref{sec:experimentsaa}, is shown in \Cref{fig:hypergridsaa}. 
When designing SC GFlowNets, we consider both a length- ($K = 3$) and octant-based ($K = 4$) partitions, with the latter perfectly separating the four high-probability regions of $\mathcal{X}$; see \Cref{fig:hypergridsaaa} for illustrations.
Results in \Cref{fig:hypergridsaa} confirm SC GFlowNets's effectiveness even when using an imperfectly balanced partitioning. 
\looseness=-1



\begin{figure}
    \begin{subfigure}{.49\textwidth}
        \centering 
        \includegraphics[width=\linewidth]{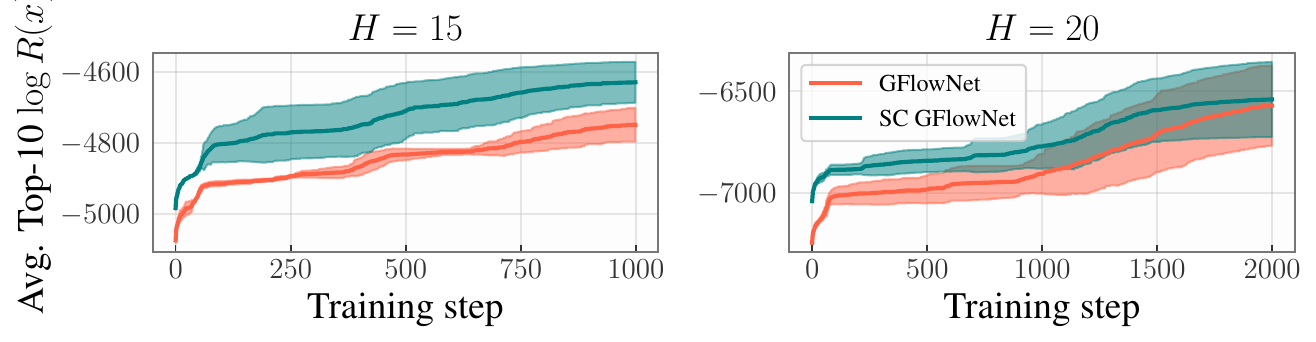}
        \caption{BIC (\Cref{eq:bicsa}).} 
    \end{subfigure}
    \begin{subfigure}{.49\textwidth}
        \centering 
        \includegraphics[width=\linewidth]{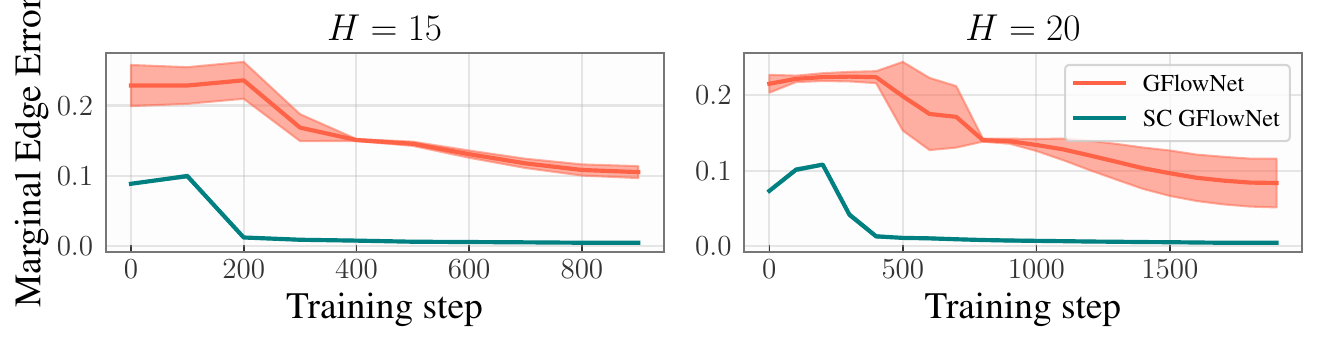}
        \caption{Erd\H{o}s-Renyi model.} 
    \end{subfigure}
    \caption{
        \textbf{SC GFlowNets enhance exploration (left) and speed up learning (right)} for the \textsc{Ancestral Graphs} domain. 
        We denote by $H$ the number of variables (nodes). 
    }
    \label{fig:graphsaa} 
    \vspace{-12pt} 
\end{figure}

\pp{Ancestral graphs.} 
This task 
was studied in \citep{silva2026causal} for addressing the problem of causal discovery under latent confounding.  
In a nutshell, $\mathcal{S}$ contains all \emph{partially edge-labelled} ancestral graphs \citep[AGs]{richardson2002ancestral} with $H$ nodes, while $\mathcal{X}$ contains all \emph{fully edge-labelled} AGs. 
The generative process starts at a graph with unlabelled edges, and each transition labels a node pair with either $\emptyset$, $\rightarrow$, $\leftarrow$, or $\leftrightarrow$; $\emptyset$ represents the absence of an edge. 
The objective is to generate high-scoring AGs $x$ according to the Bayesian information criterion (BIC) based on a dataset modelled as a linear Gaussian structural equation model \citep{drton2004iterative}; we set $R(x) = \exp\{-\mathrm{BIC}(x)\}$. 
As the distributional correctness of such method cannot be tractably measured for all but tiny $H$, we also consider an Erd\H{o}s-Renyi model, for which edge marginals can be efficiently computed using Robinson's formula \citep{robinson1973counting} and compared to that learned by our samplers (see \Cref{sec:experimentsaa}). 
In both cases, we partition $\mathcal{X}$ according to the number of non-$\emptyset$ edges for SC GFlowNets. 
\Cref{fig:graphsaa} shows that our mixture model improves both exploration and convergence. \looseness=-1 

\section{Conclusions} 

The central questions we addressed in this work are why and how to learn a mixture of GFlowNets.  
We first developed a general-purpose framework for answering them, outlining the capabilities and limitations of such approach. 
In particular, we showed SC GFlowNets, which learn an ensemble of samplers restricted to subsets of the state space, provide substantial empirical benefits at a negligible computational cost. 
Future research should investigate better ways of operationalizing our framework and, specifically, develop heuristics for automatic and tractable partitioning of the state space.  
\looseness=-1 


\bibliography{iclr2027/iclr2027_conference}
\bibliographystyle{iclr2027/iclr2027_conference}

\newpage 

\appendix


\section{GFlowNets: Background \& Challenges} \label{sec:backgroundsa} 

We provide a detailed background on GFlowNets, and review current challenges in effectively training these samplers, explaining how they can be naturally addressed by our mixture framework.  

\subsection{Background} 

This section reviews prior work on learning GFlowNets for discrete distributions, following \cite{malkin2023gflownets, theory}, to which we refer for a comprehensive overview. 
\looseness=-1 

\pp{GFlowNets.} 
Let $\mathcal{X}$ be a finite set; we say $\mathcal{X}$ is a \emph{compositional space} if there is a state graph (SG) $\mathcal{G} = (\mathcal{S}, \mathcal{E})$ with vertices $\mathcal{S} \supset \mathcal{X}$ and edges $\mathcal{E}$ having the following properties.

\begin{enumerate}[leftmargin=12pt]
    \item $\mathcal{G}$ is a directed acyclic graph (DAG) with a single source, $s_{o}$, and with $\mathcal{X}$ as the set of sink vertices. We call $s_{o}$ the \emph{initial state}. 
    \item There is a (possibly non-unique) path from from each $s_{o}$ to each $\mathcal{S} \setminus \{s_{o}\}$. 
\end{enumerate}

We further assume, for computational reasons, that the number of children of each $s \in \mathcal{S}$ on $\mathcal{G}$ and the maximum trajectory length in $\mathcal{G}$ are exponentially smaller than $|\mathcal{X}|$. 
(Otherwise, every finite set $\mathcal{X}$ would be compositional by considering $\mathcal{S} = \{s_{o}\} \cup \mathcal{X}$ and $\mathcal{E} = \{s_{o} \rightarrow x \colon x \in \mathcal{X}\}$). 
Our objective is to generate $x \in \mathcal{X}$ in proportion to a positive function $R \colon \mathcal{X} \rightarrow \mathbb{R}_{+}$. 
\looseness=-1 


\pp{Sampling as learning.}
Towards this objective, we define $p_{F} \colon \mathcal{S} \times (\mathcal{S} \cup \mathcal{X}) \rightarrow [0, 1]$ (resp. $p_{B} \colon (\mathcal{S} \cup \mathcal{X}) \times \mathcal{S} \rightarrow [0, 1]$) as the \emph{forward} (resp. \emph{backward}) policy on $\mathcal{G}$ representing the probability $p_{F}(s, s')$ of transitioning from $s$ to a child $s'$ of $s$. 
The reason for $p_{B}$'s existence will become clearer shortly. 
We also define $s \rightsquigarrow s'$ as the set of trajectories in $\mathcal{G}$ from $s$ to $s'$. 
Our objective is for the marginal probability of $p_{F}$ over $\mathcal{X}$, when starting at $s_{o}$, to match $R$, i.e., 
\begin{equation} \label{eq:marginals}
    m_{F}(x) \coloneqq \displaystyle \sum_{\substack{\tau \in s_{o} \rightsquigarrow x \\ \tau = \{(s_{o}, s_{1}), \dots, (s_{K}, x)\}}} \prod_{(s, s') \in \tau} p_{F}(s, s') \coloneqq \sum_{\tau \in s_{o} \rightsquigarrow x} p_{F}(s_{o}, \tau) \propto R(x). 
\end{equation}
The trajectory length $K$ in the equation above need not be the same for each $\tau$.
Importantly, we emphasize this equation merely integrates out from $p_{F}$ intermediate variables from the joint distribution $(s_{o}, \dots, x)$ over trajectories. 
Due to \Cref{eq:marginals}'s intractability, we estimate the summation above using $p_{B}$ as a proposal for an importance-sampling-like scheme, that is, denoting $Z = \sum_{x \in \mathcal{X}} R(x)$, 
\begin{equation*}
    \mu_{F}(s_{o}, x) = \mathbb{E}_{\tau \sim p_{B}(x, \cdot)} \left[ \frac{ p_{F}(s_{o}, \tau) } {p_{B}(x, \tau) } \right] = \frac{R(x)}{Z}, 
\end{equation*}
with $p_{B}(x, \tau) = \prod_{(s, s') \in \tau} p_{B}(s', s)$. The condition $\nicefrac{p_{F}(\tau | s_{o})}{p_{B}(\tau | x)} = \nicefrac{R(x)}{Z}$ is known as \emph{trajectory balance} (TB) \cite{malkin2022trajectory}. 
To search for policies satisfying it, we parameterize $p_{F}$ and (optionally) $p_{B}$ as 
neural networks and solve the stochastic program
\begin{equation} \label{eq:stochasticsa} 
    \min_{p_{F}, p_{B}} \mathcal{L}(p_{F}, p_{B} | \rho) \coloneqq \mathbb{E}_{\tau \sim \rho} \left( \log \frac{Z \cdot p_{F}(\tau | s_{o})}{p_{B}(\tau | x) \cdot R(x)} \right)^{2}  
\end{equation}
via stochastic gradient descent. 
As the reader may have noticed, 
$p_{B}$ was introduced arbitrarily; in general, there are infinitely many pairs $(p_{F}, p_{B})$ satisfying the TB condition, a fact that will come up later in \Cref{sec:modelsa} when we discuss mixtures of amortized samplers. 
Additional learning objectives, with distinct parameterizations and properties, have been proposed and extensively studied \citep{madan2022learninggf, malkin2023gflownets, tiapkin2024generative, pan2023better, hu2026squarederrorexploringloss}. 
In practice, we let $\theta$ be the (joint) parameters of $p_{F}$, $p_{B}$, and $\log Z$, $\{\alpha_{t}\}_{t \ge 0}$ be a sequence of step sizes, and iteratively minimize 
\begin{equation*}
    \theta^{\smash[t]{(t)}} = \min_{\theta} \theta^{\intercal} \nabla_{\theta^{\smash[t]{(t - 1)}}} \hat{\mathcal{L}}(\theta | \theta^{\smash[t]{(t - 1)}}) + \frac{1}{2 \alpha_{t}} \| \theta - \theta^{\smash[t]{(t - 1)}}\|_{2}^{2} 
\end{equation*}
(or a preconditioned variant of it \citep{khan2018fastsimplenaturalgradientdescent}), with $\hat{\mathcal{L}}(\theta | \theta^{\smash[t]{(t - 1)}})$ as an Monte Carlo estimate of $\mathcal{L}(p_{F}, p_{B} | \rho^{\smash[t]{(t - 1)}})$ when both $p_{F}$ and $p_{B}$ are instantiated according to $\theta$ and $\rho^{\smash[t]{(t - 1)}}$ as a function of $\theta^{\smash[t]{(t - 1)}}$ \citep{tiapkin2024generative, kim2025adaptive, madan2025towards, gritsaev2025optimizingbackwardpoliciesgflownets, dallantonia2026avoidknowdivergenttrajectory}. 
For example, the popular $\epsilon$-greedy exploration \citep{bengio2021, malkin2022trajectory, pan2023generative, zhou2024phylogfn, kim2025symmetryawaregflownets, dallantonia2025boostedgflownetsimprovingexploration} uses $\rho^{\smash[t]{(t - 1)}}(s, \cdot) = (1 - \epsilon) \cdot p_{F}^{\smash[t]{(t - 1)}}(s, \cdot) + \epsilon \cdot p_{U}(s, \cdot)$ 
with $p_{U}(s, \cdot)$ as the policy function which selects the children of $s$ uniformly at random and $p_{F}^{\smash[t]{(t - 1)}}(s, \cdot)$ as the $p_{F}$ corresponding to $\theta^{\smash[t]{(t - 1)}}$.  
\looseness=-1 

\subsection{Challenges in training GFlowNets} \label{sec:modelsa} 



To further motivate our work, we briefly describe current challenges in learning GFlowNets, and later argue they can be partly addressed through a mixture formulation of the sampling problem.

\begin{wrapfigure}[14]{l}{.3\textwidth}
    \vspace{-12pt} 
    \includegraphics[width=.9\linewidth, page=3]{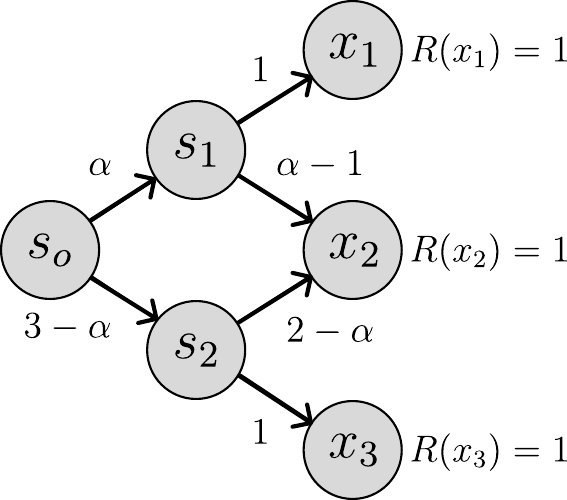}
    \caption{Any $\alpha \in [1, 2]$ solves the assignment problem in this state graph. The solution 
    for uniform $p_{B}$ is $\alpha = 1.5$.}
    \label{fig:alphasa} 
\end{wrapfigure}
The first challenge is that the optimization problem described in \Cref{eq:stochasticsa} is \emph{underdetermined}, as there is no pair $(p_{F}, p_{B})$ uniquely satisfying the TB condition in general (unless there is a single trajectory from $s_{o}$ to each $x$, as in autoregressive sequence generation tasks \cite{sequence}).
This is illustrated in \Cref{fig:alphasa}, which displays a family of $p_{F}$ for which the marginal constraint in \Cref{eq:marginals} is satisfied, highlighting the existence of a \emph{continuum} of solutions to the sampling problem. \looseness=-1  

\begin{challenge}[Underdetermination] \label{ch:uda}
    There are infinitely many pairs $(p_{F}, p_{B})$ for which TB is satisfied and the marginal of $p_{F}$ on $\mathcal{X}$ matches $R$, and it is unclear how to choose between them.
\end{challenge}

\begin{wrapfigure}[11]{r}{.4\textwidth}
    \vspace{-12pt} 
    \includegraphics[width=\linewidth]{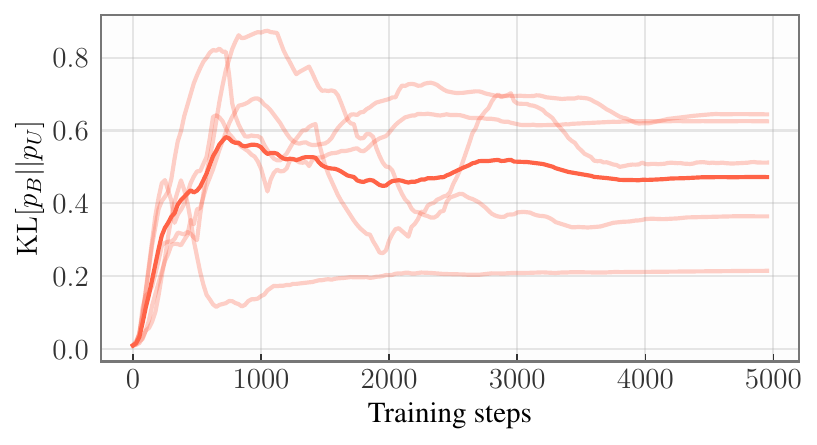}
    \caption{KL between $p_{B}$ and $p_{U}$ for the \textsc{Hypergrid} domain.}
    \label{fig:hypergridpbsa} 
\end{wrapfigure}
A common choice, inspired by the maximum entropy principle \citep{jaynes2003probability}, is to set $p_{B}(s, \cdot)$ as a uniform probability distribution over the parents of $s$ in the SG \citep{shen23gflownets}. 
Alternatively, both $p_{F}$ and $p_{B}$ can be jointly learned \citep{gritsaev2025optimizingbackwardpoliciesgflownets}. 
However, deriving a $p_{B}$ for which finding the corresponding $p_{F}$ is easiest is not possible in general, and it is unclear which pair $(p_{F}, p_{B})$ does gradient-based learning induce. 
We illustrate this in \Cref{fig:hypergridpbsa} for the \textsc{Hypergrid} domain, showing the variability of $\mathrm{KL}(p_{B}(s, \cdot) || p_{U}(s, \cdot))$ averaged over states $s \in \mathcal{S} \setminus \{s_{o}\}$ during training for 5 independent runs (the highlighted curve represents per-step averages). 
As we can observe, the learned $p_{B}$ vary drastically solely by 
modifying 
the pseudorandom seed for sample generation---while maintaining both the parameter initialization and optimizer unchanged. 
\looseness=-1 

A mixture model, as we will show, naturally addresses Challenge~\ref{ch:uda} by learning a \emph{family of backward policies}---hence avoiding strict commitment to a single, possibly unknown $p_{B}$. 
However, even when $p_{B}$ is uniquely defined---e.g., in autoregressive sequence generation \citep{sequence}---the underlying neural network might be incapable of finding an approximation 
to solution of the stochastic program in \Cref{eq:stochasticsa}. 
This problem, known as near state aliasing \citep{pardo2022timelimitsreinforcementlearning}, is particularly harmful for graph-structured domains \citep{silva2025when, kim2025symmetryawaregflownets}. 

\begin{challenge}[(Near) state aliasing] \label{ch:sia} 
    A solution $(p_{F}, p_{B})$ to \Cref{eq:stochasticsa} might not be (accurately) realizable by the chosen family of parametric models.    
\end{challenge}

In \Cref{sec:continuousa}, 
we show how our method alleviates it by establishing a relationship between Random Network Initialization (RNI) in Graph Neural Networks (GNNs) and mixture modelling. 
Even an universal approximator, however, may struggle to search for a compatible pair $(p_{F}, p_{B})$ simply due to the sheer size of the state space. 
The main reason for this is that, to accurately approximate the target distribution, the model has to reach its high-probability regions during training \cite{kim2023local, kim2025adaptive, dallantonia2026avoidknowdivergenttrajectory}
, an obstacle that typical $\epsilon$-greedy strategies fail to overcome.    

\begin{challenge}[Insufficient exploration] \label{ch:epa} 
    Accurately learning $(p_{F}, p_{B})$ requires visiting the target $R$'s high-probability regions during training, which are often diverse and sparsely distributed. 
\end{challenge}

To address Challenge~\ref{ch:epa}, recent works have proposed 
using boosting-style techniques \citep{dallantonia2025boostedgflownetsimprovingexploration} 
and 
artificial curiosity-inspired \cite{schmidhuber2020generative} 
\emph{exploratory} models---which are discarded after training. 
These approaches, however, increase the per-step computational cost of conventional samplers by severalfold, requiring the learning and evaluation of one \cite{kim2025adaptive, madan2025towards, dallantonia2026avoidknowdivergenttrajectory} or many additional models \cite{silva2025generalization, dallantonia2025boostedgflownetsimprovingexploration} throughout learning and inference.
In contrast, in \Cref{sec:experimentsa}, we showed improved exploration is achievable by a \emph{mixture} of GFlowNets---without relying on additional, independently parameterized costly models. 
\looseness=-1 

\section{Mixture of GFlowNets} \label{sec:mixturesa} 

We outline additional results on the learning of a mixture of GFlowNets: an extension of \cite{madan2022learninggf}'s SubTB condition and the associated family of loss functions. We start with the former. 
\looseness=-1 

\begin{proposition}[SubTB] \label{prop:subtbsa} 
    For each $s \in \mathcal{S} \cup \mathcal{X}$, let $F(s, \mathrm{d}\gamma)$ be a measure on $\Gamma$ such that $F(s, \mathrm{d}\gamma) = R(x, \mathrm{d}\gamma)$ for $x \in \mathcal{X}$. 
    We call $F$ a \emph{flow function.} 
    If, for $\tau \in s \rightsquigarrow s'$, \looseness=-1 
    \begin{equation*}
        F(s, \mathrm{d}\gamma) p_{F}^{\smash[t]{(\gamma)}}(s, \tau) = F(s', \mathrm{d}\gamma) p_{B}^{\smash[t]{(\gamma)}}(s', \tau), 
    \end{equation*}
    then \Cref{eq:marginalss} is satisfied. 
    In addition, for each $s \in \mathcal{S}$, 
    \begin{equation*}
        Z(\mathrm{d}\gamma) = F(s_{o}, \mathrm{d}\gamma) \text{ and } F(s, \mathrm{d}\gamma) = \sum_{x \in \mathcal{X}} \sum_{\tau \in s \rightsquigarrow x} R(x, \mathrm{d}\gamma) p_{B}^{\smash[t]{(\gamma)}}(x, \tau). 
    \end{equation*}
\end{proposition}

We provide a demonstration for this and every other statement in the supplement in \Cref{sec:psaa}, jointly with the statements in the main text. 
From Proposition~\ref{prop:subtbsa}, the usual loss functions encountered in the GFlowNet literature can be readily derived. 

\begin{definition}[SubTB \& DB \& TB losses] \label{def:lossesa} 
    Let $Z(\cdot)$ and $F(s, \cdot)$ and $R(x, \cdot)$ be measures on $\Gamma$, as in Proposition~\ref{prop:subtbsa}, for $s \in \mathcal{S}$ and $x \in \mathcal{X}$, and $p_{F}^{\smash[t]{(\gamma)}}$ and $p_{B}^{\smash[t]{(\gamma)}}$ be $\Gamma$-indexed policies.
    We concisely let $\gfn_{\Gamma} = (Z, F, R, \{p_{F}^{\smash[t]{(\gamma)}}, p_{B}^{\smash[t]{(\gamma)}}\}_{\gamma \in \Gamma})$. 
    Given $\tau = (s_{o}, \dots, s_{N})$ with $s_{N} = x \in \mathcal{X}$, we define $\tau_{i:j} = (s_{i}, s_{i + 1}, \dots, s_{j})$ and  
    \begin{equation*}
        \mathcal{L}_{\mathrm{SubTB}}(\tau, \mathrm{d} \gamma ; \gfn_{\Gamma}) = \sum_{0 \le i < j \le N} \alpha_{ij} \left( \log \frac{F(s_{i}, \mathrm{d}\gamma) p_{F}^{\smash[t]{(\gamma)}}(s_{i}, \tau_{i:j})}{F(s_{j}, \mathrm{d}\gamma) p_{B}^{\smash[t]{(\gamma)}}(s_{j}, \tau_{j:i})} \right)^{2} 
    \end{equation*}
    as the \emph{SubTB} loss, with $(\alpha_{ij})_{i, j = 0}^{N}$ as positive numbers with $\sum_{j > i} \alpha_{ij} = 1$.
    Similarly, 
    \begin{equation*}
        \mathcal{L}_{\mathrm{DB}}(\tau, \mathrm{d} \gamma ; \gfn_{\Gamma}) = \frac{1}{N} \sum_{0 \le i \le N - 1} \left(\log \frac{F(s_{i}, \mathrm{d}\gamma) p_{F}^{\smash[t]{(\gamma)}}(s_{i}, s_{i + 1})}{F(s_{i + 1}, \mathrm{d}\gamma) p_{B}^{\smash[t]{(\gamma)}}(s_{i + 1}, s_{i})}\right)^{2} 
    \end{equation*}
    is the \emph{DB} loss.
    In conclusion, we let the \emph{TB} loss be  
    \begin{equation*}
        \mathcal{L}_{\mathrm{TB}}(\tau, \mathrm{d} \gamma ; \gfn_{\Gamma}) = \left( \frac{Z(\mathrm{d}\gamma) p_{F}^{\smash[t]{(\gamma)}}(s_{o}, \tau)}{R(x, \mathrm{d}\gamma) p_{B}^{\smash[t]{(\gamma)}}(x, \tau)} \right)^{2}.  
    \end{equation*}
\end{definition}

\section{Continuously Indexed GFlowNets} \label{sec:continuousa}

In this section, we elaborate on the design of a continuously indexed mixture of GFlowNets. 
As we will show, this approach can drastically boost the expressivity of GNN-parameterized samplers \citep{kim2025symmetryawaregflownets} without requiring expensive heuristics such as look-ahead search, as proposed in \citep{silva2025when}.  
In doing so, we designedly use a \emph{collapsed} mixing posterior distribution---$q(\mathrm{d}\gamma | x) = q(\mathrm{d}\gamma)$ for every $x \in \mathcal{X}$---an approach we refer to as \emph{random features} (RFs) for GFlowNets.
The reason for this is that either $Z(\mathrm{d}\gamma)$ or $q(\mathrm{d}\gamma | x)$ are generally intractable measures when $\Gamma$ is uncountable, as we also demonstrate.

\pp{Continuous mixtures of GFlowNets.} 
We henceforth consider $\Gamma = \mathbb{R}^{d}$ and $\nu$ as the Lebesgue measure, and we assume both $q_{M}(\mathrm{d}\gamma | x)$ (resp. $R(x, \mathrm{d}\gamma)$, $F(s, \mathrm{d}\gamma)$) and $p_{M}(\mathrm{d}\gamma)$ (resp. $Z(\mathrm{d}\gamma)$) are absolutely continuous with respect to it.
As such, we write $R(x, \gamma)$ and $F(s, \gamma)$ and $Z(\gamma)$ for their corresponding densities. 
It is an obvious consequence of Proposition~\ref{prop:subtbsa} that the balance conditions can be written as a function of these densities instead of the previously introduced measures. 
We highlight this in the following corollary. 

\begin{remark}[SubTB] \label{remark:subtbsa} 
    When $\Gamma = \mathbb{R}^{d}$ and $F(s, \mathrm{d}\gamma)$ and $Z(\mathrm{d}\gamma)$ are absolutely continuous with respect to the base Lebesgue measure $\nu$, the SubTB condition is 
    \begin{equation*}
        F(s, \gamma) p_{F}^{\smash[t]{(\gamma)}}(s, \tau) = F(s', \gamma) p_{B}^{\smash[t]{(\gamma)}}(s', \tau) 
    \end{equation*}
    for each $s, s' \in \mathcal{S} \cup \mathcal{X}$ and $\tau \in s \rightsquigarrow s'$ and $\nu$-almost surely for $\gamma \in \Gamma$. 
\end{remark}

As per Remark~\ref{remark:subtbsa}, we learn $F(s, \gamma)$ as a real-valued neural network-parameterized positive function 
receiving both $s \in \mathcal{S}$ and $\gamma \in \Gamma$ as inputs; recall $F(x, \gamma) = R(x, \gamma)$ for $x \in \mathcal{X}$.
The core issue in doing so, however, is that we may not be able to sample from the prior mixing distribution during inference when the function $Z(\gamma)$ is defined arbitrarily.
Conversely, if we let $p_{M}(\gamma)$ have a simple form (e.g., Gaussian) and $q_{M}(\gamma | x)$ be flexibly learned, the latter distribution may collapse into the former. 
This can be understood via the following corollary of Proposition~\ref{prop:subtbsa}. \looseness=-1 
\looseness=-1 

\begin{corollary}[Relationship between $p_{M}$ and $q_{M}$] \label{col:cmarginalsa} 
    Assume $p_{M}, q_{M}, p_{F}^{\smash[t]{(\gamma)}}, p_{B}^{\smash[t]{(\gamma)}}$ satisfy the SubTB condition in Remark~\ref{remark:subtbsa}. 
    Then, 
    \begin{equation*}
        p_{M}(\gamma) = \sum_{x \in \mathcal{X}} \frac{R(x)}{Z} \cdot  q_{M}(\gamma|x) \text{ and } q_{M}(\gamma | x) = \frac{Z \cdot p_{M}(\gamma) \cdot m_{F}^{\smash[t]{(\gamma)}}(x)}{R(x)}.   
    \end{equation*}
\end{corollary}

The first equation in Corollary~\ref{col:cmarginalsa} shows that, if $q_{M}(\gamma|x)$ has a tractable, non-trivial form, then the 
solution for $p_{M}(\gamma)$ is a combinatorial mixture of all $q_{\gamma}(\gamma|x)$'s for $x \in \mathcal{X}$ weighted by $R(x)$, which is the distribution we are aiming to sample from in the first place. 
The second equation, on the other hand, suggests that if $p_{F}^{\smash[t]{(\gamma)}}(s, \cdot)$ is too flexible a function of $\gamma$ and $s$---able to perfectly mimic the target $R(x)$---then the optimal form for $q_{M}(\gamma|x)$ reduces to $p_{M}(\gamma)$. 
As we also discussed in \Cref{sec:discretesa}, 
reduction is not necessarily undesirable, as we will show later; however, we believe a $x$-independent $q_{M}$ should be an algorithmic design decision rather than an implicit consequence of training.
Should the goal be to have meaningfully different $p_{F}^{\smash[t]{(\gamma)}}$ for each $\gamma$ with the objective of, for instance, improving exploration, Corollary~\ref{col:cmarginalsa} suggests that sensible structural constraints should be imposed on the relationship between $p_{F}^{\smash[t]{(\gamma)}}$ and $\gamma$, a venue we pursued in \Cref{sec:discretesa}. 




\pp{Random Features (RFs) for GFlowNets.} 
In this scenario, we describe the potential benefits of setting $q_{M}(\gamma | x) = p_{M}(\gamma)$ by connecting the ensuing framework with the theory of RFs for GNNs \cite{abboud2021surprising, sato2021random}.
We first recall the expressivity limitations of message passing GNN-parameterized GFlowNets in the following example. \looseness=-1  

\begin{example}[Limitations of GNN-based GFlowNets] \label{example:graphsa} 
    We consider the following state graph (\Cref{fig:graphstategraphs}) with initial state $s_{o}$ being followed by $x_{1}$ and $x_{2}$, each of them a $6$-node graph.
    We notice that $x_{1}$ and $x_{2}$ correspond to $x_{o}$ with additional edges $(a, b)$ and $(a, c)$, respectively, and we write $x_{1} = s_{o} \cup \{(a, b)\}$ and $x_{2} = s_{o} \cup \{(a, c)\}$. 
    This is the simplest setting a GFlowNet can be applied to. \looseness=-1  
    \begin{figure}[!h]
        \centering
        \includegraphics[width=.7\linewidth, page=2]{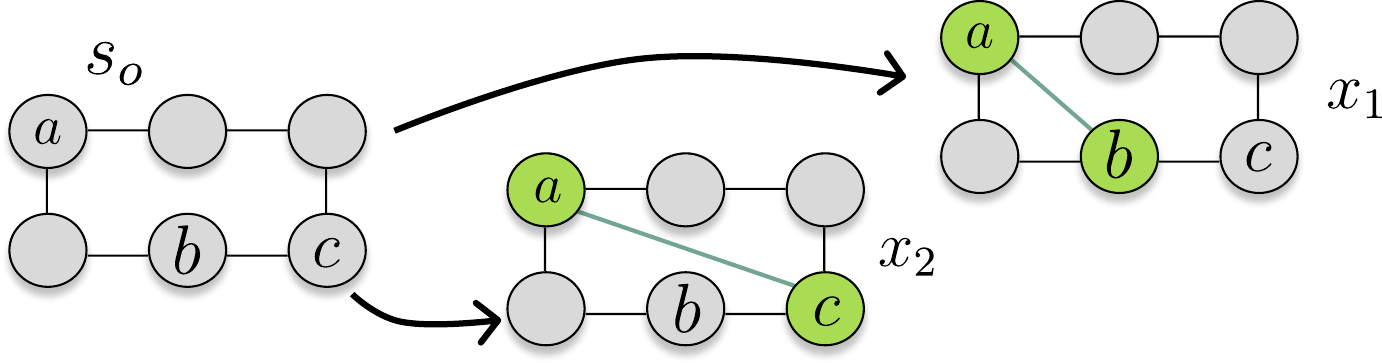}
        \caption{State graph for Example~\ref{example:graphsa}.}
        \label{fig:graphstategraphs}
    \end{figure}
    
    As in \citep{silva2025when}, we let $h \colon (\mathcal{G}, n) \mapsto \mathbb{R}^{d}$ be a 1-WL GNN \cite{kipf2016semi, xu2018powerful} yielding a representation for a node $n$ in a graph $\mathcal{G}$ and $\sigma \colon \mathbb{R}^{d} \rightarrow \mathbb{R}$ be a learnable real-valued function. 
    We parameterize the policy function according to
    \begin{equation} \label{eq:graphsa} 
        p_{F}(s, s \cup \{(n_{1}, n_{2})\}) \propto \exp\left\{ \sigma ( h(s, n_{1}) + h(s, n_{2})) \right\} 
    \end{equation}
    for nodes $n_{1}, n_{2}$ such that there is no edge between $n_{1}, n_{2}$ in $s$, and $0$ otherwise. 
    Under these conditions, unless the target $R$ satisfies $R(x_{1}) = R(x_{2})$, there is no instantiation of $p_{F}$ for which the marginal distribution of $p_{F}$ over $\{x_{1}, x_{2}\}$ matches $R$---as $h$ cannot distinguish node $b$ from node $c$. \looseness=-1   
\end{example}


When augmenting $h$ with 
random node-wise features $\gamma$, however, the policy function in \Cref{eq:graphsa} acquires the capability of approximating \emph{any target} 
with high-probability over the draw of $\gamma$. 
The intuition is that a RF-augmented 1-WL GNN will only be unable to separate nodes $b$ and $c$ in \Cref{fig:graphstategraphs} if their corresponding features match---an event of null probability since $\Gamma$ is continuous. 
\looseness=-1 

\begin{example}[RFs boost the expressivity of GNN-based GFlowNets] 
    We let $h \colon (\mathcal{G}, n, \gamma_{n}) \mapsto \mathbb{R}^{d}$ be a RF-augmented 1-WL GNN with $\gamma_{n} \sim \mathcal{N}(0, 1)$. 
    Then, if 
    \begin{equation} \label{eq:graphsaaa} 
        p_{F}^{\smash[t]{(\gamma)}}(s, s \cup \{(n_{1}, n_{2})\}) \propto \exp\left\{ \sigma ( h(s, n_{1}, \gamma_{n_{1}}) + h(s, n_{2}, \gamma_{n_{2}})) \right\} 
    \end{equation}
    as in \Cref{eq:graphsa}, it follows from \cite[Theorem 1]{abboud2021surprising} that for each $\epsilon > 0$ and $\delta \in (0, 1)$ there is a GNN $h$ such that $p_{F}$ approximates a given target with accuracy $\epsilon$ and probability $1 - \delta$ over the drawing of $\gamma$.  
\end{example}

More broadly, we show in the next 
that a GFlowNet parameterized as in \Cref{eq:graphsaaa} can arbitrarily approximate with high probability any given distribution defined on an edge-additive generative process for graph-structured objects. 
For this, we define $\mathrm{TV}(p, q)$ as the total variation distance between (possibly unnormalized) probability distributions $p$ and $q$. 
\looseness=-1 

\begin{proposition}[RFs GNN-based GFlowNets are Universal Approximators] \label{prop:universalsa} 
    Let $\{(p_{F}^{\smash[t]{(\gamma)}}, p_{B}^{\smash[t]{(\gamma)}})\}_{\gamma \in \Gamma}$ be a GFlowNet augmented with RFs $\gamma \in \Gamma$ and defined over a generative process over graphs such that each transition corresponds to either adding an edge or stopping.
    We define the following hypothesis classes. 
    \begin{enumerate} 
        \item $\mathcal{H}_{\mathrm{t}} = \{ h \colon (\mathcal{G}, n, \gamma_{n}) \mapsto \mathbb{R}^{d}, \, \sigma \colon \mathbb{R}^{d} \rightarrow \mathbb{R}, \, d \in \mathbb{N}\}$ containing 1-WL GNNs $h$ and a real-valued function $\sigma$. \looseness=-1 
        \item $\mathcal{H}_{\mathrm{s}} = \{\xi \colon (\mathcal{G}, \{\mathbf{x}_{n} \colon n \in \mathcal{G}\}) \mapsto y\}$ containing readout functions mapping a graph and its node embeddings to $y \in \mathbb{R}$. 
    \end{enumerate} 
    We parameterize $p_{F}^{\smash[t]{(\gamma)}}$ for edge addition, respectively, as in \Cref{eq:graphsaaa}, assigning also a probability proportional to $\exp\{ \xi( s, \{h(s, n) \colon n \in s\})\}$ for interrupting the generation at $s \in \mathcal{S}$, and fix $p_{B}^{\smash[t]{(\gamma)}}$. 
    Then, for any target distribution $R$, there is a $h \in \mathcal{H}_{\mathrm{t}}$ and a $\xi \in \mathcal{H}_{\mathrm{s}}$ such that the marginal distribution $m_{F}^{\smash[t]{(\gamma)}}$ satisfies $\mathrm{TV}(m_{F}^{\smash[t]{(\gamma)}}, R) \le \epsilon$ with probability larger than $1 - \delta$ over the draw of $\gamma$. \looseness=-1 
\end{proposition}

With this in mind, a thorough investigation of the benefits of RF GFlowNets is left for future research. 
An extensive analysis of how to usefully define $p_{M}$ is of particular importance.
\looseness=-1 


\pp{RFs \& Stability.} 
Besides improving the expressivity of GNN-based GFlowNets, we also demonstrate RFs impose an implicit Tikhonov regularization \citep{hastie2009elements} on the learning objective that mitigate the instability of $p_{B}$ observed in \Cref{fig:hypergridpbsa} in \Cref{sec:modelsa}. 
To understand this, we recall that, if $L(\theta)$ is a loss function associated with parameters $\theta$ and $\theta^{\star}$ is a local minimum of $L$, 
\begin{equation*}
    L(\theta) = L(\theta^{\star}) + \frac{1}{2} (\theta - \theta^{\star})^{\top} H(\theta^{\star}) (\theta - \theta^{\star}) + \mathcal{O}(\|\theta - \theta^{\star}\|^{3})
\end{equation*}
with $H(\theta^{\star}) = \nabla_{\theta}^{2} L(\theta^{\star})$ representing $L$'s Hessian at $\theta^{\star}$. 
From this, we observe that $H(\theta^{\star})$ characterizes the local geometry of $L$ around $\theta^{\star}$; when $H(\theta^{\star})$ is nearly singular, the loss surface around $\theta^{\star}$ is almost flat, a configuration that first-order gradient-based methods struggle to navigate. 
We next show that incorporating RFs into the policies of a GFlowNet has the effect of increasing the eigenvalues of the loss function's Hessian, reducing the geometry's flatness and---as we will empirically illustrate---reducing the functional variability of the learned $(p_{F}, p_{B})$ originating solely from parameter initialization. 
We start with the following lemma. 

\begin{lemma}[Local Approximation for RF-augmented Loss] \label{lemma:localsa} 
    We define $\gfn_{\Gamma} = \{p_{F}^{\smash[t]{(\gamma)}}, p_{B}^{\smash[t]{(\gamma)}}\}_{\gamma \in \Gamma}$ as a GFlowNet augmented with RFs $\gamma$. 
    We assume $\gamma$ is drawn from a zero-centered distribution $p_{M}$ with covariance matrix $\sigma^{2} I$ for some $\sigma > 0$. 
    Then, \looseness=-1 
    \begin{equation} \label{eq:localsa} 
        \mathbb{E}_{\gamma \sim p_{M}} \left[ \mathcal{L}_{\mathrm{TB}}(\tau, \gamma ; \gfn_{\Gamma}) \right] = \mathcal{L}_{\mathrm{TB}}(\tau, 0 ; \gfn_{\Gamma}) + \frac{1}{2} \sigma^{2} \cdot \mathrm{tr} \left( H_{\gamma}(\tau, 0) \right) + \mathcal{O}(\sigma^{4}),  
    \end{equation}
    with $\mathrm{tr}$ being the trace operator, $H_{\gamma}$ representing the Hessian of $\mathcal{L}_{\mathrm{TB}}(\tau, \gamma ; \gfn_{\Gamma})$ with respect to $\gamma$, and $\mathcal{L}_{\mathrm{TB}}$ as in Definition~\ref{def:lossesa}. 
\end{lemma}

Lemma~\ref{lemma:localsa} follows from a Taylor approximation of the loss function around $\gamma = 0$ and the fact that the odd moments of $\gamma$ null out. 
Notably, the quantity $\mathcal{L}_{\mathrm{TB}}(\tau, 0 ; \gfn_{\Gamma})$ corresponds to the loss of an unaugmented GFlowNet. 
By averaging $\tau$ out of \Cref{eq:localsa},  we show that the trace term within it acts as a regularizer. 
\looseness=-1

\begin{proposition}[RFs as Regularizers] \label{prop:regularizationsa} 
    Let $\gfn_{\Gamma}$ be as in Lemma~\ref{lemma:localsa} and $\gfn_{o}$ be the GFlowNet corresponding to $\gamma = 0$. 
    We define 
    \begin{equation*}
        \mathcal{L}_{\mathrm{TB}}(\gfn_{\Gamma}) = \mathbb{E}_{\tau \sim \rho} \mathbb{E}_{\gamma \sim p_{M}} \left[ \mathcal{L}_{\mathrm{TB}}(\tau, \gamma ; \gfn_{\Gamma}) \right], \, \mathcal{L}_{\mathrm{TB}}(\gfn_{o}) = \mathbb{E}_{\tau \sim \rho} \left[ \mathcal{L}_{\mathrm{TB}}(\tau, 0 ; \gfn_{\Gamma}) \right],   
    \end{equation*}
    and $\mathcal{L}_{\mathrm{Reg}}(\gfn_{\Gamma}) = \mathbb{E}_{\tau \sim \rho} \left[ \mathrm{tr} ( H_{\gamma} (\tau, 0) ) \right]$ for an exploratory policy $\rho$. 
    Also, let $\theta$ be the parameters of 
    $\gfn_{\Gamma}$, and 
    let $H_{\mathrm{TB}} = \nabla_{\theta}^{2} \mathcal{L}_{\mathrm{TB}}(\gfn_{\Gamma})$, $H_{o} = \nabla_{\theta}^{2} \mathcal{L}_{\mathrm{TB}}(\gfn_{o})$, and $H_{\mathrm{Reg}} = \nabla_{\theta}^{2} \mathcal{L}_{\mathrm{Reg}}(\gfn_{\Gamma})$. 
    In this setting, 
    \begin{equation} \label{eq:regularizationsa} 
        H_{\mathrm{TB}} = H_{o} + \frac{\sigma^{2}}{2} \cdot H_{\mathrm{Reg}} + \mathcal{O}(\sigma^{4}). 
    \end{equation}
\end{proposition}

\begin{wrapfigure}[14]{r}{.5\linewidth}
    \includegraphics[width=\linewidth]{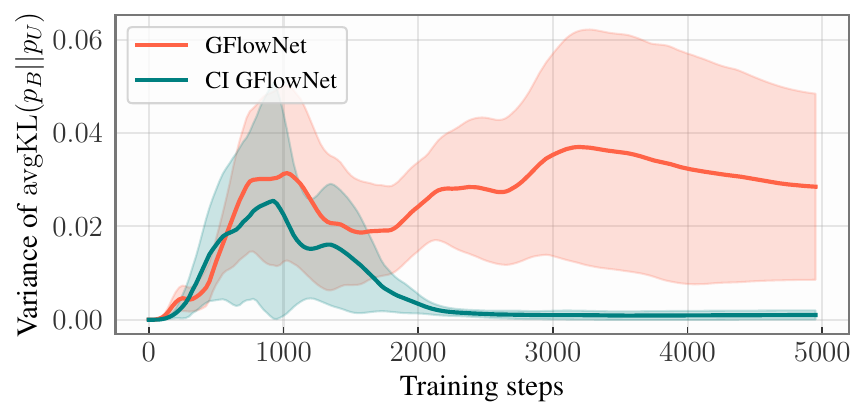}
    \caption{Random features act as a regularizer for GFlowNet training.} 
    \label{fig:hvariances} 
\end{wrapfigure}
Importantly, \Cref{eq:regularizationsa} suggests $\sigma$ should be set to a small positive number to ensure the reminder term in $\mathcal{O}(\sigma^{4})$ is negligible. 
To measure the effects of 
our method in practice, we calculate 
the variance of $\mathrm{avgKL}(p_{B} || p_{U}) = \frac{1}{|\mathcal{S}|} \sum_{s \in \mathcal{S}} \mathrm{KL}[p_{B}(s, \cdot) || p_{U}(s, \cdot)]$, 
with $p_{U}$ being an uniform backward policy. 
We train a continuously indexed (CI) GFlowNet with $\Gamma = \mathbb{R}^{8}$ and $\gamma \sim \mathcal{N}(0, 0.15^{2} \cdot I)$ and a conventional GFlowNet for the \textsc{Hypergrid} environment (see \Cref{sec:experimentsa}). 
As in \Cref{fig:hypergridpbsa}, we fix both the architecture and the parameter initialization of the neural networks underlying $p_{F}^{\smash[t]{(\gamma)}}$ and $p_{B}^{\smash[t]{(\gamma)}}$, and compute the per-step variance of the KL divergence by solely modifying the initial pseudorandom seed of the exploratory policy. 
We repeat this process across 5 independent runs. 
In conformity with Proposition~\ref{prop:regularizationsa}, we observe in \Cref{fig:hvariances} that RFs have a stabilizing effect on the learning problem, smoothing the loss surface and increasing the model's robustness to the randomness of sample generation---as emphasized by the reduced variance of $\mathrm{avgKL}(p_{B} || p_{U})$. 

\section{Discretely Indexed GFlowNets}

We complement \Cref{sec:discretesa} by introducing further examples of how to partition the state space, besides the abstract modulo-based technique described in Remark~\ref{remark:modulosa}.

\begin{example}[Length-based $\mathrm{TC}$] \label{example:lengthsa} 
    Let $d(x)$ be the trajectory-size, or \emph{description-length}, for generating a $x \in \mathcal{X}$. 
    Common GFlowNet applications such as sequence design \citep{laajil2025curriculumaugmentedgflownetsmrnasequence}, variable selection, drug discovery \citep{bengio2021}, and structure learning \citep{deleu2022bayesian} often have $d(x) \in [d_{\mathrm{min}}, d_{\mathrm{max}}]$ with both $d_{\mathrm{min}}$ and $d_{\mathrm{max}}$ known. 
    Under this condition, we can define $\mathrm{TC}(\mathcal{X}, K) = \{\mathcal{X}_{k}\}_{k=1}^{K}$ for $K \le D \coloneqq d_{\mathrm{max}} - d_{\mathrm{min}} + 1$ as 
    \begin{equation*}
        \mathcal{X}_{k} = \left\{ x \in \mathcal{X} \colon d_{\mathrm{min}}^{\smash[t]{(k)}} \coloneqq d_{\mathrm{min}} + \frac{k - 1}{K} \cdot D \le d(x) < d_{\mathrm{max}}^{\smash[t]{(k)}} \coloneqq d_{\mathrm{min}} + \frac{k}{K} \cdot D \right\}.   
    \end{equation*}
    The policy $p_{F}^{\smash[t]{(k)}}$ can be simply implemented by masking out terminating transitions leading to $x$ for which $d(x) \notin [d_{\mathrm{min}}^{\smash[t]{(k)}}, d_{\mathrm{max}}^{\smash[t]{(k)}})$. 
\end{example}

\begin{example}[Hamming weight-based $\mathrm{TC}$] \label{example:hammingsa} 
    Let $\mathcal{X} = \{1, 0\}^{d}$, which is often the cases in applications such as bit sequences \citep{malkin2022trajectory} and Ising model simulation \citep{liu2024generativemarginalizationmodels}. 
    We define $w(x) = \sum_{1 \le i \le d} x_{i}$ as the \emph{Hamming weight} of $x$ and $\mathrm{TC}(\mathcal{X}, K) = \{\mathcal{X}_{k}\}_{k=1}^{K}$ as \looseness=-1 
    \begin{equation*}
        \mathcal{X}_{k} = \{ x \in \mathcal{X} \colon w(x) \in [a_{k}, b_{k}) \} 
    \end{equation*}
    with $a_{1} = 0$, $b_{K} = d + 1$, and $\bigcup_{k=1}^{K} [a_{k}, b_{k}) = [0, d + 1)$. 
    We implement $p_{F}^{\smash[t]{(k)}}$ by pruning transitions resulting in states for which the condition $w(x) \in [a_{k}, b_{k})$ is unattainable.  \looseness=-1 
\end{example}

The tasks above can be easily seen to fall into the symmetric-junta category we outlined in Proposition~\ref{prop:npsaa}. 
Example~\ref{example:lengthsa} straightforwardly follows from picking $\psi(s)$ to be equal to the number of components defining $s$.
Example~\ref{example:hammingsa}, on the other hand, follows from letting $s$ be represented as a subset of $\{1, 0\} \times [d]$ and defining $\psi(s) = \# (s \cap r)$ with the reference set $r = \{(1, 1), \dots, (1, d)\}$ corresponding to a sequence containing only $1$s.
In either case, a tractable implementation for SC GFlowNets is admissible. 

\pp{Implementation of SC GFlowNets.} In addition, to further clarify the implementation of SC GFlowNets, we provide a high-level description of its training and inference in the pseudocode in Algorithms~\ref{alg:trainingsa} and~\ref{alg:inferencesa}, respectively, alongside detailed and efficient computer code for reproducibility.

\begin{minipage}{.5\textwidth} 
\begin{algorithm}[H]
    \caption{Training of SC GFlowNets} 
    \label{alg:trainingsa} 
    \begin{algorithmic}
        \Require $K$, $\mathbf{Z} = (Z_{k})_{k \in [K]}$ 
        \Require $\mathrm{trainingIterations}$, $\epsilon > 0, \, \mathrm{batchSize}$ 
        \Require SC GFlowNet $\gfn_{[K]}$ 
        \Require $\mathrm{getPolicies} \colon k, \gfn_{[K]} \mapsto \{p_{F}^{(k)}, p_{B}^{(k)}\}$
        \For{$\_$ in $\mathrm{range(trainingIterations)}$}
            \State $\mathrm{Samples} \gets \{\}$
            \ParFor{$\_$ in $\mathrm{range(batchSize)}$}
                \State $\mathbf{p}_Z \gets (1 - \epsilon) \cdot \frac{\mathbf{Z}}{\mathbf{1}_{K}^{T} \mathbf{Z}} + \epsilon \cdot \frac{\mathbf{1}_{K}}{K} $
                \State $k \sim \mathrm{Categorical}\left( \mathbf{p}_{Z} \right)$
                \State $p_{F}^{(k)}, p_{B}^{(k)} \leftarrow \mathrm{getPolicies}(k, \gfn_{[K]})$ 
                \State \textcolor{gray}{$\#$ Sample trajectory from $p_{F}^{(k)}$}  
                \State $\tau \sim \epsilon\text{-greedy of} \, p_{F}^{(k)}(s_{o}, \cdot)$ 
                \State $\mathrm{Samples} \gets \mathrm{Samples} \cup \{ (\tau, k) \}$
            \EndParFor 
            \State $\mathrm{loss} \gets \mathrm{lossFunction}(\mathrm{Samples}, \gfn_{[K]}, \mathbf{Z})$
            \State $\gfn_{[K]}, \mathbf{Z} \gets \mathrm{GradientStep}(\mathrm{loss}, \gfn_{[K]}, \mathbf{Z})$
        \EndFor
    \end{algorithmic}
\end{algorithm}
\end{minipage} 
\hspace{6pt} 
\begin{minipage}{.48\textwidth}
    \begin{algorithm}[H]
        \caption{Inference with SC GFlowNets}
        \label{alg:inferencesa} 
        \begin{algorithmic}
            \Require $K$, $\mathbf{Z} = (Z_{k})_{k \in [K]}$
            \Require SC GFlowNet $\gfn_{[K]}$
            \Require $\mathrm{getPolicies} \colon k, \gfn_{[K]} \mapsto \{p_{F}^{(k)}, p_{B}^{(k)}\}$ 
            \State $k \sim \mathrm{Categorical}\left(\frac{\mathbf{Z}}{\mathbf{1}_{K}^{T}\mathbf{Z}} \right)$
            \State $p_{F}^{(k)}, p_{B}^{(k)} \gets \mathrm{getPolicies}(k, \gfn_{[K]})$ 
            \State \textcolor{gray}{$\#$ Sample trajectory}
            \State $\tau \sim p_{F}^{(k)}(s_{o}, \cdot)$
            \State \textcolor{gray}{$\#$ Get terminal state from $\tau$} 
            \State $x \gets \mathrm{TerminalState}(\tau)$
            \State $\mathbf{return}$ $x$ 
        \end{algorithmic}
    \end{algorithm}
\end{minipage}

\section{Revisiting Subgraph Asynchronous Learning (SAL)} 

This section and the next highlight how our framework can be instantiated to accommodate prior algorithms for learning GFlowNets. 
We start with SAL \citep{silva2025generalization}. 
Similarly to SC GFlowNets, SAL defines a fixed-horizon covering (FHC)\footnote{\citet{silva2025generalization} called it a fixed-horizon \emph{partition} (FHP); however, the underlying sets need not be disjoint, as the word \emph{partition} suggests.} of $\mathcal{X}$, referred to as $\mathrm{FHC}(\mathcal{X}, K)$.

\begin{definition}[Fixed-Horizon Covering] \label{def:fhcsa} 
    We call $\mathrm{FHC}(\mathcal{X}, K) = \mathcal{X}_{o} \cup \{\mathcal{X}_{k}\}_{k=1}^{K}$ a \emph{fixed-horizon covering} (FHC) if the following conditions are satisfied. 
    \begin{enumerate}
        \item There are disjoint subsets $\{\mathcal{S}_{k}\}_{k=1}^{K}$ of $\mathcal{S}$ such that $\mathcal{X}_{k}$ contains all states in $\mathcal{X}$ reachable from some $s \in \mathcal{S}_{k}$. 
        \item Each $s \in \mathcal{S}_{k}$ is within a fixed-distance $d$ to $s_{o}$, and all states within distance $d$ of $s_{o}$ are contained in some $\mathcal{S}_{k}$ (\text{fixed-horizon}). 
        \item $\mathcal{X}_{o}$ is the subset of $\mathcal{X}$ whose elements' to $s_{o}$ to $x$ is smaller than $d$. 
    \end{enumerate}
    We call \emph{distance from $s$ to $s_{o}$} the length of the shortest-path from $s_{o}$ and $s$. 
\end{definition}

Drawing on Definition~\ref{def:fhcsa}, each partition enforces a so-called Amortized Trajectory Balance (ATB), defined as 
\begin{equation} \label{eq:atbsa} 
    F(s, k) p_{F}^{\smash[t]{(k)}}(s, \tau) = R(x) p_{B}(x, \tau) \text{ for all } k \in [K], \, s \in \mathcal{S}_{k}, \, x \in \mathcal{X}_{k}, \text{ and } \tau \in s \rightsquigarrow x,
\end{equation}
and then $Z \cdot p_{F}^{\smash[t]{(o)}}(s, \tau') = F(s, k) p_{B}(s, \tau')$ and $Z \cdot p_{F}^{\smash[t]{(o)}}(s, \tau') = R(x) p_{B}(s, \tau')$ for $k \in [K]$, $s \in \mathcal{S}_{k}$, $x \in \mathcal{X}_{o}$, and $\tau \in s_{o} \rightsquigarrow s$ and $\tau' \in s_{o} \rightsquigarrow x$.
In other words, each model first enforces a balance condition within its own partition according to a shared (uniform) backward policy $p_{B}$, and their learned flow functions $F(\cdot, k)$ are later used as the target distribution of the centralized model's ($p_{F}^{\smash[t]{(o)}}$) own learning objective.  

We next show that a solution to SAL's amortized TB can be interpreted as a DI GFlowNet satisfying its own MTB in Definition~\ref{def:mtbsa}.     
We do this by introducing a Doob $h$-transform (see \Cref{sec:doobsa}) that maps $p_{F}^{\smash[t]{(o)}}$ and $p_{F}^{\smash[t]{(k)}}$ into a policy function that strictly goes through $\mathcal{S}_{k}$ when sampling from $\mathcal{X}_{k}$.
With this in mind, we let
\begin{equation*}
    T_{k}^{\smash[t]{(f)}} = \min \{ t \colon S_{t}^{\smash[t]{(f)}} \in \mathcal{S}_{k} \} \text{ and } T_{k}^{\smash[t]{(b)}} = \min \{t \colon S_{t}^{\smash[t]{(b)}} \in \mathcal{S}_{k}\} 
\end{equation*}
be the hitting times on $\mathcal{S}_{k}$ of the Markov chains introduced in Proposition~\ref{prop:markovsa}, and 
\begin{equation*}
    h_{k}(s) = \mathbb{P}_{\kappa_{F}} \left[ T_{k}^{\smash[t]{(f)}} = \min_{k' \in [K]} T_{k'}^{\smash[t]{(f)}} < \infty \big| s \right] \text{ and } g_{k}(s) = \mathbb{P}_{\kappa_{B}}\left[T_{k}^{\smash[t]{(b)}} = \min_{k' \in [K]} T_{k'}^{\smash[t]{(b)}} < \infty \big| s \right] 
\end{equation*}
be the corresponding probabilities that a Markov chain starting at $s \in \bar{\mathcal{S}}$ and following $\kappa_{F}$ (resp. $\kappa_{B}$) arrives at $\mathcal{S}_{k}$ prior to arriving at any other $\mathcal{S}_{k'}$ with $k' \neq k$. 
We first show that both $h_{k}$ and $g_{k}$ are harmonic functions. 

\begin{lemma} \label{lemma:harmonicsaa} 
    Both $h_{k}$ and $g_{k}$ are harmonic functions on $\bar{\mathcal{S}} \setminus \mathcal{S}_{k}$ for $k \in [K]$. 
\end{lemma}

A core property required by the next Proposition~\ref{prop:salsa} is that both $h_{k}(s_{o})$ and $g_{k}(x)$ for $x \in \mathcal{X}$ can be interpreted as probability distributions over $[K]$. 
As we emphasize next, this intuitively follows from the fact that $\{\mathcal{S}_{k}\}_{k=1}^{K}$ forms a boundary between $s_{o}$ and the set $\mathcal{X}$ of terminal states.

\begin{lemma} \label{lemma:harmonicsaaa} 
    The functions $h_{k}$ and $g_{k}$ in Lemma~\ref{lemma:harmonicsaa} satisfy 
    \begin{equation*}
        \sum_{0 \le k \le K} h_{k}(s_{o}) = 1 \text{ and } \sum_{0 \le k \le K} g_{k}(x) = 1 \text{ for } x \in \mathcal{X} 
    \end{equation*}
    with $h_{o}(s_{o}) = \mathbb{P}_{\kappa_{F}}\left[T_{k}^{(f)} = \infty \forall k \in [K] \bigg| s_{o}\right]$ and $g_{o}(x) = \mathbb{P}_{\kappa_{B}} \left[ T_{k}^{(b)} = \infty \, \forall k \in [K]  \bigg| x \right]$ referring to the case when 
    the generative process 
    terminates on $\mathcal{X}_{o}$ (see Definition~\ref{def:fhcsa}). \looseness=-1 
\end{lemma}

Building on Lemmas~\ref{lemma:harmonicsaa} and~\ref{lemma:harmonicsaaa}, we demonstrate in Proposition~\ref{prop:salsa} that SAL can be viewed as an instance of a DI GFlowNet by showing that a solution to \Cref{eq:atbsa} can be uniquely represented as a solution to the mixture learning problem in Definition~\ref{def:mtbsa} with a specific choice of $Z_{k}$ and $q(k | x)$. 
This suggests yet another form of designing a DI GFlowNet. 
It should nonetheless be noticed that, contrarily to SC GFlowNets, SAL is an inherently distributed algorithm---as neither $h_{k}$ nor $g_{k}$ in Lemma~\ref{lemma:harmonicsaa} can be tractably computed in general. 

\begin{proposition}[SAL as a DI GFlowNet] \label{prop:salsa} 
    Let $\{p_{F}^{\smash[t]{(k)}}\}_{k=0}^{K}$ and $p_{B}$ be the solutions to SAL's balance equations having $d$ as the distance from $s_{o}$ to $\{\mathcal{S}_{k}\}_{k=1}^{K}$. 
    We define 
    \begin{equation*}
        \tilde{p}_{F}^{\smash[t]{(k)}}(s, s') = 
        \begin{cases}
            p_{F}^{\smash[t]{(o)}}(s, s') \cdot \frac{h_{k}(s')}{h_{k}(s)} \text{ if } \mathrm{dist}(s_{o}, s') \le d, \\ 
            p_{F}^{\smash[t]{(k)}}(s, s') \text{ otherwise}, 
        \end{cases}
    \end{equation*}
    with $\mathrm{dist}$ representing the shortest-path distance in the underlying state graph, and 
    \begin{equation*}
        \tilde{p}_{B}^{\smash[t]{(k)}}(s, s') =
        \begin{cases}
            p_{B}(s, s') \text{ if } \mathrm{dist}(s_{o}, s') \le d, \\ 
            p_{B}(s, s') \cdot \frac{g_{k}(s')}{g_{k}(s)} \text{ otherwise},  
        \end{cases}
    \end{equation*}
    for $k \ge 1$ and $\tilde{p}_{B}^{\smash[t]{(o)}}(s, s') = p_{B}(s, s')$ for $k = 0$. 
    Then, for $k \in \{0, \dots, K\}$ and each trajectory $\tau$ from $s_{o}$ to $x \in \mathcal{X}_{k}$ going through $\mathcal{S}_{k}$, 
    \begin{equation*}
        Z_{k} \tilde{p}_{F}^{\smash[t]{(k)}}(s_{o}, \tau) = R(x) q(k|x) \tilde{p}_{B}^{\smash[t]{(k)}}(x, \tau) 
    \end{equation*}
    with $Z = \sum_{x \in \mathcal{X}} R(x)$ and $Z_{k} = Z \cdot h_{k}(s_{o})$ and $q(k|x) = g_{k}(x)$. 
\end{proposition}

In simpler terms, Proposition~\ref{prop:salsa} highlights SAL can be cast into our mixture learning framework by reframing $h_{k}(s_{o})$ and $g_{k}(x)$ as our \emph{prior} and \emph{posterior} mixing distributions, an interpretation supported by the normalization property established in Lemma~\ref{lemma:harmonicsaaa}.
Notably, the central statements 
on the correctness of SAL in \citep{silva2025generalization} can be immediately derived from Proposition~\ref{prop:salsa}, highlighting the generality of DI GFlowNets as a theoretical tool for devising GFlowNet training algorithms with proven correctness.
The next section emphasizes this point further. 
\looseness=-1 

\section{Revisiting Boosted GFlowNets} \label{sec:boostingsa} 

This section is separated into two parts. 
The first one shows how Boosted GFlowNets can be understood through the lens of our abstract framework. 
The second one evaluates Boosted GFlowNets against our proposed SC GFlowNets, highlighting the computational amenability of the latter when compared to the former.  

\subsection{Boosted GFlowNets satisfy the Mixed Trajectory Balance}

As with SC GFlowNets and SAL, Boosted GFlowNets learn distinct samplers focusing on nearly-disjoint regions of the state space \citep{locatello18boosting, dallantonia2025boostedgflownetsimprovingexploration}.
However, instead of relying on amortization (as in SC GFlowNets) or parallelization (SAL) over fixed structural primitives, Boosted GFlowNets adaptively learn a \emph{residual target distribution} in a sequential fashion.   
Their connection to DI GFlowNets, which will be developed in this section, will reveal our framework's applicability in the context of \emph{countably infinite} mixtures.
\looseness=-1 

To see this, we briefly recall the algorithmic design of 
Boosted GFlowNets.
Let $\{\gfn_{k}\}_{k \ge 1} \coloneqq \{(Z_{k}, p_{F}^{\smash[t]{(k)}}, p_{B}^{\smash[t]{(k)}})\}_{k \ge 1}$ be 
GFlowNets and, for $\tau \in s_{o} \rightsquigarrow x$, define \looseness=-1 
\begin{equation*}
    R^{\smash[t]{(0)}}(x, \tau) = 0 \text{ and } R^{\smash[t]{(k)}}(x, \tau) = R^{\smash[t]{(k - 1)}}(x, \tau) + \frac{Z^{\smash[t]{(k)}} p_{F}^{\smash[t]{(k)}}(s_{o}, \tau)}{p_{B}^{\smash[b]{(k)}}(x, \tau)}  
\end{equation*}
for $k \ge 2$. 
Intuitively, $R^{\smash[t]{(k)}}(x, \tau)$ represents an estimate of the accumulated probability mass on $x$ by the first $k$ models. 
An implicit assumption of Boosted GFlowNets, which is compatible with the observed empirical behavior of GFlowNets \citep{shen23gflownets}, is that each GFlowNet underestimates the total probability mass in any given region.
\looseness=-1 

\begin{assumption}[Underallocation of probability mass] \label{assumption:ua}
    After trained, we assume assume the first $k$-th GFlowNets satisfy $R^{\smash[t]{(k)}}(x, \tau) \le R(x)$ for all 
    $x$ 
    and $\tau \in s_{o} \rightsquigarrow x$. \looseness=-1 
\end{assumption}

Under Assumption~\ref{assumption:ua}, each GFlowNet is trained by minimizing the \emph{boosted trajectory balance (BTB)} loss function, which replaces $R(x)$ by the \emph{residual target} $R(x) - R^{\smash[t]{(K)}}(x, \tau)$ with $\tau$ drawn from the corresponding backward policy $p_{B}^{\smash[b]{(K)}}(x, \cdot)$. 

\begin{definition}[Boosted Trajectory Balance Condition \& Loss] \label{def:btbsa} 
    Let $\{\gfn_{k}\}_{k=1}^{K}$ be a sequence of GFlowNets with $\{\gfn_{k}\}_{k=1}^{K - 1}$ satisfying Assumption~\ref{assumption:ua}. 
    By defining 
    \begin{equation*}
        R^{\smash[t]{(k)}}(x) = \mathbb{E}_{\tau \sim p_{B}^{\smash[t]{(k)}}(x, \cdot)} \left[ R^{\smash[t]{(k)}}(x, \tau) \right],  
    \end{equation*}
    we call $K$-th \emph{Boosted Trajectory Balance} (BTB) 
    condition
    \begin{equation*}
        Z_{K} \cdot p_{F}^{\smash[t]{(K)}}(s_{o}, \tau) = (R(x) - R^{\smash[t]{(K - 1)}}(x)) \cdot p_{B}^{\smash[t]{(K)}}(x, \tau), 
    \end{equation*}
    which is enforced through the loss function 
    \begin{equation*}
        \mathcal{L}_{\mathrm{BTB}} \left( \tau \big| \{\gfn_{k}\}_{k=1}^{K} \right) = \left( \log \frac{Z_{K} \cdot p_{F}^{\smash[t]{(K)}}(s_{o}, \tau)}{( R(x) - R^{\smash[t]{(K - 1)}}(x)) \cdot p_{B}^{\smash[b]{(K)}}(x, \tau)} \right)^{2}, 
    \end{equation*}
    which we refer to as \emph{BTB loss}.
\end{definition}

The intuition for the BTB loss is that, when the the first $K - 1$ GFlowNets \emph{collapse} onto a subset of the state space, the $K$-th GFlowNet will---through the influence of the residual target---mostly ignore such region. 
In practice, \citep{dallantonia2025boostedgflownetsimprovingexploration} proposes several techniques for stabilizing training when the quantity $R(x) - R^{\smash[t]{(K)}}(x)$ becomes overly small, to which the interested reader is referred to for further details; however, none of them are relevant to our discussion.
It should be clear from Definition~\ref{def:btbsa} and Assumption~\ref{assumption:ua} that Boosted GFlowNets only sample from the target distribution when the trajectory-wise residual target vanishes and that the BTB loss can only be non-biasedly evaluated if we can compute $R^{(K - 1)}(x)$ exactly. 
The latter is satisfied when $R^{(K - 1)}(x, \tau) = R(x)$ for every $\tau$ in the support of $p_{F}$, and $0$ otherwise; equivalently, when the on-policy expectation of the BTB loss under $p_{F}^{(K - 1)}$ vanishes.
This was the main condition for \citep[Theorem 1]{dallantonia2025boostedgflownetsimprovingexploration}. 
We explicitly state these behaviors in our next assumption. \looseness=-1 

\begin{assumption}[Convergence of residuals \& Collapse] \label{assumption:crsa} 
    For all $x \in \mathcal{X}$ and $\tau \in s_{o} \rightsquigarrow x$, the series $\{R^{\smash[t]{(k)}}\}_{k \ge 1}$ satisfies $\lim_{k \rightarrow \infty} R^{\smash[t]{(k)}}(x, \tau) = R(x)$. 
    Also, there is a subset $\mathcal{X}_{k}$ of $\mathcal{X}$ such that $R^{(k)}(x, \tau) = R^{(k)}(x)$ for each $\tau \in s_{o} \rightsquigarrow x$ and $x \in \mathcal{X}_{k}$, and $R^{(k)}(x, \tau) = 0$ otherwise.  
\end{assumption}

We now show that Assumptions~\ref{assumption:ua} and~\ref{assumption:crsa} allow us to characterize Boosted GFlowNets as a solution to the MTB equations in Definition~\ref{def:mtbsa} based on a countably infinite mixing distribution.   
Importantly, as with classical topic modelling \citep{blei2003latent, Hoffman2010Online} based on Dirichlet processes, we typically choose a 
$K$ and truncate the prior mixing distribution to be supported on $[K]$. 
\citet{dallantonia2025boostedgflownetsimprovingexploration} sets $K$ to $2$ or $3$ in their experiments. 

\begin{proposition}[Boosted GFlowNets as DI GFlowNets] \label{prop:bgsa} 
    We define $\{\gfn_{k}\}_{k \ge 1} \coloneqq \{(Z_{k}, p_{F}^{\smash[t]{(k)}}, p_{B}^{\smash[t]{(k)}})\}_{k \ge 1}$ as GFlowNets satisfying 
    Assumptions~\ref{assumption:ua} and~\ref{assumption:crsa}.
    By letting 
    \begin{equation*}
        R_{k}(x) = R^{\smash[t]{(k)}}(x) - R^{\smash[t]{(k - 1)}}(x) \text{ and } q(k|x) = \frac{R_{k}(x)}{R(x)},  
    \end{equation*}
    then the following properties are satisfied. 
    \begin{enumerate}[leftmargin=12pt]
        \item $q(k|x)$ is a 
        probability distribution, i.e., $\sum_{k \ge 1} q(k|x) = 1$ for each $x \in \mathcal{X}$, \looseness=-1 
        \item The sequence $\{\gfn_{k}\}_{k \ge 1}$ satisfies a MTB condition
        \begin{equation*}
            Z_{k} \cdot p_{F}^{\smash[t]{(k)}}(s_{o}, \tau) = R(x) q(k|x) p_{B}^{\smash[t]{(k)}}(x, \tau) 
        \end{equation*}
        for $k \ge 1$, with $\sum_{k \ge 1} Z_{k} = Z$. 
    \end{enumerate}
\end{proposition}

Once again, the correctness of Boosted GFlowNets can be directly inferred from Corollary~\ref{col:pmassa} and Proposition~\ref{prop:bgsa}.
However, contrarily to SC GFlowNets, which support both amortization and distributed training, Boosted GFlowNets can only be trained via sequential learning.
Also, in contrast to both SC GFlowNets and SAL, Boosted GFlowNets cannot be trained with an near-embarrassingly parallel algorithm, primarily due to the unbounded support of the mixing distribution and the dependence of each model on its predecessors' residuals. 

In addition, we notice SC GFlowNets are the most foundational instantiation of DI GFlowNets when compared to both SAL and Boosted GFlowNets: stratum-conditioning could be used for training each booster in Boosted GFlowNets and each sampler in SAL, but neither of these algorithms could be used to learn SC GFlowNets without sacrificing its monolithicness, which we deem to be one of its fundamental properties---a single model \emph{can} be used for learning the entire mixture.  

\subsection{Experiments}

We now present several experiments to support our claim that (i) Boosted GFlowNets introduce a substantial computational overhead over the base model, which is not observed for SC GFlowNets, and (ii) SC GFlowNets often converge faster than Boosted GFlowNets.   

\begin{figure*}
    \begin{subfigure}{.49\textwidth}
        \includegraphics[width=\linewidth]{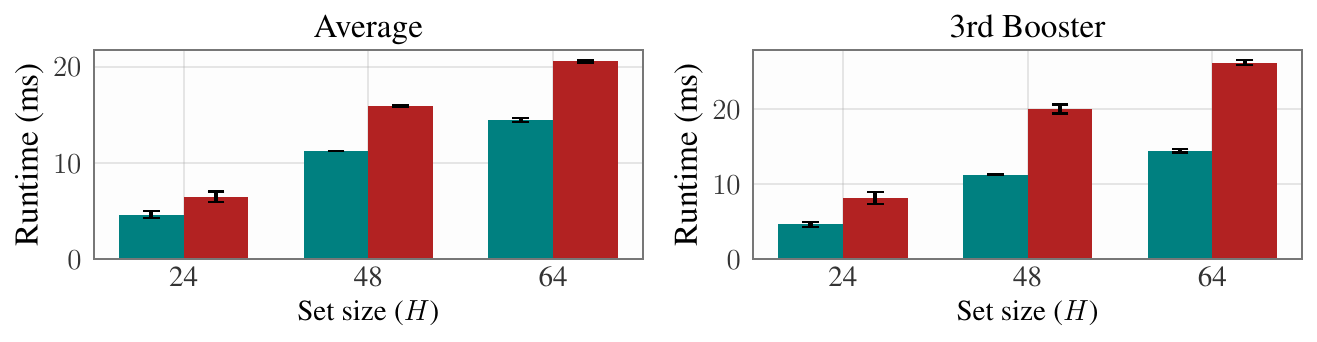}
        \caption{Runtimes per training step.}
    \end{subfigure}
    \begin{subfigure}{.49\textwidth}
        \includegraphics[width=\linewidth]{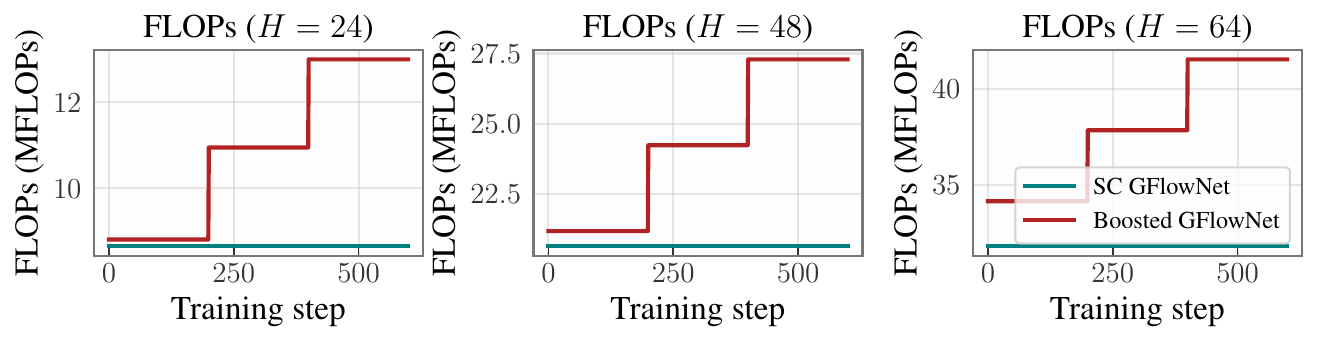}
        \caption{FLOPs per training step.}
    \end{subfigure}
    \caption{
        SC GFlowNets run both faster (left) and more compute-efficiently (right) than Boosted GFlowNets for the \textsc{Set Generation} domain.
        Throughout this section, runtimes represent per-step averages over $3$ independent runs, each with $600$ iterations.
    }
    \label{fig:compssetsa} 
\end{figure*}

\begin{figure*}
    \begin{subfigure}{.49\textwidth}
        \includegraphics[width=\linewidth]{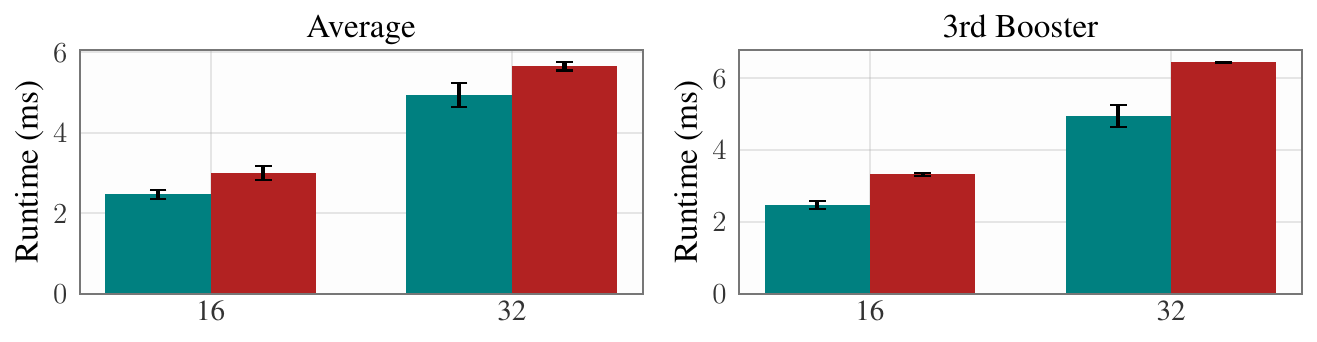}
        \caption{Runtimes per training step.} 
    \end{subfigure}
    \begin{subfigure}{.49\textwidth}
        \includegraphics[width=\linewidth]{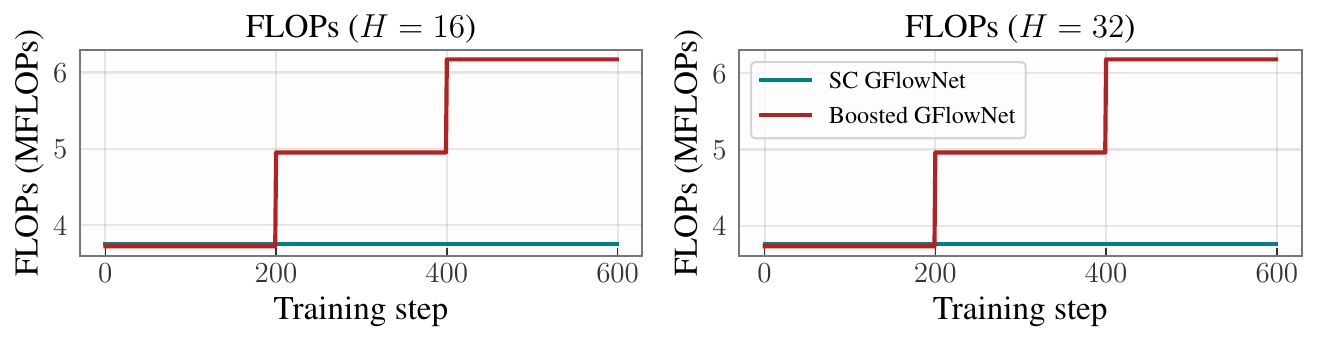}
        \caption{FLOPs per training step.} 
    \end{subfigure}
    \caption{SC GFlowNets are more compute-efficient than Boosted GFlowNets in the \textsc{Hypergrid} domain, requiring substantially fewer FLOPs per gradient step during training. \looseness=-1 }
    \label{fig:compshypergridsa} 
\end{figure*}

\begin{figure*}
    \begin{subfigure}{.49\textwidth}
        \includegraphics[width=\linewidth]{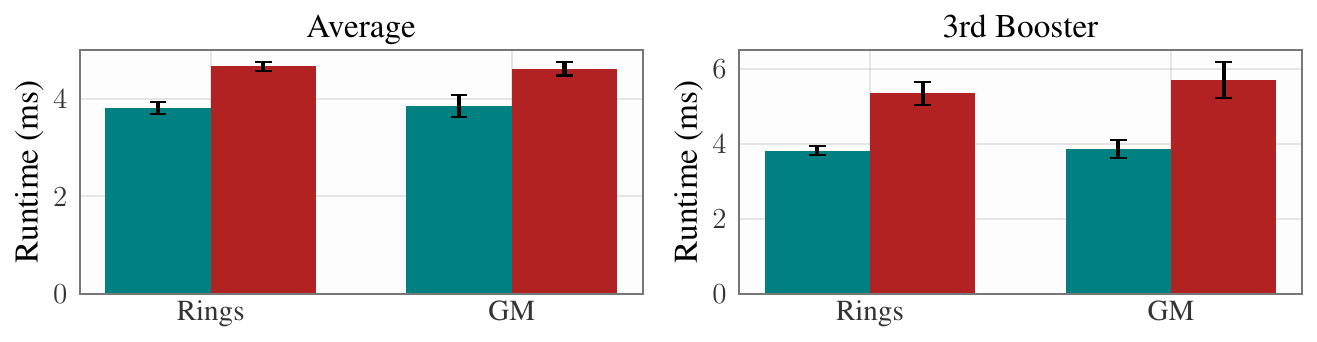}
        \caption{Runtimes per training step.} 
    \end{subfigure}
    \begin{subfigure}{.49\textwidth}
        \includegraphics[width=\linewidth]{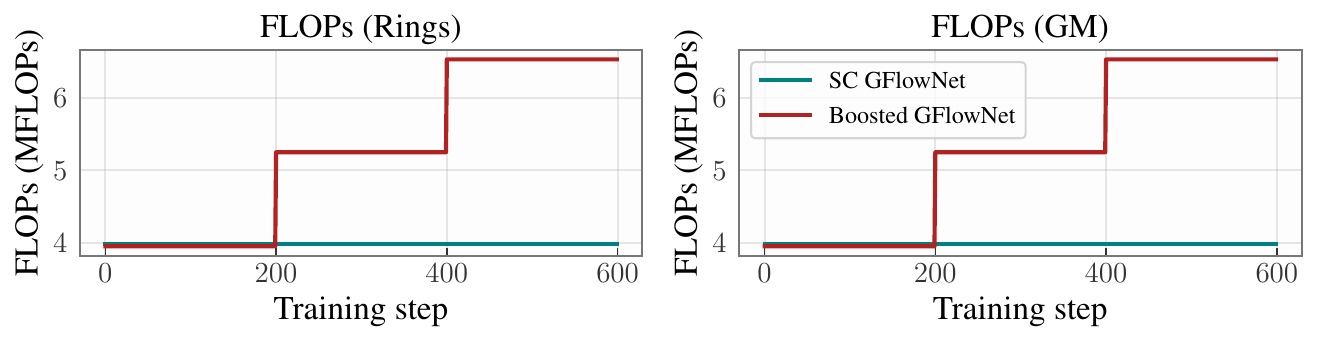}
        \caption{FLOPs per training step.} 
    \end{subfigure}
    \caption{SC GFlowNets is far more efficient than Boosted GFlowNets in the \textsc{Lazy Random Walk} domain (\texttt{Rings} corresponds to the \textsc{Rings} target; \texttt{GM}, to \textsc{Gaussian Mixture}).}
    \label{fig:compsdiffusionsa} 
\end{figure*}

\pp{Computational profile.} 
We start by showcasing the cost of Boosted GFlowNets is substantially larger than of SC GFlowNets in terms of both runtime and FLOPs per training step for the \textsc{Set Generation}, \textsc{Hypergrid}, and \textsc{Lazy Random Walk} domains. 
To evaluate the latter, we use JAX's \texttt{cost\_analysis} method for compiled functions \citep{jax2018github}.  
Results are displayed in \Cref{fig:compssetsa,fig:compshypergridsa,fig:compsdiffusionsa}. 
Notably, SC GFlowNets are up to roughly $1.6\times$ faster than Boosted GFlowNets per step (for \textsc{Set Generation}), and execute up to $50\%$ fewer FLOPs per training step. 
The reasons for this are clear: as discussed before, our method learns an \emph{amortized policy} over the mixture's components, while boosting requires training independent neural networks in a sequential fashion, having a monotonically increasing cost 
in the number of \emph{boosters} (components).
Importantly, during our experiments, we proceeded as suggested by \citet{dallantonia2025boostedgflownetsimprovingexploration} and allocated a prescribed number of iterations equally among a predefined number of boosters for training; automatic sampler elicitation remains an unresolved issue for Boosted GFlowNets.
\looseness=-1

\begin{figure*}
    \includegraphics[width=\linewidth]{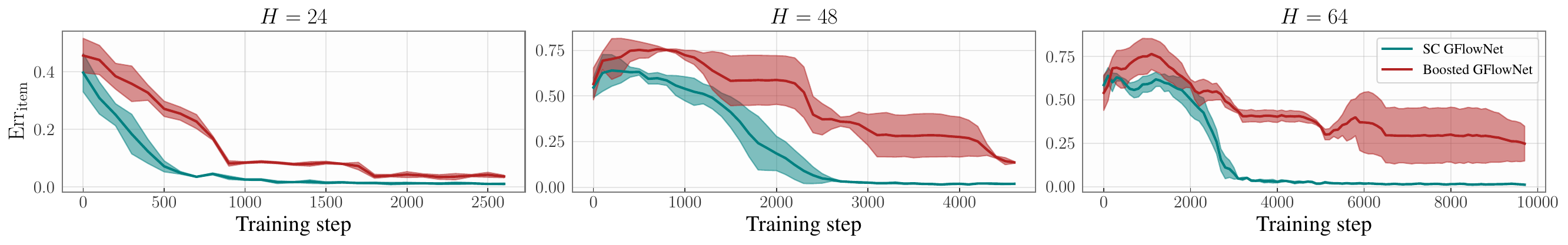}
    \caption{SC GFlowNets converge substantially faster than Boosted GFlowNets in the \textsc{Set Generation} domain, particularly for larger state spaces. See also \Cref{fig:setsconvergences} for $\mathrm{Err}_{\mathrm{item}}$'s definition.}
    \label{fig:setsboostingsa} 
\end{figure*}

\begin{figure*}
    \begin{subfigure}{.49\textwidth}
        \includegraphics[width=\linewidth]{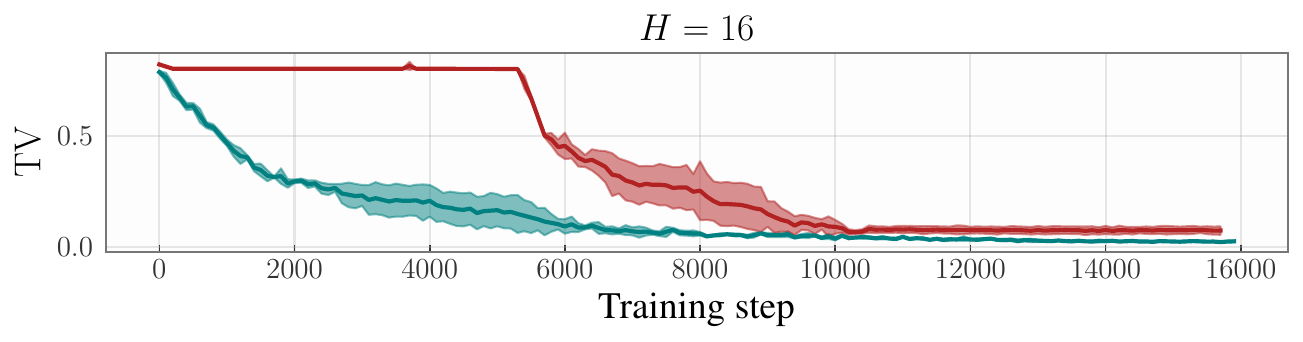}
    \end{subfigure}
    \begin{subfigure}{.49\textwidth}
        \includegraphics[width=\linewidth]{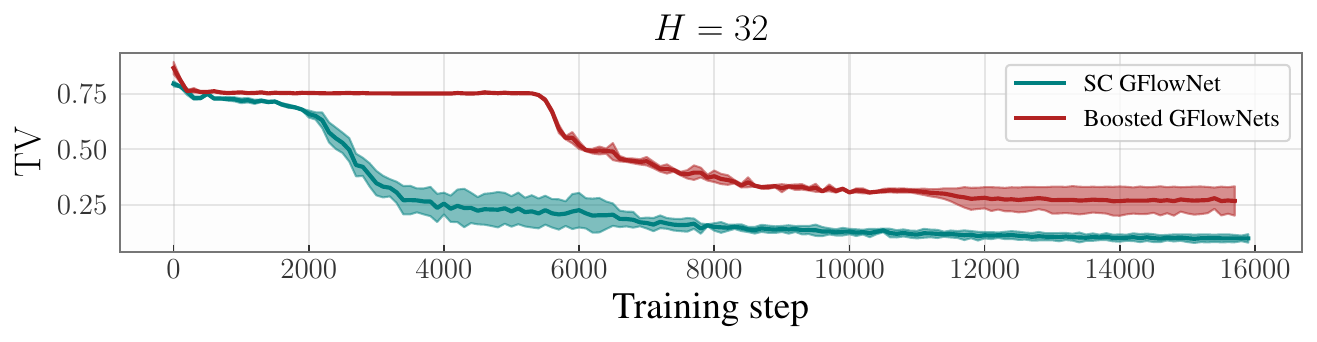}
    \end{subfigure}
    \caption{SC GFlowNets exhibit faster learning convergence when compared to Boosted GFlowNets, while also being far more compute-efficient (\Cref{fig:compshypergridsa}).}
    \label{fig:hypergridsboostingsa} 
\end{figure*}

\begin{figure*}
    \begin{subfigure}{.49\textwidth}
        \includegraphics[width=\linewidth]{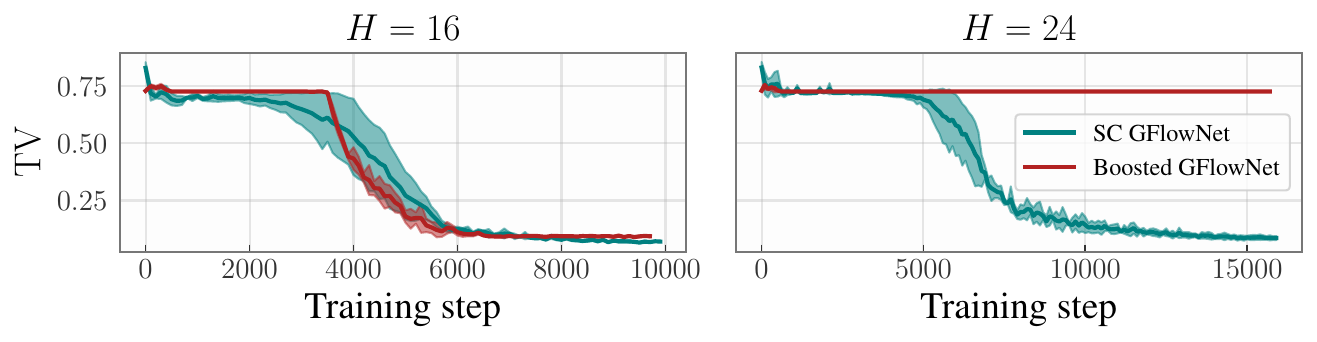}
        \caption{\textsc{Rings}.} 
    \end{subfigure}
    \begin{subfigure}{.49\textwidth}
        \includegraphics[width=\linewidth]{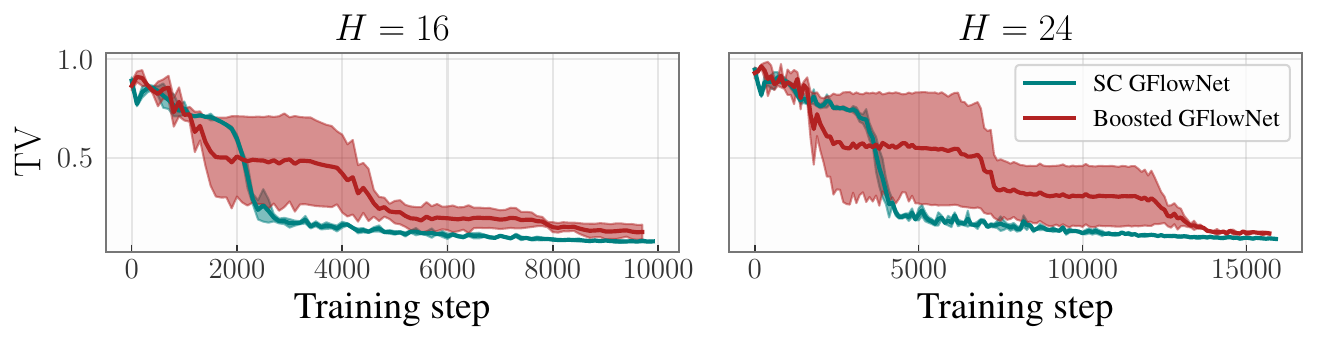}
        \caption{\textsc{Gaussian Mixture}.} 
    \end{subfigure}
    \caption{On top of being more computationally efficient (\Cref{fig:compsdiffusionsa}), SC GFlowNets also learn a more accurate distributional approximation to the target relatively to Boosted GFlowNets in both the \textsc{Rings} and \textsc{Gaussian Mixture} variants of the \textsc{Lazy Random Walk} domain.} 
    \label{fig:diffusionsboostingsa} 
\end{figure*}

\pp{Learning convergence.} 
In addition to being more compute-efficient, SC GFlowNets often lead to faster learning convergence in the considered tasks.
This is highlighted in \Cref{fig:setsboostingsa,fig:hypergridsboostingsa,fig:diffusionsboostingsa}. 
The intuition for this is that, while both Boosted and SC GFlowNets enhance state space exploration by design, hence improving the rate of discovery of high-probability regions, our method does so by stably partitioning the state space, ensuring consistent, non-random exploration of specific subsets of it.
In contrast, boosting is built on a more delicate approach that is is reliant upon accurate training of several independent neural networks, and can be biased---unless Assumptions~\ref{assumption:ua} and~\ref{assumption:crsa} are satisfied.
\looseness=-1 

\section{Experimental Details \& Further Experiments} \label{sec:experimentsaa}



This section extends empirical analysis in \Cref{sec:experimentsa} with additional domains, target distributions, and TC's constructions.
In addition, we provide further experimental details for reproducing our results. 
We also attach the computer code, written using JAX \citep{jax2018github}, we used for training and evaluating each of the described models on the considered tasks. 

\subsection{More Experiments} 

\begin{figure*}
    \begin{subfigure}{.49\textwidth}
        \includegraphics[width=\linewidth]{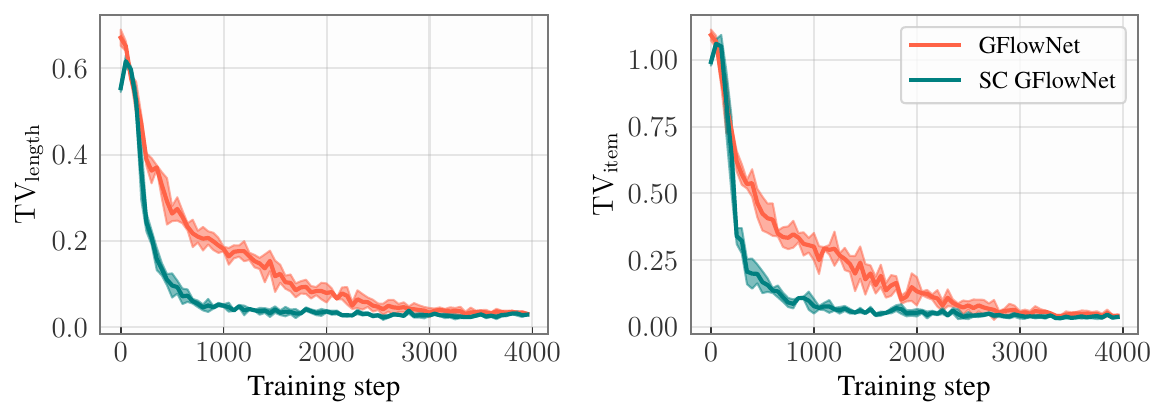}
        \caption{$L = 18$.}
    \end{subfigure}
    \begin{subfigure}{.49\textwidth}
        \includegraphics[width=\linewidth]{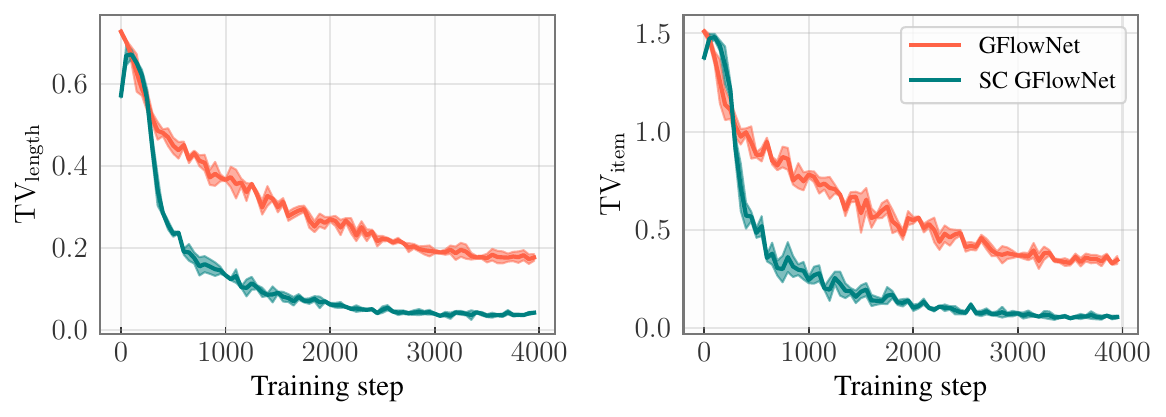}
        \caption{$L = 24$.}
    \end{subfigure}
    \caption{\textbf{SC GFlowNets significantly accelerate learning convergence} when approximating the generative process in \Cref{eq:gsequencesa}, with greater improvements being observed for larger state spaces. $\mathrm{TV}_{\mathrm{item}}$ and $\mathrm{TV}_{\mathrm{length}}$ are defined in \Cref{eq:msequencesa}.}
    \label{fig:sequencesa}  
\end{figure*}

\pp{Autoregressive sequence generation.} 
We train a GFlowNet to reproduce the following generative process 
of token sequences $(t_{1}, \dots, t_{l})$ over a finite vocabulary, 
\begin{equation} \label{eq:gsequencesa} 
    t_{1} \sim p_{\mathrm{init}}, \, l \sim p_{\mathrm{length}}, \, t_{i} | t_{i - 1} \sim p_{\mathrm{kernel}}(t_{i - 1}, \cdot) \text{ for } i \in \{2, \dots l\}.  
\end{equation}
We let $[V] \coloneqq \{1, \dots, V\}$ be the support of $p_{\mathrm{init}}$ (\emph{vocabulary}) and $p_{\mathrm{kernel}}(t, \cdot)$ for $t \in [V]$ and $\{0\} \cup [L]$ be 
that of $p_{\mathrm{length}}$.
The target distribution associated with $\mathbf{t} = (t_{1}, \dots, t_{l})$ is \looseness=-1 
\begin{equation} \label{eq:sequencesa} 
    R(\mathbf{t}) = p_{\mathrm{init}}(t_{1}) p_{\mathrm{length}}(l) \prod_{2 \le i \le l} p_{\mathrm{kernel}}(t_{i - 1}, t_{i}). 
\end{equation}
The GFlowNet is built upon a generative process starting at an empty sequence and iteratively deciding either to add a token from $[V]$ or to stop the generation, and the SC GFlowNet is designed according to Example~\ref{example:lengthsa} with $K = 3$ partitions. 
Due to its simplicity, many quantities of interest from the model in \Cref{eq:sequencesa} can be tractably computed, including the marginal distribution over sequence lengths and the value of $t_{i}$ for $i \in [L]$. 
We denote these by $\pi(l)$ and $\pi(i, v)$ for $i \in [L]$ and $v \in [V]$, respectively, and evaluate our model with 
\begin{equation} \label{eq:msequencesa} 
    \mathrm{TV}_{\mathrm{length}} = \frac{1}{2} \sum_{0 \le l \le L} | \hat{\pi}(l) - \pi(l)| \text{ and } \mathrm{TV}_{\mathrm{item}} = \frac{1}{2} \max_{v \in [V]} \sum_{1 \le i \le L} | \hat{\pi}(i, v) - \pi(i, v)|,  
\end{equation}
with $\hat{\pi}$ being a Monte Carlo estimate of $\pi$ based on the learned sampler. 
We illustrate the evolution of these metrics throughout training in \Cref{fig:sequencesa}.
As we can observe, SC GFlowNets significantly improve learning convergence.   

\begin{figure}
    \centering 
    \begin{subfigure}{.49\textwidth}
        \centering 
        \includegraphics[width=\linewidth]{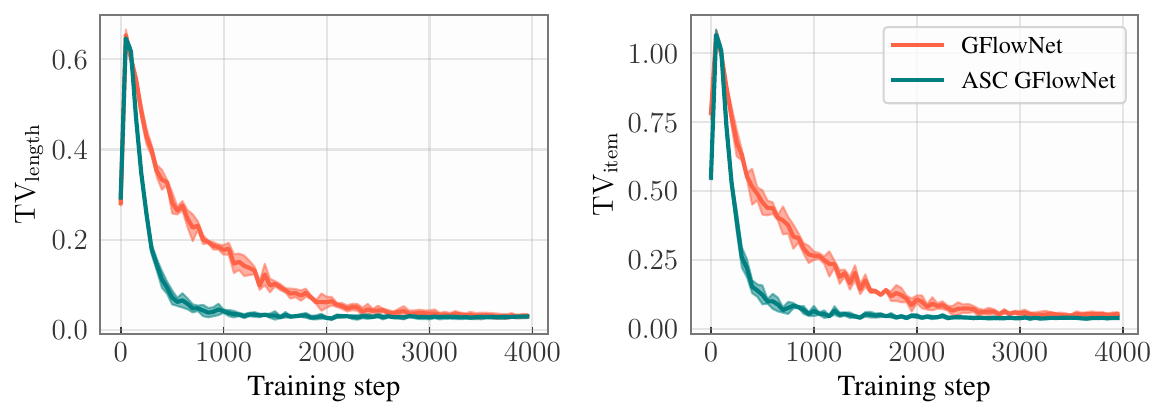}
        \caption{$L = 18$.} 
    \end{subfigure}
    \begin{subfigure}{.49\textwidth}
        \centering
        \includegraphics[width=\linewidth]{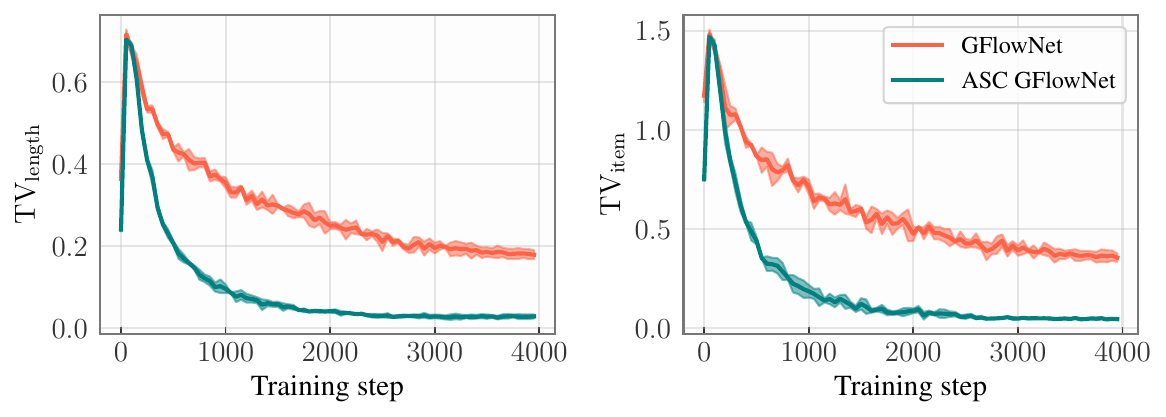}
        \caption{$L = 24$.} 
    \end{subfigure}
    \caption{
        Embarrassingly parallel counterpart of \Cref{fig:sequencesa}, comparing ASC GFlowNet against the best of $K$ randomly initialized and independently trained GFlowNets with matched computational cost.
        ASC GFlowNet accelerates learning convergence.
    }
    \label{fig:psequencesa} 
\end{figure} 

We also evaluate ASC GFlowNets on this domain, training $K = 3$ separate samplers in parallel on distinct subsets of the state space.
We show the results in \Cref{fig:psequencesa}.
Again, we notice our partitioning technique significantly speeds up training.  \looseness=-1 

\begin{wrapfigure}[10]{l}{.5\textwidth}
    \includegraphics[width=\linewidth]{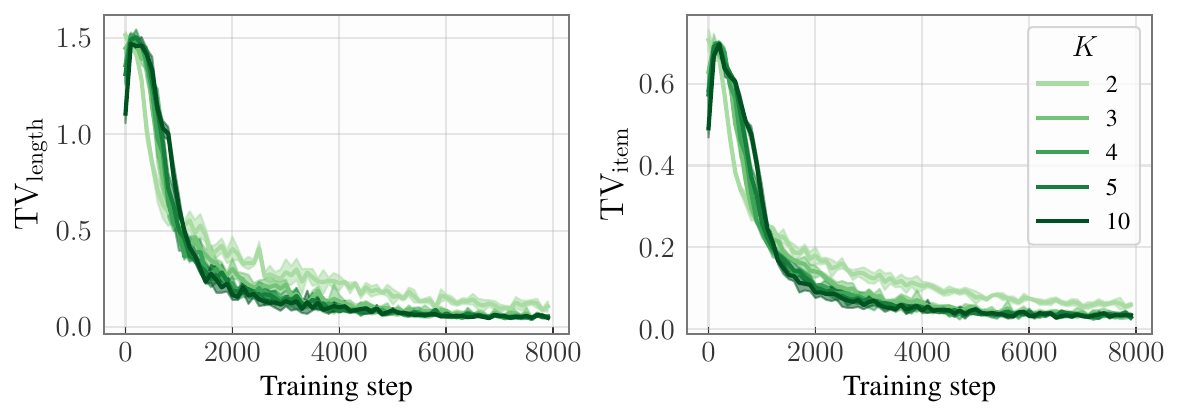}
    \caption{Accuracy of SC GFlowNets as a function of the number of components in the mixture.} 
    \label{fig:sequencesaaa} 
\end{wrapfigure}
Another question that comes up when designing SC GFlowNets is \emph{how to select the number of partitions}. 
\Cref{fig:sequencesaaa} shows that, for this domain, there is no significant difference between choosing $K \in \{2, 3, 4, 5, 10\}$. 
This may not always be the case, however; see \Cref{fig:randomwalksaaa}. 
We discuss this further when describing the \textsc{Lazy Random Walk} domain.

Importantly, the question of whether one should use ASC or SC GFlowNets depends primarily on the available computational resources. 
ASC GFlowNets, being massively parallelizable, have access to larger training sets and greater model capacity than any individual SC GFlowNet, being expected to both improve state space exploration and learning convergence, should the extra processing power be available. 
From an implementational perspective, as the samplers are  architecturally identical, this parallelization can be directly attained by using tools designed to express single-program, multiple-data (SPMD) computational models \citep{darema1988spmd}, such as \texttt{jax.pmap}. 
\looseness=-1 



\pp{Lazy Random Walk.}
The state space for the \textsc{Lazy Random Walk} domain is $\mathcal{S} = \{(\mathbf{x}, t) \colon \mathbf{x} \in \{-H, \dots, H\}^{d} \text{ and } 0 \le t < T\}$ and $\mathcal{X} = \{(\mathbf{x}, T) \colon \mathbf{x} \in \{-H, \dots, H\}^{d}\}$ for given positive integers $H$, $d$, and $T$. 
The generative process starts at $s_{o} = (\mathbf{0}, 0)$, 
and each iteration adds either $1$, $-1$, or $0$ to a chosen coordinate of $\mathbf{x}$---as long as it remains within the boundary of $\mathcal{S}$---and increments $t$ by one until $t = T$.
As in \citep{dallantonia2026avoidknowdivergenttrajectory, silva2026pathdependent}, we consider $d = 2$, $H = 16$, and $T = 32$.
Our target distributions, \textsc{Rings} and \textsc{Gaussian Mixture}, are shown in \Cref{fig:randomwalksa}.

Notably, as the trajectory size is the same for each $(\mathbf{x}, T)$, the length-based partitioning of Example~\ref{example:lengthsa} cannot be used for \textsc{Lazy Random Walk}. 
Instead, we propose a uniform \emph{distance}-based partitioning using the function $\psi(\mathbf{x}, T) = \|\mathbf{x}\|_{1}$ in Remark~\ref{remark:modulosa} with $K = 3$. 
We compute the TV distance between the learned and target distributions to measure a sampler's goodness-of-fit. 

\begin{figure}

    \begin{subfigure}{.49\textwidth}
        \centering 
        \includegraphics[width=\linewidth]{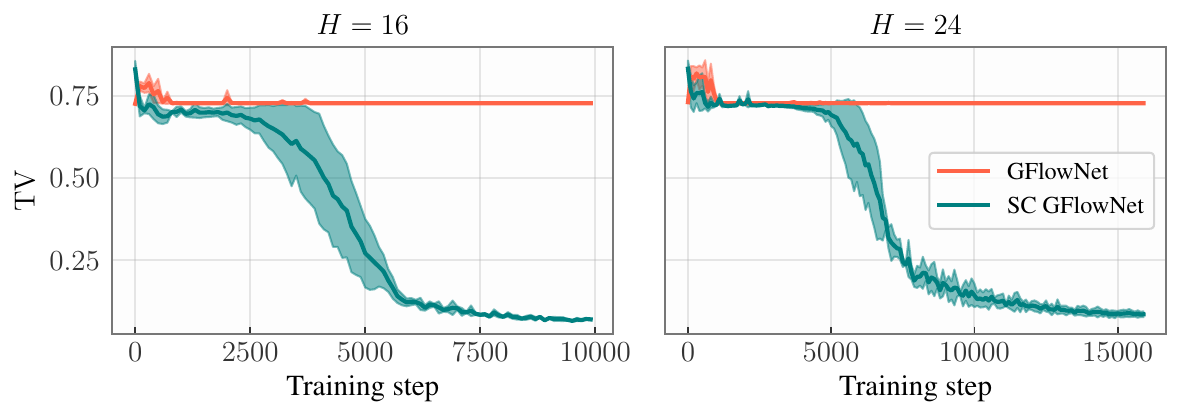}
    \end{subfigure}
    \begin{subfigure}{.49\textwidth}
        \centering 
        \includegraphics[width=\linewidth]{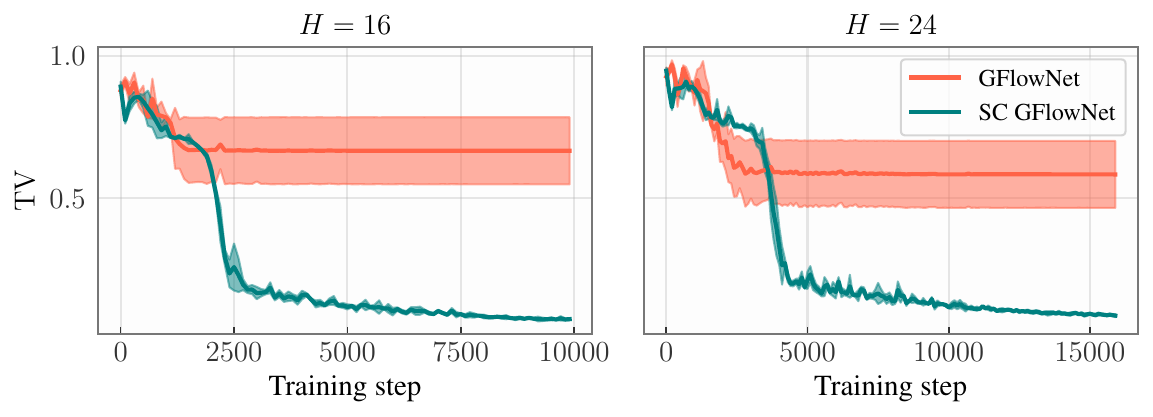}
    \end{subfigure}

    \begin{subfigure}{.49\textwidth}
        \raggedleft 
        \includegraphics[width=.94\linewidth]{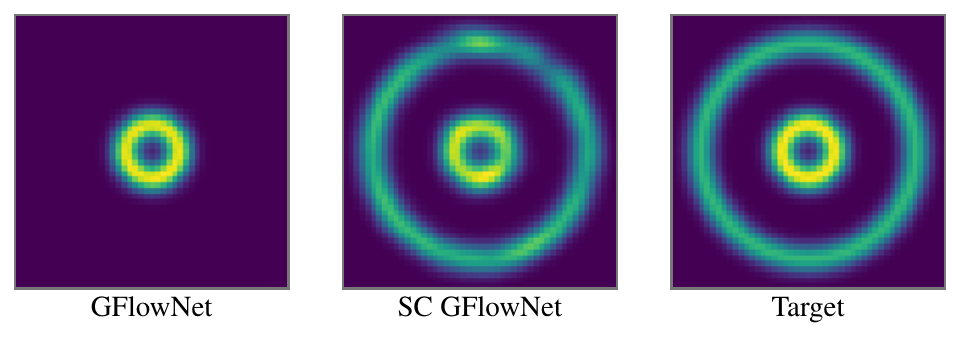}
        \caption{\textsc{Rings}.} 
    \end{subfigure}
    \begin{subfigure}{.49\textwidth}
        \raggedleft 
        \includegraphics[width=.94\linewidth]{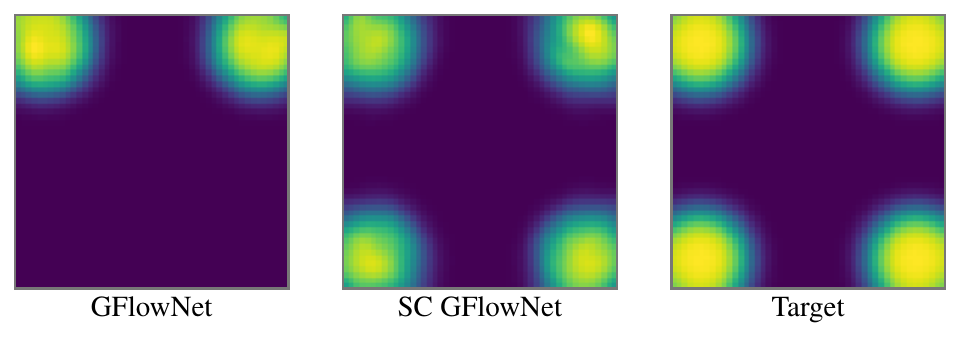}
        \caption{\textsc{Gaussian Mixture}.} 
    \end{subfigure}

     \caption{\textbf{SC GFlowNets improve goodness-of-fit} for the \textsc{Lazy Random Walk} task, exhaustively covering the target distribution's high probability regions even for an imperfect state space partitioning (as for the \textsc{Gaussian Mixture}; see \Cref{fig:prandomwalksa}).}
     \label{fig:randomwalksa}
\end{figure}

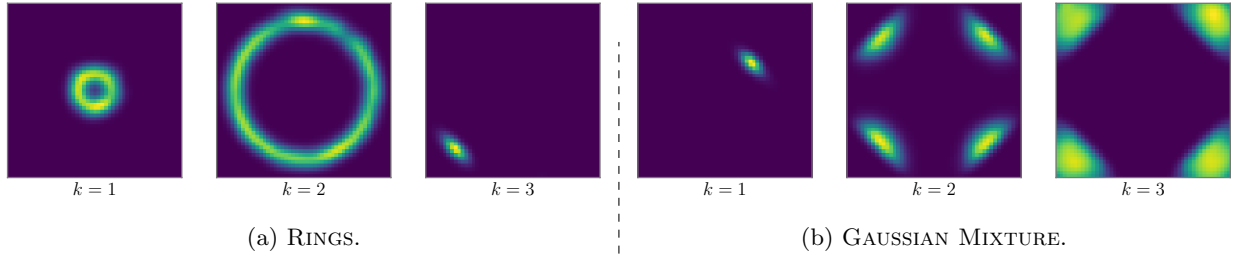
\begin{figure}
    \begin{subfigure}{.49\textwidth}
        \centering 
        \includegraphics[width=\linewidth]{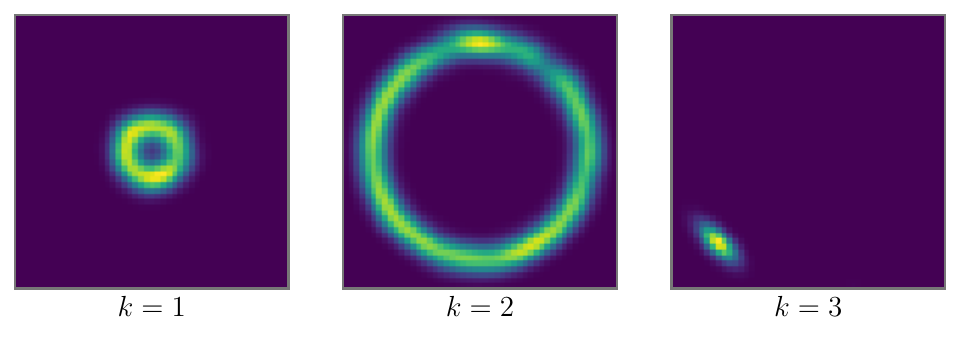}
        \caption{\textsc{Rings}.}
    \end{subfigure}
    \begin{tikzpicture}
        \draw[dashed] (0,0) -- (0, 2.8); 
    \end{tikzpicture}
    \begin{subfigure}{.49\textwidth}
        \includegraphics[width=\linewidth]{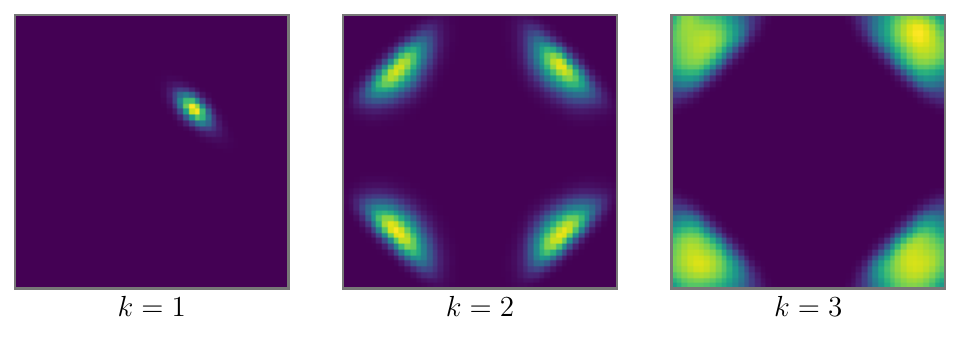}
        \caption{\textsc{Gaussian Mixture}.}
    \end{subfigure}
    \caption{
        SC GFlowNet's marginal $m_{F}^{\smash[t]{(k)}}(s_{o}, x)$ (recall \Cref{eq:marginalsa}) in each $x \in \mathcal{X}$ for both the (a) \textsc{Rings} and (b) \textsc{Gaussian Mixture} target distributions of the \textsc{Lazy Random Walk} domain.
        \looseness=-1 
    }
    \label{fig:prandomwalksa} 
\end{figure}

\Cref{fig:randomwalksa} shows both the model's accuracy throughout training and a snapshot of the learned distribution for both SC and a traditional GFlowNet.
In addition, we present in \Cref{fig:prandomwalksa} each component $p_{F}^{\smash[t]{(k)}}$ for $k \in \{1, 2, 3\}$ of the SC GFlowNet.
We notice that our mixture modelling framework both accelerates learning convergence and improves fitness---even when the chosen partitioning does not separate the target distribution's high probability regions, as 
in the \textsc{Gaussian Mixture} example. 
The intuition for this is that the sampler will naturally avoid being pulled back to and bogged in previously visited modes during training when we constraint the future value of $\|\mathbf{x}\|_{1}$ at $t = T$. 
We also compare ASC to SC GFlowNets in \Cref{fig:pprandomwalksa} for both examples; 
as expected, ASC GFlowNets find a distinctly better fit to the target---particularly in the case of \textsc{Rings}, for which each component receives an easier target with connected high-probability regions.

\begin{figure}
    \begin{subfigure}{.24\textwidth}
        \centering 
        \includegraphics[width=\linewidth]{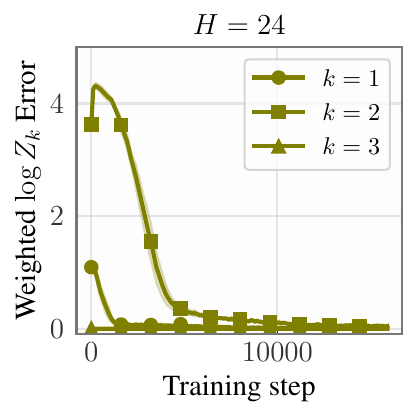}
        \caption{\textsc{Rings}.} 
    \end{subfigure}
    \begin{subfigure}{.24\textwidth}
        \centering 
        \includegraphics[width=.95\linewidth]{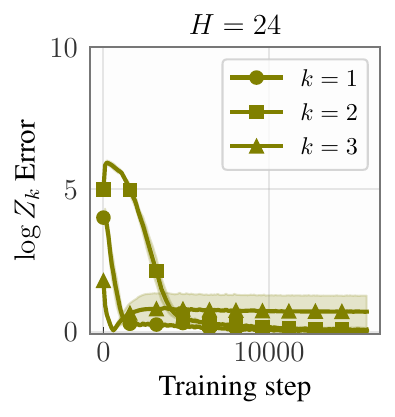}
        \caption{\textsc{Rings}.} 
    \end{subfigure}
    \begin{subfigure}{.24\textwidth}
        \centering 
        \includegraphics[width=\linewidth]{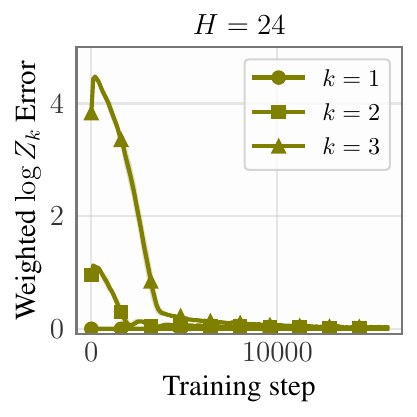}
        \caption{\textsc{Gaussian Mixture}.}
    \end{subfigure}
    \begin{subfigure}{.24\textwidth}
        \centering 
        \includegraphics[width=.95\linewidth]{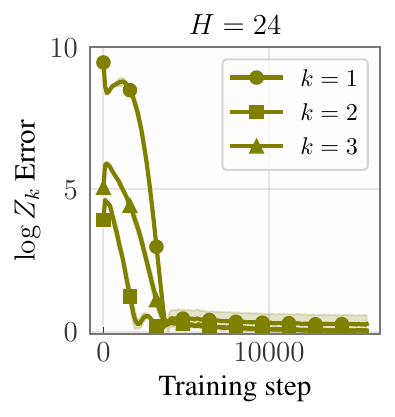}
        \caption{\textsc{Gaussian Mixture}.} 
    \end{subfigure}
    \caption{
        Learned log-partition functions by ASC GFlowNets in the \textsc{Lazy Random Walk} domain. 
        We recall that, for \textsc{Rings}, $k = 3$ is the \emph{starved} component, containing negligible probability mass; for \textsc{Gaussian Mixture}, it is $k = 1$. 
        As expected, ASC GFlowNets learn the correct values 
        (a, c); however, the approximation of $\log Z_{k}$ for regions with near-zero probability 
        is imperfect (b, d), as sizeable deviations in $\log Z_{k}$ in this case have little effect on distributional accuracy. \looseness=-1 
    } 
    \label{fig:partitionsaa} 
\end{figure}

\begin{figure*}
    \begin{subfigure}{.49\linewidth}
        \includegraphics[width=.8\linewidth]{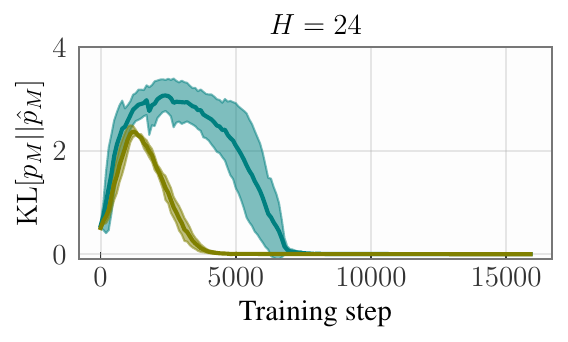}
        \caption{\textsc{Rings}.}
    \end{subfigure}
    \begin{subfigure}{.49\linewidth}
        \includegraphics[width=.8\linewidth]{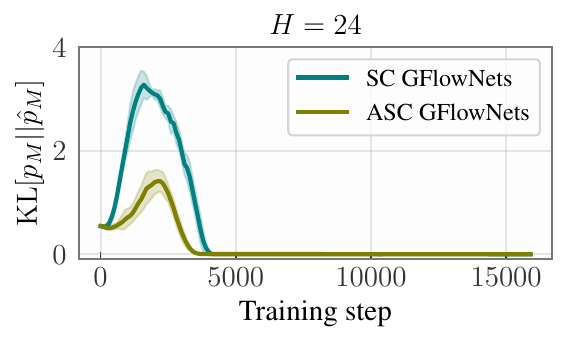}
        \caption{\textsc{Gaussian Mixture}.}
    \end{subfigure}
    \caption{KL between prior mixing distributions.} 
    \label{fig:partitionsaaaa} 
\end{figure*}
Importantly, due to having a tractably enumerable state space, the \textsc{Lazy Random Walk} domain allows us to efficiently evaluate the log-partitions functions. 
This provides another road for evaluating our model. 
To see this, we denote by $\log \hat{Z}_{k}$ and $\log \hat{Z}$ the learned partition functions, and by $\log Z_{k}$ and $\log \hat{Z}$ their ground-truth values for $k \in [K]$.
We consider the metrics \looseness=-1 
\begin{equation*}
    \mathrm{KL}[p_{M} || \hat{p}_{M}] \coloneqq \sum_{1 \le k \le K} \frac{Z_{k}}{Z} \cdot \log \frac{\nicefrac{Z_{k}}{Z}}{\nicefrac{\hat{Z}_{k}}{\hat{Z}}}, \, \left( \frac{Z_{k}}{Z} \cdot \left| \log \hat{Z}_{k} - \log Z_{k} \right| \right), \text{ and } \left( \left| \log \hat{Z}_{k} - \log Z_{k} \right| \right),   
\end{equation*}
for $k \in [K]$, which respectively evaluate the accuracy of the learned prior mixing distribution and the relative and absolute errors of the learned partition functions. 
We separate relative and absolute metrics as, in practice, if a $Z_{k}$ is negligible when compared to other $Z_{k'}$, small discrepancies in ${Z}_{k}$ will result in large deviations for $\log \hat{Z}_{k}$. 
As an example, the distributions with logits $(1, 2, -10)$ and $(1.05, 1.99, -15)$ are effectively the same, despite the large difference in the values of the third component. 
We display these values in \Cref{fig:partitionsaa,fig:partitionsaaaa}, with a comparison between SC and ASC GFlowNets according to them in the latter. 

\begin{figure}
    \centering 
    \begin{subfigure}{.49\textwidth}
        \centering 
        \includegraphics[width=.59\linewidth]{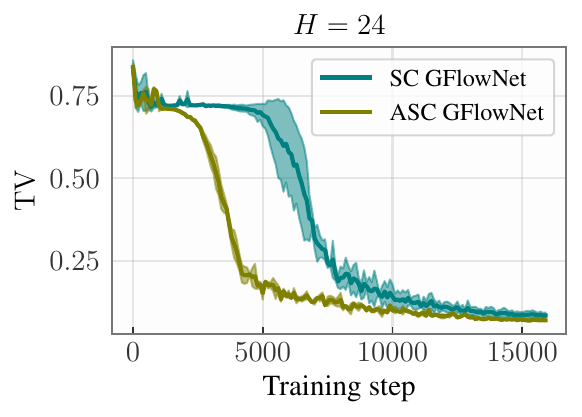}
        \includegraphics[width=.39\linewidth]{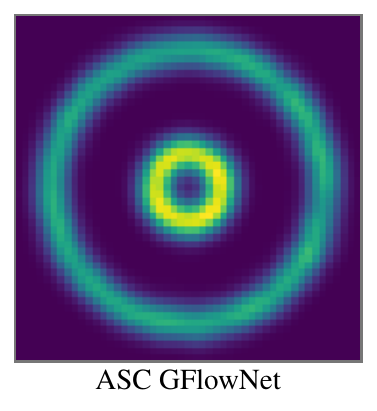}
        \caption{\textsc{Rings}.}
    \end{subfigure}
    \begin{subfigure}{.49\textwidth}
        \includegraphics[width=.59\linewidth]{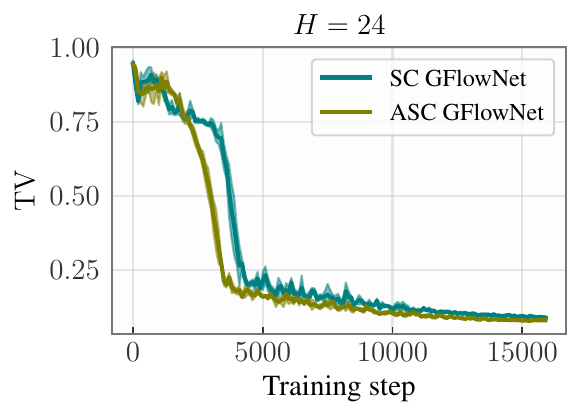}
        \includegraphics[width=.39\linewidth]{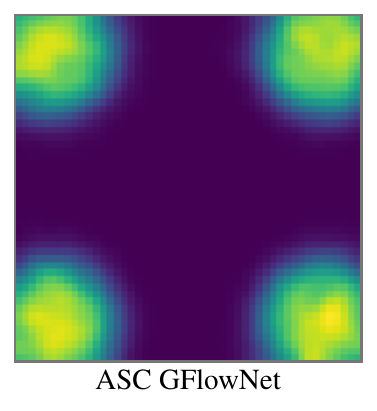}
        \caption{\textsc{Gaussian Mixture}.} 
    \end{subfigure}
    \caption{
        \textbf{ASC GFlowNets converge faster than their synchronous counterparts} in the \textsc{Lazy Random Walk} task, achieving a more accurate 
        approximation in fewer training iterations.
        \looseness=-1 
    } 
    \label{fig:pprandomwalksa} 
\end{figure}

\begin{figure}
    \centering 
    \begin{subfigure}{.47\textwidth}
        \centering 
        \includegraphics[width=\linewidth]{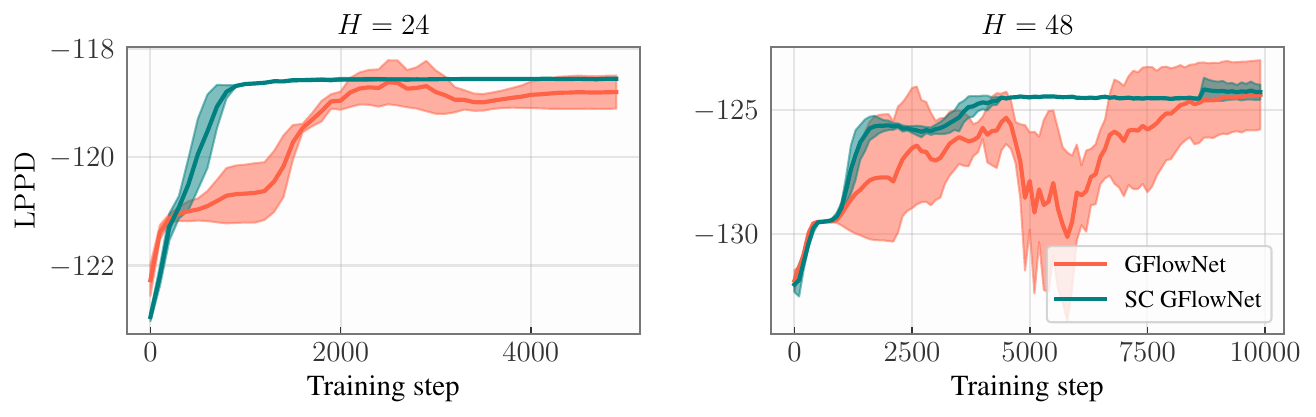}
        \caption{Log-predictive posterior density of the observed data (LPPD).}
    \end{subfigure}
    \hspace{12pt} 
    \begin{subfigure}{.47\textwidth}
        \centering 
        \includegraphics[width=\linewidth]{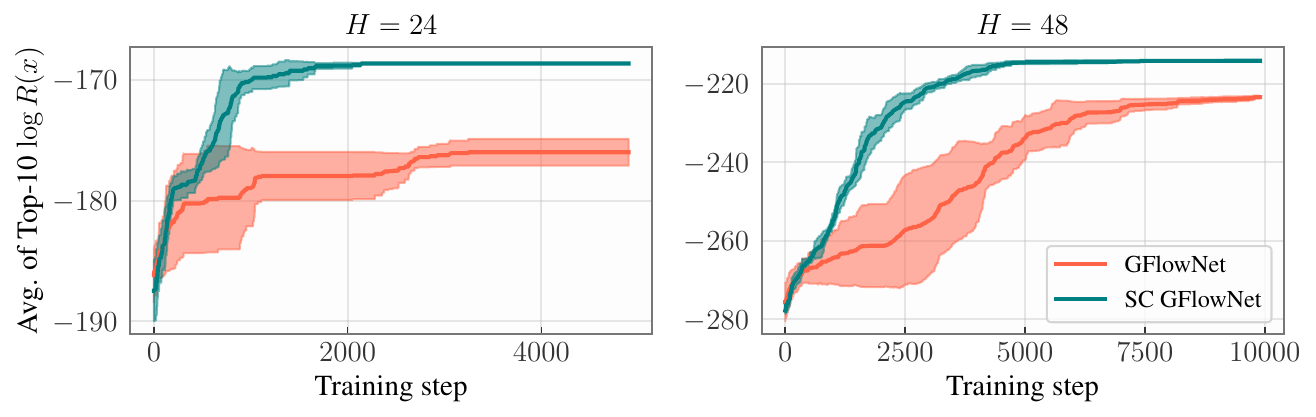}
        \caption{Average unnormalized target density of the 10 most likely encountered states.}
    \end{subfigure}
    \caption{
        \textbf{SC GFlowNets improve learning convergence} (left) and accelerate the discovery of high-probabiltiy states (right) for the Bayesian \textsc{variable selection} task; $H$ represents the number of variables (columns) in the dataset.
    }
    \label{fig:varselsa} 
\end{figure}

\begin{wrapfigure}[12]{l}{.5\textwidth}
    \vspace{-12pt} 
    \begin{subfigure}{.49\linewidth}
        \includegraphics[width=\linewidth]{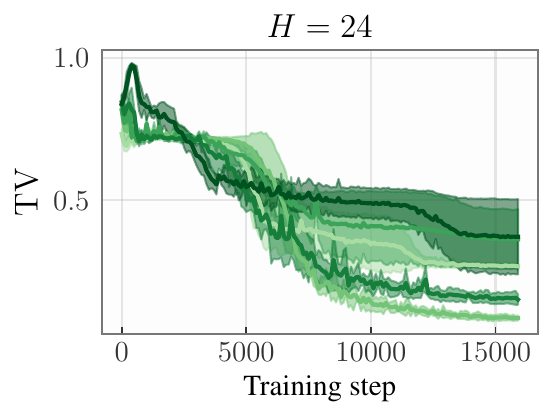}
        \caption{\textsc{Rings}.} 
    \end{subfigure}
    \begin{subfigure}{.49\linewidth}
        \includegraphics[width=\linewidth]{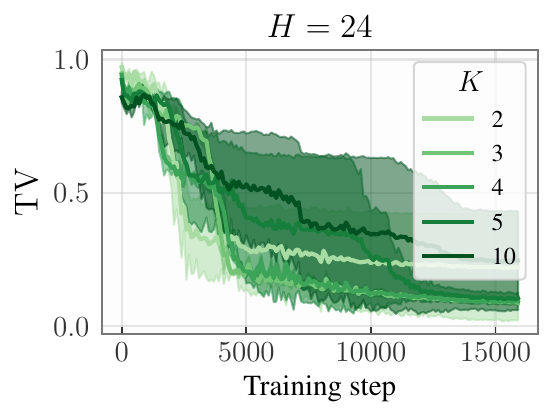}
        \caption{\textsc{Four Gaussians}.}
    \end{subfigure}
    \caption{SC GFlowNets' distributional accuracy for distinct numbers of partitions.}  
    \label{fig:randomwalksaaa}
\end{wrapfigure}
We once again consider the problem of choosing $K$. 
\Cref{fig:randomwalksaaa} highlights the dependence of SC GFlowNets's accuracy on $K$ for the \textsc{Lazy Random Walk} domain. 
Although SC GFlowNets improve upon a monolithic sampler for all evaluated choices of $K \ge 2$, there is also a noticeable difference between them. 
In view of this, we provide two separate suggestions for selecting this hyperparameter. 
The first is to always let $K = 3$ components, should this choice be compatible with the chosen TC; despite its simplicity, we found setting $K = 3$ to consistently achieve an accurate approximation to the target distributions we considered. 
The second is to rely upon standard hyperparameter selection protocols, such as carrying out short training runs and measuring the sampler's performance afterwards, to pick $K$.
With this in mind, future research should investigate principled approaches for algorithmic picking of $K$.  
Nonetheless, 
a useful intuition to have is that the complexity of the function learned by our component-wise amortized policy increases with $K$, due to which we advise against setting too large a value for the number of components in SC GFlowNets. 
\looseness=-1

The main takeaway from these experiments, however, is that designing an effective covering for SC GFlowNets \emph{does not require} fine-grained prior knowledge about the shape of the target distribution.



\begin{figure}
    \begin{subfigure}{.42\textwidth}
        \centering 
        \includegraphics[width=.9\linewidth]{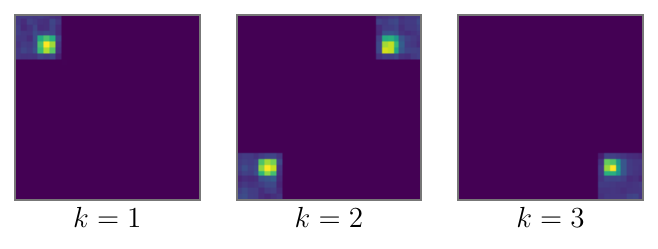}
        \caption{Length-based TC.}
    \end{subfigure}
    \begin{subfigure}{.55\textwidth}
        \centering 
        \includegraphics[width=.9\linewidth]{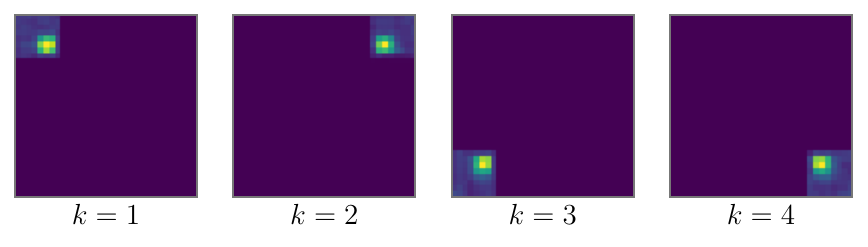}
        \caption{Octant-based TC.} 
    \end{subfigure}
    \caption{Component-wise learned distributions for both considered SC GFlowNets in \Cref{fig:hypergridsaa}.} 
    \label{fig:hypergridsaaa} 
    \vspace{-6pt} 
\end{figure}

\pp{Variable Selection.} 
We consider the task of Bayesian variable selection \citep{george1993variable}. 
Given a data set $\mathcal{D} = \{(x^{\smash[t]{(i)}}, y^{\smash[t]{(i)}})\}_{i=1}^{n}$ with $x^{\smash[t]{(i)}} \in \mathbb{R}^{H}$ and $y^{\smash[t]{(i)}} \in \mathbb{R}$, we let $p(y|x, S, \beta) = \mathcal{N}(y | \beta_{S}^{T} x_{S}, \sigma^{2})$ be defined according to a linear model with observational noise $\sigma$ and coefficients $\beta \in \mathbb{R}^{H}$ and features $S \subseteq [H]$.
We here denote by $x_{S}$ the vector $(x^{\smash[t]{(i)}})_{i \in S}$ composed of the indices $i$ in $S$, with a similar definitiong for $\beta_{S}$. 
In this setting, we train a sampler to draw from the marginal posterior distribution over $S$, i.e., 
\begin{equation*}
    R(S)  \coloneqq p(S | \mathcal{D}) = \int_{\mathbb{R}^{H}} p(S, \beta | \mathcal{D}) \mathrm{d}\beta, 
\end{equation*}
with $p(S, \beta | \mathcal{D})$ being the usual Bayesian posterior for $S$ and $\beta$.
We notice $p(S | \mathcal{D})$ can be computed in closed form up to a normalizing constant independent of $S$ for a linear model. 
After training, we evaluate the samplers according to the predictive posterior distribution of the observed data, as it is usually done in Bayesian modelling, and to the average unnormalized log-posterior of the 10 most likely variable sets, a common approach for assessing exploration effectiveness in the GFlowNet literature, e.g., \citep{kim2024learningscalelogitstemperatureconditional, kim2023local, kim2025adaptive}.   
We show both metrics in \Cref{fig:varselsa}, with the SC GFlowNet being designed as in the \textsc{Set Generation} domain. 
Notably, our mixture model improves both learning and the rate at which modes are encountered during training. \looseness=-1  






\subsection{Experimental Details} 

We provide further details for how each domain was implemented, including model architecture, mixture-encoding details, optimizer, training iterations, and batch size per gradient step. 
We ran all our experiments in a MacBook Pro M4 with 16 GB of RAM, and we implemented our models using Jax's Flax library \citep{jax2018github, flax2020github}. 
All reported results, including error bars, represent averages and standard deviations computed  over three independent runs. 

In every experiment, we used the Muon optimizer \citep{jordan2024muon} for the neural network's parameters, and Adam \citep{kingma2014adam} for the $\log Z$ and $\{\log Z_{k}\}_{k \in [K]}$ parameters. 
Unless stated otherwise, we parameterized the policies $p_{F}$ and $p_{B}$ of the GFlowNet using an MLP with two 64-dimensional layers, 
\begin{equation*}
    p_{F}(s, \cdot) = \mathrm{Softmax}\left( \mathrm{MLP}(\iota(s)) \odot \mathbf{m}(s) \right),  
\end{equation*}
with $\iota(s) \colon \mathcal{S} \rightarrow \mathbb{R}^{64}$ being a learned vector embedding for the state, and $\mathbf{m}(s) \in \{1, -\infty\}^{f}$ being masking function that ensures the sampler is constrained to the target region (either $\mathcal{X}$ or $\mathcal{X}_{k}$, in the case of SC GFlowNets); $f$ is the maximum number of children a given state $s$ has in the underlying state graph. 
We implement 
a similar 
representation for $p_{B}$, 
using a shared encoder and simply changing the output layer and the mask. 
In its centralized form, SC GFlowNets are implemented as \looseness=-1  
\begin{equation} \label{eq:modelsaa} 
    p_{F}^{(\gamma)}(s, \cdot) = \mathrm{Softmax}(\mathrm{MLP}(\iota(s) \oplus \iota_{\gamma}(\gamma)) \odot \mathbf{m}(s, \gamma)). 
\end{equation}
Again, $\mathbf{m}(s, \gamma)$ is a masking function, which depends on $\gamma$ in the case of SC GFlowNets but not in the case of RF GFlowNets; $\oplus$ represents concatenation. 
To ensure matching computational costs and similar parameter counts for both SC GFlowNets and GFlowNets, we let $\iota \colon \mathcal{S} \rightarrow \mathbb{R}^{64 - d_{\gamma}}$ and $\iota_{\gamma} \colon \Gamma \rightarrow \mathbb{R}^{d_{\gamma}}$, with $d_{\gamma}$ being the dimension of the learned representation for $\gamma$. 
For SC GFlowNets, $\iota_{\gamma}$ is a learned dictionary from $[K]$ to $\mathbb{R}^{d_{\gamma}}$. 
For RF GFlowNets, $\iota_{\gamma}$ is a linear layer. 
We set $d_{\gamma} = 8$ in all experiments, which we found to be consistently effective. 
ASC GFlowNets, on the other hand, was trained with $K$ identical copies of \Cref{eq:modelsaa}; as each model learns 
to sample from a distinct $\mathcal{X}_{k}$, information about $k$ was not required to be explicitly encoded into the neural network. 

\pp{Set generation.} 
We considered sets of size $\{24, 48, 64\}$, and trained each sampler for $3000$, $5000$, $10000$ iterations, respectively, using a batch size of $64$ trajectories.

\pp{Hypergrid.}
We trained the models for $H \in \{16, 32\}$, both of which over $16000$ iterations and a batch size of $64$ trajectories per step. The reward function is the same as in \citep{malkin2022trajectory}, i.e., \looseness=-1 
\begin{equation}
    \log R(\mathbf{x}) = \log \left( 
        \frac{1}{2} \prod_{i=1}^{d} \mathbbm{1}(a_i > 0.5) 
        + 2 \prod_{i=1}^{d} \mathbb{I}(0.6 < a_i < 0.8) 
        + \rho 
    \right), 
\end{equation}
with $\rho = 10^{-3}$ and $a_{i} = 2 \cdot \mathbf{x}_{i} \cdot (H - 1)^{-1} - 1$. 

\pp{Ancestral Graphs.} 
We considered two distinct tasks under the \textsc{Ancestral Graphs} domain. 
The first, considered also in \citep{silva2026causal}, considered graphs with $H \in \{15, 20\}$ nodes and used the BIC for scoring each AG, 
\begin{equation} \label{eq:bicsa}  
    R(x) = -\mathrm{BIC}(x) = 2 \log \hat{p}(\mathcal{D} | x) - |x|_{E} \log n - 2 |x|_{E} \log |x|_{N}, 
\end{equation}
in which $\hat{p}$ is the maximum likelihood of $\mathcal{D}$ given $x$, computed via \citet{drton2004iterative}'s Iteratively Reweighted Least Squares (IRLS) algorithm, $|x|_{N}$ is the number of nodes ($H$) in $x$, $|x|_{E}$, the number of edges, and $\mathcal{D}$ was drawn from a randomly parameterized linear Gaussian structural equation model with $H$ variables.
We used the same data generative process as in \citep{silva2026causal}. 
\looseness=-1 

The second task consisted of simulating the Erd\H{o}s Renyi model, whose unnormalized density is defined as 
\begin{equation*}
    R(x) = \alpha^{|x_{E}|}. 
\end{equation*}
We restrict the GFlowNet's generative process to directed acyclic graphs (DAGs) by masking out bidirected edges.
Under this condition, both the partition function $Z_{H}(\alpha)$ of $R(x)$ and the probability that there is an edge between each pair of nodes (the same for all of them, due to symmetry) can be computed in closed form and compared against that learned by our sampler. 
We notice that these quantities are actually related, and both of them can be computed efficiently by adapting Robinson's formula \citep{robinson1973counting}, which we quickly review. 
Fix any pair $(a, b)$ of nodes. 
Let $\mathcal{X}_{H}$ be the space of all DAGs with $H$ nodes. 
The probability that there is an edge from $a$ to $b$ in the graph is 
\begin{equation*}
    E_{H}(\alpha) = \frac{\sum_{x \in \mathcal{X}, (a \rightarrow b ) \text{ in } x} \alpha^{|x_{E}|}}{Z_{H}(\alpha)} \text{ with } Z_{H}(\alpha) = \sum_{x \in \mathcal{X}_{H}} \alpha^{|x_{E}|}.  
\end{equation*}
To compute the latter, pick any subset of size $d$ of nodes, which will be deemed roots of the DAG, for $d \in \{1, \dots, H\}$; in a DAG, there is always at least one.
There are $\binom{H}{d}$ ways of choosing them. 
From them, we may pick each of the $H - d$ remaining nodes to decide whether there is an edge between them and the root nodes, totaling a maximum of $d(H - d)$ edges. 
Given this choice, the remaining $H - d$ nodes may form any DAG. 
In our summation, the corresponding term is 
\begin{equation*}
    m(d, \alpha) = \binom{H}{d} \cdot Z_{H - d}(\alpha) \cdot \sum_{0 \le k \le d(H - d)}  \binom{d (H - d)}{k} \alpha^{k} = \binom{H}{d} \cdot Z_{H - d}(\alpha) \cdot (1 + \alpha)^{d(H - d)}.  
\end{equation*}
However, summing this expression over $d \in \{1, \dots, H\}$ over-counts certain graphs; e.g., the edgeless graph is counted several times when $k = 0$ for each $d$. 
More comprehensively, let $R_{H, d}$ be the set of all DAGs having exactly $d$ nodes as roots, and let $c(d, \alpha) = \sum_{x \in R_{H, d}} \alpha^{|x|_{E}}$; our objective is to compute $\sum_{1 \le d \le H} c(d, \alpha)$. 
However, the expression above evaluates 
\begin{equation*}
    m(d, \alpha) = \sum_{d \le j \le H} \binom{j}{d} \cdot c(j, \alpha), 
\end{equation*}
since $R_{H, j}$ appears $\binom{j}{d}$ times when evaluating $m(d, \alpha)$; to see this, fix a set $\mathrm{n} = \{n_{1}, \dots, n_{j}\}$ of $j$ nodes, and notice that our computation above considers the case in which every $n \in \mathrm{n}$ is root each time we pick any $d$-sized subset of $\mathrm{n}$, i.e., $\binom{j}{d}$ times, for $d \le j$, as our computation considers DAGs with \emph{at least} $d$ roots. 
We can circumvent this issue by instead computing  
\begin{equation*}
    \begin{aligned} 
        \sum_{1 \le d \le H} (-1)^{d + 1} m(d, \alpha) &= \sum_{1 \le d \le H} \sum_{d \le j \le H} (-1)^{d + 1} \cdot \binom{j}{d} \cdot c(j, \alpha) \\ 
        &= \sum_{1 \le j \le H} c(j, \alpha) \cdot \sum_{1 \le d \le j} (-1)^{d + 1} \cdot \binom{j}{d} \\ 
        &= \sum_{1 \le j \le H} c(j, \alpha) \cdot \left(1 - \sum_{0 \le d \le j} (-1)^{d} \cdot \binom{j}{d}\right) \\ 
        &=\sum_{1 \le j \le H} c(j, \alpha) \left(1 - (1 - 1)^{j}\right) = \sum_{1 \le j \le H} c(j, \alpha).
    \end{aligned} 
\end{equation*}
In conclusion, 
\begin{equation*}
    Z_{H}(\alpha) = \sum_{1 \le d \le H} (-1)^{d + 1} \cdot \binom{H}{d} \cdot Z_{H - d}(\alpha) \cdot (1 + \alpha)^{d (H - d)}. 
\end{equation*}
This can be computed recursively via the boundary condition $Z_{1}(\alpha) = 1$. 
To evaluate $E_{H}(\alpha)$, we notice the probability of existing an edge from $a$ to $b$ is equal, by symmetry, to the expected number of edges divided by the total number of edges.
On the other hand, the number of edges follows the exponential family $|x|_{E} \mapsto \exp\left\{ |x|_{E} \log \alpha - \log Z_{H}(\alpha) \right\}$ with cumulant $\log Z_{H}(\alpha)$; as a consequence, its expected value can be computed as the derivative of the cumulant, 
\begin{equation*}
    \mathbb{E}_{x}[|x|_{E}] = \frac{\mathrm{d} \log Z_{H}(\alpha)}{\mathrm{d}\log \alpha} = \alpha \cdot \frac{\mathrm{d}\log Z_{H}(\alpha)}{\mathrm{d}\alpha}, 
\end{equation*}
which can be computed efficiently through our previously discussed recursive formula.
The probability of there being an edge between any pair of nodes under Erd\H{o}s-Renyi's model is then 
\begin{equation*}
    E_{H}(\alpha) = \frac{1}{H(H - 1)} \cdot \alpha \cdot \frac{\mathrm{d}\log Z_{H}(\alpha)}{\mathrm{d}\alpha}. 
\end{equation*}
During training, we compare $E_{H}(\alpha)$ with the edge-wise probabilities learned by our sampler, estimated via Monte Carlo with $2048$ samples; the average $L_{1}$ between both quantities is the value reported in \Cref{fig:graphsaa}, called $\mathrm{Marginal \, Edge \, Error}$ and evaluated as  
\begin{equation*}
    \mathrm{Marginal \, Edge \, Error} = \frac{2}{H(H - 1)} \sum_{n_{1}, n_{2} \in [H], n_{1} \neq n_{2}}  | E_{H}(\alpha) - \hat{\pi}(n_{1}, n_{2})| 
\end{equation*}
with $\hat{\pi}(n_{1}, n_{2})$ being the referred Monte Carlo estimate. 
\looseness=-1 

\pp{Autoregressive sequence generation.}
We set $p_{\mathrm{length}}$ to an uniform distribution and both $p_{\mathrm{init}}$ and $p_{\mathrm{kernel}}$ to randomly sampled distribution with logits drawn according a standard multivariate normal model of dimension $L \in \{18, 24\}$.
We train the samplers for $8000$ iterations in both cases. 

\pp{Lazy Random Walk.} 
We consider the follow distributions for the \textsc{Rings} and \textsc{Gaussian Mixture} targets, respectively, 
\begin{equation*}
    \log R_{\text{Rings}}(\mathbf{x}) = \log \left( \sum_{k=1}^{K} w_k \exp \left( -\frac{1}{2} \left( \frac{\|\mathbf{x}\|_2 - r_k}{s} \right)^{\!2} \right) \right), 
\end{equation*}
with $K = 2$, $r_{1} = 0.2$, $r_{2} = 0.8$, $w_{1} = 0.6$, $w_{2} = 0.4$, and $s = 0.05$, and 
\begin{equation*}
    \log R_{\textsc{GaussianMixture}}(\mathbf{x}) =  \left[ \log \left( \sum_{j=1}^{4} \exp \left( -\frac{1}{2} \left( \frac{\|\mathbf{x} - \boldsymbol{\mu}_j\|_2^2}{\sigma} \right)^{\!2} \right) \right) - \log \left( 4\pi \sigma \sqrt{2\pi} \right) \right], 
\end{equation*}
with $\sigma = 0.1$ and $\mu_{1} = (-0.8, -0.8)$, $\mu_{2} = (-0.8, 0.8)$, $\mu_{3} = (0.8, -0.8)$, and $\mu_{4} = (0.8, 0.8)$.  
In both cases, we normalize $\mathbf{x}$ to $[-1, 1]$ through the transformation $\mathbf{x} \mapsto H^{-1} \mathbf{x}$. 

\pp{Variable Selection.} 
We consider $H \in \{24, 48\}$ and train both models for $5000$ and $10000$ iterations, respectively, and we used a batch size of $64$ trajectories per training step.

\section{Proofs} \label{sec:psaa}

We provide detailed proofs for each theoretical result in both the main text and the supplement.

\subsection{Proof of Proposition~\ref{prop:relsa}}

As each $\gfn_{\gamma}$ abides by the MTB condition, we observe that 
\begin{equation*}
    Z(\mathrm{d}\gamma) \cdot p_{F}^{\smash[t]{(\gamma)}}(s_{o}, \tau) = R(x, \mathrm{d}\gamma) \cdot p_{B}^{\smash[t]{(\gamma)}}(x, \tau)  
\end{equation*}
for each $\gamma \in \Gamma$ and trajectory $\tau$ from $s_{o}$ to $x$. 
We then infer that, by marginalizing $\tau \in s_{o} \rightsquigarrow x$, 
\begin{equation*}
    Z(\mathrm{d}\gamma) \sum_{\tau \in s_{o} \rightsquigarrow x} p_{F}(s_{o}, \tau) = R(x, \mathrm{d}\gamma) \underbrace{\sum_{\tau \in s_{o} \rightsquigarrow x} p_{B}^{\smash[t]{(\gamma)}}(x, \tau)}_{= 1} = R(x, \mathrm{d}\gamma). 
\end{equation*}
In addition, the left-hand side of the above equation is exactly $m_{F}^{\smash[t]{(\gamma)}}(x)$ (by definition). Hence, 
\begin{equation*}
    Z(\mathrm{d}\gamma) m_{F}^{\smash[t]{(\gamma)}}(x) = R(x, \mathrm{d}\gamma). 
\end{equation*}
This shows the second equality in Proposition~\ref{prop:relsa}.
By summing the above equation over $x \in \mathcal{X}$ and noticing that $m_{F}^{\smash[t]{(\gamma)}}$ is a probability measure on $\mathcal{X}$, we obtain 
\begin{equation*}
    Z(\mathrm{d}\gamma) = Z(\mathrm{d}\gamma) \underbrace{\sum_{x \in \mathcal{X}} m_{F}^{\smash[t]{(\gamma)}}(x)}_{= 1} = \sum_{x \in \mathcal{X}} R(x, \mathrm{d}\gamma). 
\end{equation*}
This demonstrates the first equality in Proposition~\ref{prop:relsa}. 

\subsection{Proof of Corollary~\ref{col:pmassa}}

Under the notation of Proposition~\ref{prop:relsa}, we observe that 
\begin{equation*}
    R_{k}(x) = R(x, \{k\}) = R(x) q_{k}(x) = R(x) N(x)^{-1}  
\end{equation*}
if $x \in \mathcal{X}_{k}$, and $R(x, \{k\}) = 0$ otherwise. 
This proves the second equality in Corollary~\ref{col:pmassa}. 
The first equality can be derived by recalling that $Z_{k} \coloneqq \sum_{x \in \mathcal{X}_{k}} R_{k}(x)$ and hence 
\begin{equation*}
    Z_{k} = \sum_{x \in \mathcal{X}_{k}} \frac{R(x)}{N(x)}. 
\end{equation*}
This is exactly the first equality in Corollary~\ref{col:pmassa}. 

\subsection{Proof of Proposition~\ref{prop:npsa}}

We demonstrate there are numbers $\{a_{k}, b_{k}\}_{k \in [K]}$ and a function $\psi$ for which the problem 
\begin{equation*}
    \exists s' \in 2^{\mathcal{C} \setminus s} \text{s.t. } s \cup s' \in \mathcal{X} \text{ and } \psi(s \cup s') \in [a_{k}, b_{k})
\end{equation*}
is equivalent to \texttt{SubsetSum}. 
We assume $\mathcal{C} = [C]$ for some $C \in \mathbb{N}$. 
We first notice that since $\psi$ is modular, i.e., 
\begin{equation*}
    \psi(s \cup \{c\}) = \psi(s) + \psi(\{c\})  
\end{equation*}
for each $c \in \mathcal{C} \setminus s$, we may write 
\begin{equation*}
    \psi(s) = \psi(\emptyset) + \sum_{c \in s} \psi(\{c\}). 
\end{equation*}
As such, we define $\eta_{c} = \psi(\{c\}) \ge 0$ for conciseness and assume $\psi(\emptyset) = 0$ without loss of generality. 
In this case, $\psi(s)$ can be computed in linear time in the size of $s$. 
We also assume $\eta_{c}$ are positive integers and fix $s = \emptyset$. 
Under these conditions, our decision problem in \Cref{eq:decisionsa} becomes 
\begin{equation*}
    \exists s \in 2^{\mathcal{C}} \colon \sum_{c \in s} \eta_{c} \in [a_{k}, b_{k}). 
\end{equation*}
Let there be a $k$ for which $b_{k} = a_{k} + 1$ with both $a_{k}$ and $b_{k}$ positive integers; as $\eta_{c}$ are also integers, the above sum can only be satisfied when 
\begin{equation*}
    \sum_{c \in s} \eta_{c} = a_{k} 
\end{equation*}
for some subset $s$ of $\mathcal{C}$. 
This is exactly the subset sum problem (with positive inputs), which is NP-complete \citep{clrs2009}.
In particular, deciding upon the reachability of $\mathcal{X}_{k}$ from a given $s \in \mathcal{S}$ is NP-complete in general. 

\subsection{Proof of Proposition~\ref{prop:npsaa}}

We show the decision problem in \Cref{eq:decisionsa} can be efficiently solved when $\psi$ is a symmetric junta by relying upon the fact that, since $\mathcal{C}$ is finite, there are only so many values realizable by $\psi$, all of which can be quickly verified in polynomial time.

Fix any $s \in \mathcal{S}$ with $n = \#(s \cap r)$. 
We notice $n \le \#( \mathcal{C} \cap r )$. 
Also, $\psi(s) = \xi(n)$ by definition.  
Our decision problem then becomes 
\begin{equation*}
    \exists s' \in 2^{\mathcal{C} \setminus s} \colon \psi(s \cup s') \in [a_{k}, b_{k}). 
\end{equation*}
As $s$ and $s'$ are disjoint, $\# ( (s \cup s') \cap r) = \# (s \cap r) + \# (s' \cap r)$. 
Similarly, $\# (s' \cap r) \le \# (\mathcal{C} \cap r) - n$. 
And not only this, but $\#(s' \cap r)$ can also assume \emph{any value} within $\{0, \dots, \# (\mathcal{C} \cap r) - n\}$ as $s' \in 2^{\mathcal{C} \setminus s}$.
As such, deciding whether the above problem is solvable is equivalent to deciding whether there is a $m \in \{0, \dots, \# (\mathcal{C} \cap r) - n\}$ for which 
\begin{equation*}
    \xi(n + m) \in [a_{k}, b_{k}) 
\end{equation*}
which can (by definition) be executed efficiently as $\psi$ is assumed to be efficiently evaluable in polynomial time.

To see that the result remains true if we replace $2^{\mathcal{C}}$ by the set of bounded multisubsets of $\mathcal{C}$, we define 
\begin{equation*}
    \mathcal{M}(B) = \{ \{\!\{c_{1}, \dots, c_{k}\}\!\} \colon c_{1}, \dots, c_{k} \in \mathcal{C} \text{ and } k \le B\}   
\end{equation*}
as the set of multisubsets of $\mathcal{C}$ of size up to $B$. 
We emphasize the operator $s \mapsto \#s$ counts repeated entries as distinct in $s$. 
We assume $B$ is polynomial on $\#\mathcal{C}$. 
In this case, the above argument applies by replacing $2^{\mathcal{C}}$ with $\mathcal{M}(B)$ and noticing that it is enough to verify whether, given $s \in \mathcal{C}$, there is a $m \in \{0, \dots, B - \# (s \cap r)\}$ for which $\xi(\#(s \cap r) + m) \in [a_{k}, b_{k})$, which can again be done efficiently. 

\subsection{Proof of Proposition~\ref{prop:markovsa}} 

Let $\{S_{t}^{\smash[t]{(f)}}\}_{t \ge 1}$ be the Markov chain defined by $\kappa_{F}$, and let $L$ be the maximum trajectory length in the GFlowNet's state graph.
For conciseness, we will denote by $\mathbb{P}_{\kappa_{F}}[\cdot]$ the probability measure corresponding to the case in which the forward Markov chain starts $s_{o}$, i.e.,   $\mathbb{P}_{\kappa_{F}}[\cdot | s_{o}]$. 
It is clear that $T_{\mathcal{X}} \le L$ and, hence, $\mathbb{P}_{\kappa_{F}}[T_{\mathcal{X}} < \infty | s] = 1$ for any $s$. 
We now proceed to show that $S_{T_{\mathcal{X}}}^{\smash[t]{(f)}}$ is distributed according to the target $R(x)$. 
We start by noticing by 
\begin{equation*}
    \mathbb{P}_{\kappa_{F}}[ S_{T_{\mathcal{X}}}^{\smash[t]{(f)}} = x] = \sum_{1 \le t \le L} \mathbb{P}_{\kappa_{F}}[T_{\mathcal{X}} = t] \mathbb{P}_{\kappa_{F}}\left[ S_{T_{\mathcal{X}}}^{\smash[t]{(f)}}  = x | T_{\mathcal{X}} = t \right]. 
\end{equation*}
and, by the definition of conditional probabilities, 
\begin{equation*}
    \mathbb{P}_{\kappa_{F}}\left[ S_{T_{\mathcal{X}}}^{\smash[t]{(f)}} = x | T_{\mathcal{X}} = t \right] = \frac{ \sum_{(s_{1}, \dots, s_{t - 1}, x) \in \bar{S}^{t}} \kappa_{F}(s_{t - 1}, x) \cdot \prod_{1 \le i \le t - 2} \kappa_{F}(s_{i}, s_{i + 1}) } { \sum_{x' \in \mathcal{X}} \sum_{(s_{1}, \dots, s_{t - 1}, x') \in \bar{S}^{t}} \kappa_{F}(s_{t - 1}, x') \cdot \prod_{1 \le i \le t - 2} \kappa_{F}(s_{i}, s_{i + 1})}.  
\end{equation*}
This equation follows from the definition of $T_{\mathcal{X}}$ and the Markov property.
In addition, if there is a $(s_{i}, s_{i + 1})$ for which the transition $s_{i} \rightarrow s_{i + 1}$ is not in the underlying GFlowNet's state graph, the probability of $(s_{o}, \dots, s_{t - 1}, x)$ is zero by design (as $\kappa_{F}$ matches $p_{F}$); hence, the above ratio can be written as 
\begin{equation*}
    \frac{ \sum_{\tau \in s_{o}  \rightsquigarrow x, |\tau| = t} p_{F}(s_{o}, \tau)  } { \sum_{x' \in \mathcal{X}} \sum_{\tau' \in s_{o} \rightsquigarrow x', |\tau'| = t} p_{F}(s_{o}, \tau')},  
\end{equation*}
with $|\tau| = t$ stating that the length (number of states) of the trajectory $\tau$ is $t$. 
A similar reasoning shows that 
\begin{equation*}
    \begin{aligned} 
        \mathbb{P}_{\kappa_{F}}[T_{\mathcal{X}} = t] &= \sum_{x' \in \mathcal{X}} \sum_{\tau \in s_{o} \rightsquigarrow x', |\tau| = t} p_{F}(s_{o}, \tau). 
    \end{aligned} 
\end{equation*}
Thus, 
\begin{equation*}
    \mathbb{P}_{\kappa_{F}}[ S_{T_{\mathcal{X}}}^{\smash[t]{(f)}} = x] = \sum_{1 \le t \le L} \sum_{\tau \in s_{o} \rightsquigarrow x, |\tau| = t} p_{F}(s_{o}, \tau) = \sum_{\tau \in s_{o} \rightsquigarrow x} p_{F}(s_{o}, \tau) = m_{F}(x). 
\end{equation*}
When the GFlowNet abides a balance condition, this implies 
\begin{equation*}
    \mathbb{P}_{\kappa_{F}}\left[ S_{T_{\mathcal{X}}}^{\smash[t]{(f)}} = x \right] = \frac{R(x)}{\sum_{x \in \mathcal{X}} R(x)}, 
\end{equation*}
as stated by Proposition~\ref{prop:markovsa}.

Correspondingly, we observe that 
\begin{equation*}
    \mathbb{P}_{\kappa_{B}}\left[ T_{\{s_{o}\}} < \infty | x \right] = 1 
\end{equation*}
since each trajectory starting at $x$ must eventually reach $s_{o}$, and the maximum distance between $s_{o}$ and any $x \in \mathcal{X}$ is $L$. 
This shows the second part of Proposition~\ref{prop:markovsa}. 

\subsection{Proof of Lemma~\ref{lemma:harmonicsa}}

We show the function $h_{k}$ defined as 

\begin{equation*}
    h_{k}(s) = \mathbb{P}_{\kappa_{F}}\left[ T_{k} = \min_{k'} T_{k'} | s \right] 
\end{equation*}
for $s \in \mathcal{S}$ is harmonic. 
For this, we observe that 
\begin{equation*}
    \begin{aligned}
        \mathbb{E}_{s' \sim \kappa_{F}(s, \cdot)} \left[ h_{k}(s') \right] &= \sum_{s' \in \bar{\mathcal{S}}} \kappa_{F}(s, s') \mathbb{P}_{\kappa_{F}}\left[ T_{k} = \min_{k'} T_{k'} | s' \right] \\ 
        &= \sum_{s' \in \bar{\mathcal{S}}} \kappa_{F}(s, s') \mathbb{P}_{\kappa_{F}} \left[ T_{k} < \infty | s' \right] \\ 
        &= \mathbb{P}_{\kappa_{F}}[T_{k} < \infty | s] = \mathbb{P}_{\kappa_{F}}\left[ T_{k} = \min_{k'} T_{k'} | s \right] = h_{k}(s).   
    \end{aligned}
\end{equation*}
The first equation follows from definition. 
The second and fourth equations are a consequence of the fact that $\mathcal{X}_{k}$ is an absorbing set of the Markov chain dictated by $\kappa_{F}$ and hence that $\mathcal{X}_{k}$ being arrived at before $\mathcal{X}_{k'}$ for $k \neq k'$ implies that $T_{k}$ is the only finite hitting time among $\{T_{k}\}_{k \in [K]}$.
The third equation follows from the law of total probability, and the fifth equation follows from the definition of $h_{k}$. 
This demonstrates $h_{k}$ is an harmonic function for $k \in [K]$. 

The fact that $h_{k}(x) = 1$ for $x \in \mathcal{X}_{k}$ and $h_{k}(x) = 0$ for $x \in \mathcal{X} \setminus \mathcal{X}_{k}$ follows from the definition of $h_{k}$ as the probability of the Markov process reaching at $\mathcal{X}_{k}$ before any $\mathcal{X}_{k'}$ for $k' \neq k$. 

Similarly, $h_{k}(s_{o})$ is the probability that $\kappa_{F}$ arrives at $\mathcal{X}_{k}$ before arriving at any other $\mathcal{X}_{k'}$ for $k' \neq k$; as the underlying GFlowNet is assumed to abide by any of its balance conditions, the probability of the former is proportional to $Z_{k} = \sum_{x \in \mathcal{X}_{k}} R(x)$, and of the latter, proportional to $Z = \sum_{x \in \mathcal{X}} R(x)$. This shows 
\begin{equation*}
    h_{k}(s_{o}) = \frac{\sum_{x \in \mathcal{X}_{k}} R(x)}{\sum_{x \in \mathcal{X}} R(x)} = \frac{Z_{k}}{Z}. 
\end{equation*}
We also provide a direct derivation of this fact in \Cref{sec:doobsaa}.

\subsection{Proof of Proposition~\ref{prop:doobsaa}} \label{sec:doobsaa} 

We show that the marginal distribution of the transformed kernel $\kappa_{F}^{\smash[t]{(k)}}$, as defined in \Cref{prop:markovsa}, matches the restriction of $R(x)$ to $\mathcal{X}_{k}$ up to a normalizing constant.

As before, we denote by $\mathbb{P}_{\kappa_{F}}[\cdot]$ the probability measure $\mathbb{P}_{\kappa_{F}}[\cdot | s_{o}]$ corresponding to the scenario in which the forward Markov chain starts at $s_{o}$.  
We first notice that, since $h_{k}$ is harmonic, 
\begin{equation*}
    \sum_{s' \in \bar{\mathcal{S}}} \kappa_{F}^{\smash[t]{(k)}}(s, s') = \sum_{s' \in \bar{\mathcal{S}}} \kappa_{F}(s, s') \cdot \frac{h_{k}(s')}{h_{k}(s)} = \frac{h_{k}(s)}{h_{k}(s)} = 1.  
\end{equation*}

We then demonstrate that, under $\kappa_{F}^{\smash[t]{(k)}}$, the probability of any path that arrives at $\mathcal{X}_{m}$ before $\mathcal{X}_{k}$ for a $m \neq k$ is zero. 
To see this, let $(s_{o}, s_{1} \dots, s_{t - 1}, x)$ for $x \notin \mathcal{X}_{k}$; then, $h_{k}(x) = 0$ and there is a smallest index $i \ge 1$ for which $h_{k}(s_{i}) = 0$.
(By the definition of $\kappa_{F}$, $h_{k}(s_{o}) > 0$ for each $k \in [K]$). 
The transition from $s_{i - 1}$ to $s_{i}$ under $\kappa_{F}^{\smash[t]{(k)}}$ thus has probability zero. 
This shows that, under $\kappa_{F}^{\smash[t]{(k)}}$, the Markov chain arrives at $\mathcal{X}_{k}$ before any $\mathcal{X}_{m}$ for $m \neq k$ with probability one.
In particular, dropping the superscript $(k)$ from $T_{m}$ for clarity, 
\begin{equation*}
    \mathbb{P}_{\kappa_{F}^{\smash[t]{(k)}}} \left[ T_{k} < \infty = T_{m} \right] = 1.  
\end{equation*}
This is Proposition~\ref{prop:doobsaa}'s first statement. 

To show that 
\begin{equation*}
    \mathbb{P}_{\kappa_{F}^{\smash[t]{(k)}}} \left[ S_{T_{k}}^{\smash[t]{(k)}} = x \right] = \frac{R(x)}{\sum_{x \in \mathcal{X}_{k}} R(x)},   
\end{equation*}
we proceed as in the demonstration of Proposition~\ref{prop:markovsa}. 
Since 
\begin{equation*}
    \begin{aligned}
        \mathbb{P}_{\kappa_{F}^{\smash[t]{(k)}}} \left[ S_{T_{k}}^{\smash[t]{(k)}} = x \right] &= \sum_{1 \le t \le L} \mathbb{P}_{\kappa_{F}^{\smash[t]{(k)}}}[T_{k} = t] \mathbb{P}_{\kappa_{F}^{\smash[t]{(k)}}}\left[ S_{T_{k}}^{\smash[t]{(k)}} = x | T_{k} = t \right] \\ 
        &= \sum_{1 \le t \le L} \mathbb{P}_{\kappa_{F}^{\smash[t]{(k)}}}\left[ T_{k} = t \right] \frac{\sum_{\tau \in s_{o} \rightsquigarrow x, |\tau| = t} \kappa_{F}^{\smash[t]{(k)}}(s_{o}, \tau)}{\sum_{x' \in \mathcal{X}_{k}} \sum_{\tau' \in s_{o} \rightsquigarrow x', |\tau'| = t}  \kappa_{F}^{\smash[t]{(k)}}(s_{o}, \tau')}, 
    \end{aligned}
\end{equation*}
in which we define $\kappa_{F}^{\smash[t]{(k)}}(s_{o}, \tau) = \prod_{(s, s') \in \tau} \kappa_{F}^{\smash[t]{(k)}}(s, s')$, and $h_{k}(x) = 1$ and $h_{k}(s_{o}) > 0$ for each $k \in [K]$ and $x \in \mathcal{X}_{k}$, a telescoping product shows that 
\begin{equation*}
    \prod_{(s, s') \in \tau} \kappa_{F}^{\smash[t]{(k)}}(s, s') = \frac{1}{h_{k}(s_{o})} \prod_{(s, s') \in \tau} \kappa_{F}(s, s') = \frac{1}{h_{k}(s_{o})} \cdot \kappa_{F}(s_{o}, \tau) = \frac{1}{h_{k}(s_{o})} \cdot p_{F}(s_{o}, \tau)  
\end{equation*}
for any $\tau \in s_{o} \rightsquigarrow x$ with $x \in \mathcal{X}_{k}$. 
As a consequence, 
\begin{equation*}
    \frac{\sum_{\tau \in s_{o} \rightsquigarrow x, |\tau| = t} \kappa_{F}^{\smash[t]{(k)}}(s_{o}, \tau)}{\sum_{x' \in \mathcal{X}_{k}} \sum_{\tau' \in s_{o} \rightsquigarrow x', |\tau'| = t}  \kappa_{F}^{\smash[t]{(k)}}(s_{o}, \tau')} = \frac{\sum_{\tau \in s_{o} \rightsquigarrow x, |\tau| = t} p_{F}(s_{o}, \tau)}{\sum_{x' \in \mathcal{X}_{k}} \sum_{\tau' \in s_{o} \rightsquigarrow x', |\tau'| = t} p_{F}(s_{o}, \tau')}. 
\end{equation*}
The rest proceeds as in Proposition~\ref{prop:markovsa}.
In fact, as 
\begin{equation*}
    \mathbb{P}_{\kappa_{F}^{\smash[t]{(k)}}} \left[ T_{k} = t \right] = \sum_{x \in \mathcal{X}_{k}} \sum_{\tau \in s_{o} \rightsquigarrow x} \kappa_{F}^{\smash[t]{(k)}}(s_{o}, \tau) = \frac{1}{h_{k}(s_{o})} \sum_{x \in \mathcal{X}_{k}} \sum_{\tau \in s_{o} \rightsquigarrow x} p_{F}(s_{o}, \tau) 
\end{equation*}
and 
\begin{equation*}
    h_{k}(s_{o}) = \mathbb{P}_{\kappa_{F}}\left[ T_{k} < \infty  | s_{o} \right] = \sum_{x \in \mathcal{X}_{x}} \sum_{\tau \in s_{o} \rightsquigarrow x} p_{F}(s_{o}, \tau),  
\end{equation*}
we deduce that 
\begin{equation*}
    \mathbb{P}_{\kappa_{F}^{\smash[t]{(k)}}} \left[ S_{T_{k}}^{\smash[t]{(k)}} = x \right] = \frac{\sum_{1 \le t \le L} \sum_{\tau \in s_{o} \rightsquigarrow x, |\tau| = t} p_{F}(s_{o}, \tau)}{\sum_{x \in \mathcal{X}_{x}} \sum_{\tau \in s_{o} \rightsquigarrow x} p_{F}(s_{o}, \tau)}. 
\end{equation*}
When the GFlowNet abides by a balance condition, this expression can be rewritten as
\begin{equation*}
    \mathbb{P}_{\kappa_{F}^{\smash[t]{(k)}}} \left[ S_{T_{k}}^{\smash[t]{(k)}} = x \right] = \frac{R(x)}{\sum_{x \in \mathcal{X}_{k}} R(x)}.  
\end{equation*}
This demonstrates the second statement in Proposition~\ref{prop:doobsaa}. 

\subsection{Proof of Proposition~\ref{prop:subtbsa}}

We show that, if our mixture model satisfies the defined SubTB condition, then it samples from the target distribution.
We then show write $F(s, \mathrm{d}\gamma)$ in terms of the well-known quantities $Z(\mathrm{d}\gamma)$ and $R(x, \mathrm{d}\gamma)$. 

The first statement can be derived directly from the fact that, since the condition is satisfied for all trajectories $\tau$ between reachable pairs $s, s'$, it is in particular satisfied by \emph{complete} trajectories going from $s_{o}$ to $x \in \mathcal{X}$. 
This implies the trajectory balance is satisfied and, by \citep[Theorem 1]{malkin2022trajectory}, the marginal distribution of $p_{M}(\mathrm{d}\gamma) p_{F}^{(\gamma)}$ over $x$ matches $R(x, \mathrm{d}\gamma)$. 
By integrating $\gamma$ out, we conclude that \Cref{eq:marginalss} is satisfied. 

To see that $Z(\mathrm{d}\gamma) \coloneqq \sum_{x \in \mathcal{X}} R(x, \mathrm{d}\gamma)$, we notice that 
\begin{equation*}
    F(s_{o}, \mathrm{d}\gamma) p_{F}^{(\gamma)}(s_{o}, \tau) = R(x, \mathrm{d}\gamma) p_{B}^{(\gamma)}(x, \tau) 
\end{equation*}
for all $x \in \mathcal{X}$ and all trajectories $\tau \in s_{o} \rightsquigarrow x$. 
In particular, marginalizing out both $\tau$ and $x$ (in sequence), we infer (as in Proposition~\ref{prop:relsa}) that
\begin{equation*}
    F(s_{o}, \mathrm{d}\gamma) = \sum_{x \in \mathcal{X}} R(x, \mathrm{d}\gamma) \coloneqq Z(\mathrm{d}\gamma). 
\end{equation*}
In a similar way, as $F(s, \mathrm{d}\gamma) p_{F}^{(\gamma)}(s, \tau) = R(x, \mathrm{d}\gamma) p_{B}^{(\gamma)}(x, \tau)$, we directly infer that 
\begin{equation*}
    F(s, \mathrm{d}\gamma) = \sum_{x \in \mathcal{X}} \sum_{\tau \in s \rightsquigarrow x} R(x, \mathrm{d}\gamma) p_{B}^{(\gamma)}(x, \tau) 
\end{equation*}
with the convention that $\sum_{\tau \in \emptyset} \eta_{\tau} = 0$ for any trajectory-indexed real sequence $\{\eta_{\tau}\}_{\tau}$. 
This is exactly Proposition~\ref{prop:subtbsa}'s second statement. 

\subsection{Proof of Corollary~\ref{col:cmarginalsa}} 

We derive the relationship between the prior and posterior mixing distributions when they both satisfy the SubTB condition.
It is clear that, since $p_{M}$, $q_{M}$, $p_{F}^{\smash[t]{(\gamma)}}$, and $p_{B}^{(\gamma)}$ abide by SubTB, they also abide by TB. 
In particular, 
\begin{equation*}
    p_{M}(\gamma) p_{F}^{(\gamma)}(s_{o}, \tau) = \frac{R(x)}{Z} q_{M}(\gamma | x) p_{B}^{(\gamma)}(x, \tau).  
\end{equation*}
As in Proposition~\ref{prop:relsa}, we marginalize over $\tau \in s_{o} \rightsquigarrow x$ to infer that 
\begin{equation*}
    p_{M}(\gamma) \cdot m_{F}^{(\gamma)}(x) = \frac{R(x)}{Z} \cdot q(\gamma | x);  
\end{equation*}
we then derive the second equation from Corollary~\ref{col:cmarginalsa} by isolating $q(\gamma | x)$. 
Similarly, the first equation can be obtained by marginalizing over $x \in \mathcal{X}$ and recalling $m_{F}^{(\gamma)}(x)$ is a probability distribution over $\mathcal{X}$, i.e., 
\begin{equation*}
    p_{M}(\gamma) = p_{M}(\gamma) \sum_{x \in \mathcal{X}} m_{F}^{(\gamma)}(x) = \sum_{x \in \mathcal{X}} \frac{R(x)}{Z} q_{M}(\gamma|x). 
\end{equation*}

\subsection{Proof of Proposition~\ref{prop:universalsa}}

Our proof has two steps. 
We first show that, if a model achieves an error of at most $\epsilon$ in terms of $\mathrm{TV}$ distance to the true policy (which is unique since we are fixing the backward policy for each $\gamma$), then the $\mathrm{TV}$ between the marginal distribution over $\mathcal{X}$ is bounded as a monotonically decreasing and continuous function of $\epsilon$.
Then, using \citet[Theorem 1]{abboud2021surprising}, which characterizes the universality of GNNs with random features, and the intermediate value theorem for the prior bound, we infer that our sampler can approximate any distribution with high-probability.
\looseness=-1 

To start with, we temporarily simplify our notation of $p_{F}^{(\gamma)}$ as $p$, ommiting the dependence on $\gamma$.
We also refer to the true solution to the chosen balance equations as $p^{\star}$. 
Our GNN approximates the distribution $p(s, \cdot)$; however, we are interested in the difference 
\begin{equation*}
    \left| \sum_{\tau \in s_{o} \rightsquigarrow x}  \prod_{(s, s') \in \tau} p(s, s') - \prod_{(s, s') \in \tau} p^{\star}(s, s') \right|   
\end{equation*}
for each $x \in \mathcal{X}$. 
We start with the following lemma. 
\begin{lemma}
    Let $a_{1}, \dots, a_{n}$ and $b_{1}, \dots, b_{n}$ be sequences in $[0, 1]$ such that $| a_{i} - b_{i}| \le \delta$ for some $\epsilon > 0$.
    Then, 
    \begin{equation*}
        \left| \prod_{1 \le i \le n} a_{i} - \prod_{1 \le i \le n} b_{i} \right| \le n \epsilon.   
    \end{equation*}
\end{lemma}
\begin{proof}
    We proceed by induction. 
    This clearly holds for $n = 1$. 
    To build intuition, we also show for $N = 2$. 
    Notice that 
    \begin{equation*}
        |a_{1}a_{2} - b_{1}b_{2}| = |(a_{1} - b_{1}) a_{2} + b_{1}( a_{2} - b_{2})| \le a_{2} |a_{1} - b_{1}| + b_{1} |a_{2} - b_{2}| \le 2\epsilon   
    \end{equation*}
    since $a_{1}, a_{2}, b_{1}, b_{2} \in [0, 1]$. 
    Assume the identity holds for $n$. 
    Then, 
    \begin{equation*}
        \begin{aligned}
            \left| \prod_{1 \le i \le n + 1} a_{i} - \prod_{1 \le i \le n + 1} b_{i}  \right| &= \left| (a_{n + 1} - b_{n + 1}) \prod_{1 \le i \le n} a_{i} + b_{n + 1} \cdot \left( \prod_{1 \le i \le n} a_{i} - \prod_{1 \le i \le n} b_{i} \right) \right| \\ 
            &= \epsilon + n \epsilon = (n + 1) \epsilon. 
        \end{aligned}
    \end{equation*}
    by the induction hypothesis. 
    This shows our desired bound. 
\end{proof}

As there is a realization of our GNN for which $\mathrm{TV}\left( p(s, \cdot), p^{\star}(s, \cdot) \right) \le \epsilon$ for any $s \in \mathcal{S}$, we infer that $|p(s, s') - p^{\star}(s, s')| \le 2 \epsilon$. 
Under this condition, we let $T$ be the maximum number of trajectories from $s_{o}$ to $x$ for $x \in \mathcal{X}$ and $L$ to be the largest trajectory length from $s_{o}$ to $x \in \mathcal{X}$. 
By \citet[Theorem 1]{abboud2021surprising}, there are $h \in \mathcal{H}_{t}$ and $\xi \in \mathcal{H}_{s}$ for which $p(s, \cdot)$ approximates $p^{\star}(s, \cdot)$ with error at most $\epsilon' = \frac{\epsilon}{2|\mathcal{X}|TL}$ and with probability $1 - \frac{\delta}{|\mathcal{X}|}$ over the realization of $\gamma$. 
The above difference in marginals can then be bounded as, with probability $1 - \frac{\delta}{|\mathcal{X}|}$ over the realization of $\gamma$,   
\begin{equation*}
    \begin{aligned}
        &\left|  \left( \sum_{\tau \in s_{o} \rightsquigarrow x} \prod_{(s, s') \in \tau} p(s, s') - \prod_{(s, s') \in \tau} p^{\star}(s, s') \right) \right|
        \\ 
        &\le \left( \sum_{\tau \in s_{o} \rightsquigarrow x} \left| \prod_{(s, s') \in \tau} p(s, s') - \prod_{(s, s') \in \tau} p^{\star}(s, s') \right| \right) \\ 
        &\le \left(1 - \frac{\delta}{T}\right) (2TL) \frac{\epsilon}{2|\mathcal{X}|TL} = \frac{\epsilon}{|\mathcal{X}|}. 
    \end{aligned}
\end{equation*}
In other words, denoting by $m_{F}^{(\gamma)}$ and by $\pi(x)$ the estimated and true marginals for the learned and target models, 
\begin{equation*}
    |m_{F}^{(\gamma)}(x) - \pi(x)| \le \frac{\epsilon}{|\mathcal{X}|} 
\end{equation*}
with probability $1 - \frac{\delta}{|\mathcal{X}|}$ over $\gamma$. 
Through an union bound\footnote{We briefly recall that, if $P_{1}, \dots, P_{n}$ are events with probability at least $1 - \delta$, then $\cap_{i} P_{i}$ is an event with probability at least 
\begin{equation*}
    \mathbb{P}[ \cap_{i} P_{i} ] = 1 - \mathbb{P}[\cup_{i} P_{i}^{c} ] \ge 1 - \sum_{i} \mathbb{P}[P_{i}^{c}] \ge 1 - n \delta.   
\end{equation*}
} for $x \in \mathcal{X}$, we infer that 
\begin{equation*}
    \mathrm{TV}\left( m_{F}^{(\gamma)}, \pi \right) = \frac{1}{2} \sum_{x \in \mathcal{X}} \left| m_{F}^{(\gamma)}(x) - \pi(x) \right| \le \frac{\epsilon}{2}  
\end{equation*}
with probability $1 - \delta$ over the sampling of $\gamma$. 
This demonstrates that RF GFlowNets are universal approximators of distributions over graphs when using the kind of edge-additive process we described in Proposition~\ref{prop:universalsa}. 


\subsection{Proof of Lemma~\ref{lemma:localsa}}

We first expand $\mathcal{L}_{\mathrm{TB}}$ around $\gamma = 0$ and then take the expectation of the corresponding expansion over $\gamma \sim p_{M}$.
To start with, let $L(\gamma) \coloneqq \mathcal{L}(\tau, \gamma ; \gfn_{\Gamma})$; in this case, 
\begin{equation*}
    L(\gamma) = L(0) + \gamma^{T} \nabla_{\gamma = 0} L(\gamma) + \frac{1}{2} \gamma^{T} H_{\gamma}(\tau, 0) \gamma + \mathcal{O}(\sigma^{3}). 
\end{equation*}
Since, for any matrix $A$, 
\begin{equation*}
    \mathbb{E}_{\gamma \sim p_{M}}  \left[ \gamma \right] = 0 \text{ and } \mathbb{E}_{\gamma \sim p_{M}} \left[ \gamma^{T} A \gamma \right] = \sigma^{2} \cdot \mathrm{tr} (A)    
\end{equation*}
as $\gamma^{T} A \gamma = \mathrm{tr}(\gamma^{T} A\gamma) = \mathrm{tr}(A \gamma\gamma^{T})$ for any matrix $A$ and $\mathrm{tr}$ is a linear operator, and in general any odd moment of $\gamma$ has nulls out, we infer that  
\begin{equation*}
    \mathbb{E}_{\gamma \sim p_{M}} \left[ L(\gamma) \right] = L(0) + \frac{1}{2} \sigma^{2} \cdot \mathrm{tr} \left( H_{\gamma}(\tau, 0)  \right) + \mathcal{O}(\sigma^{4}),  
\end{equation*}
which is the statement of Lemma~\ref{lemma:localsa} after replacing $L$ with $\mathcal{L}_{\mathrm{TB}}$. 
This characterizes the local behavior of the (population) loss in the learning objective of CI GFlowNets. 

\subsection{Proof of Proposition~\ref{prop:regularizationsa}}

We derive from Lemma~\ref{lemma:localsa} that 
\begin{equation*}
    \mathcal{L}_{\mathrm{TB}}(\gfn_{\Gamma}) = \mathcal{L}_{\mathrm{TB}}(\gfn_{o}) + \frac{1}{2} \cdot \sigma^{2} \cdot \mathbb{E}_{\tau \sim \rho} \left[ \mathrm{tr} \left( H_{\gamma}(\tau, 0) \right) \right] + \mathcal{O}(\sigma^{4}) = \mathcal{L}_{\mathrm{TB}}(\gfn_{o}) + \frac{1}{2} \cdot \sigma^{2} \cdot \mathcal{L}_{\mathrm{reg}}(\gfn_{\Gamma}) + \mathcal{O}(\sigma^{4}).  
\end{equation*}
As taking the derivative is a linear operation, $f \in \mathcal{O}(\sigma^{4})$ implies $\nabla_{\theta}^{2} f \in \mathcal{O}(\sigma^{4})$ for any twice differentiable function $f \colon \Theta \rightarrow \mathbb{R}$ on a differentiable manifold $\Theta$. 
In particular, we infer that, by evaluating the Hessian from both members of the above equation with respect to the parameters $\theta$ of the underlying neural network,   
\begin{equation*}
    H_{\mathrm{TB}} = H_{o} + \frac{\sigma^{2}}{2} \cdot H_{\mathrm{reg}} + \mathcal{O}(\sigma^{4}), 
\end{equation*}
with $H_{\mathrm{TB}} = \nabla_{\theta}^{2} \mathcal{L}_{\mathrm{TB}}(\gfn_{\Gamma})$, $H_{o} = \nabla_{\theta}^{2} \mathcal{L}_{o}(\gfn_{\Gamma})$, and $H_{\mathrm{reg}} = \nabla_{\theta}^{2} \mathcal{L}_{\mathrm{reg}}(\gfn_{\Gamma})$. 

This in particular shows that the random features act as a form of \emph{spectral regularization}, shifting the eigenvalues of $\mathrm{H}_{o}$ by $\frac{\sigma^{2}}{2} \cdot \mathrm{eig}(H_{\mathrm{reg}})$, with $\mathrm{eig}(H_{\mathrm{reg}})$ being the vector of eigenvalues of $H_{\mathrm{reg}}$. 

\subsection{Proof of Lemma~\ref{lemma:harmonicsaa}}

We show both functions $h_{k}$ and $g_{k}$ are harmonic on $\bar{\mathcal{S}} \setminus \mathcal{S}_{k}$. 
Let $s \in \bar{\mathcal{S}} \setminus \mathcal{S}_{k}$. 
Then, 
\begin{equation*}
    \begin{aligned} 
        \mathbb{E}_{s' \sim \kappa_{F}(s, \cdot)} \left[ h_{k}(s') \right] &= \sum_{s' \in \bar{\mathcal{S}}} \kappa_{F}(s, s') \cdot \mathbb{P}_{\kappa_{F}}\left[T_{k}^{(f)} = \min_{k'} T_{k'}^{(f)} < \infty | s'\right] \\ 
        &= \sum_{s' \in \bar{\mathcal{S}}} \kappa_{F}(s, s') \mathbb{P}_{\kappa_{F}}\left[T_{k}^{(f)} < \infty | s'\right] \\ 
        &= \mathbb{P}_{\kappa_{F}} \left[ T_{k}^{(f)} < \infty | s \right] = h_{k}(s). 
    \end{aligned} 
\end{equation*}
The first equation follows from the definition of $h_{k}$; the second and third equations follow from the fact that the distance from $s \in \mathcal{S}_{k}$ to $s_{o}$ in the FHC in Definition~\ref{def:fhcsa} is constant for $k \in [K]$ and, hence, if the Markov chains goes through $\mathcal{S}_{k}$ it will not go through any $\mathcal{S}_{k'}$ for $k' \neq k$, .
A similar argument can be carried out for $g_{k}(s)$ and $T_{k}^{(b)}$. 

\subsection{Proof of Lemma~\ref{lemma:harmonicsaaa}}

We show that $h_{k}(s_{o})$ and $g_{k}(x)$ can be interpreted as probability distributions over $k \in \{0, \dots, K\}$. 
The intuition is that the Markov chain has to either hit one of $\mathcal{S}_{k}$ when starting at $s_{o}$, or end up in $\mathcal{X}_{o}$, in which case none of $\mathcal{S}_{k}$ are hit for $k \in [K]$; similarly, if $x \in \mathcal{X}_{o}$, the Markov chain following $\kappa_{B}$ and starting at $x$ is absorbed by $s_{o}$ without going through any of the $\mathcal{S}_{k}$ for $k \in [K]$.

We start with $h_{k}(s_{o})$. 
By Definition~\ref{def:fhcsa}, the events 
\begin{equation*}
    \left\{ T_{1}^{(f)} < \infty, S_{1}^{(f)} = s_{o} \right\}, \dots, \{ T_{K}^{(f)} < \infty, S_{1}^{(f)} = s_{1} \}, \{ T_{k}^{(f)} = \infty \text{ for } k \in [K], S_{1}^{(f)} = s_{o} \} 
\end{equation*}
are mutually joint. 
In addition, $T_{k}^{(f)} < \infty$ implies $T_{k'}^{(f)} = \infty$ almost surely for $k' \neq k$. 
As a consequence, these events also cover the space of possibilities.
Hence, by the law of total probability, 
\begin{equation*}
    \left( \sum_{1 \le k \le K} \mathbb{P}_{\kappa_{F}} \left[ T_{k}^{(f)} < \infty | S_{1}^{(f)} = s_{o} \right]\right) + \mathbb{P}_{\kappa_{F}}\left[ T_{k}^{(f)} = \infty \text{ for } k \in [K] | S_{1}^{(f)} = s_{o} \right] = 1. 
\end{equation*}
This sum is exactly $h_{o}(s_{o}) + \sum_{k \in [K]} h_{k}(s_{o})$, which demonstrates the first statement of Lemma~\ref{lemma:harmonicsaaa}. 

We proceed similarly for $g_{k}$. 
As noticed, if $x \in \mathcal{X}_{o}$, then $T_{k}^{(b)} = \infty$ for $k \in [K]$. 
Otherwise, as $s_{o}$ is the only absorbing state of a Markov chain dictated by $\kappa_{B}$ in $\bar{\mathcal{S}}$ has to eventually go through some of the $\mathcal{S}_{k}$ when starting at $x$. 
This means that the events 
\begin{equation*}
    \{T_{1}^{(b)} < \infty\}, \dots, \{T_{K}^{(b)} < \infty\}, \{T_{k} = \infty \text{ for } k \in [K]\} 
\end{equation*}
are mutually disjoint when the Markov chain starts at $x \in \mathcal{X}$ and cover the space of possibilities. 
As with $h_{k}$, this implies 
\begin{equation*}
    g_{o}(x) + \sum_{1 \le k \le K} g_{k}(x) = 1, 
\end{equation*}
which is the second statement in Lemma~\ref{lemma:harmonicsaaa}. 

\subsection{Proof of Proposition~\ref{prop:salsa}} 

We show that, if the policies $p_{F}^{(k)}$ and $p_{B}^{(k)}$ satisfy the ATB condition from \citet{silva2025generalization}, then the policies $\tilde{p}_{F}^{(k)}$ and $\tilde{p}_{B}^{(k)}$ satisfy the MTB condition in Definition~\ref{def:mtbsa}. 

We first notice that, for any trajectory $\tau$ from $s_{o}$ to $x \in \mathcal{X}_{o}$,
\begin{equation*}
    \tilde{p}_{F}^{(o)}(s_{o}, \tau) = \frac{h_{o}(x)}{h_{o}(s_{o})} p_{F}^{(o)}(s_{o}, \tau) = \frac{1}{h_{o}(s_{o})} \cdot p_{F}^{(o)}(s_{o}, \tau), 
\end{equation*}
a fact that can be derived from a telescoping product and $h_{o}(x) = 1$ since for a Markov chain starting at $x \in \mathcal{X}_{o}$ and following $\kappa_{F}$ it is true that $T_{k}^{(f)} = \infty$ (as none of $\mathcal{S}_{k}$ is reachable from $x \in \mathcal{X}_{o}$, which is an absorbing state). 
By the ATB, 
\begin{equation*}
    p_{F}^{(o)}(s_{o}, \tau) = \frac{R(x)}{Z} p_{B}^{(o)}(x, \tau) = \frac{R(x)}{Z} \tilde{p}_{B}^{(o)}(x, \tau), 
\end{equation*}
as $p_{B}^{(o)}(s, s') = \tilde{p}_{B}^{(o)}(s, s')$ for states $s, s'$ within a distance of $d$ of $s_{o}$; since $p_{B}$ does not increase the distance to $s_{o}$, and $x \in \mathcal{X}_{o}$ is by definition within a distance $d$ of $s_{o}$, it holds that $p_{F}^{(o)}(x, \tau) = \tilde{p}_{F}^{(o)}(x, \tau)$ for any $\tau \in s_{o} \rightsquigarrow x \in \mathcal{X}_{o}$ . 
Similarly, for any trajectory going from $s_{o}$ to $x \in \mathcal{X}_{k}$ through some $s \in \mathcal{S}_{k}$ (if the trajectory goes through $s \in \mathcal{S}_{k}$, it must eventually end up in $\mathcal{X}_{k}$ by definition), 
\begin{equation*}
    \tilde{p}_{F}^{(k)}(s_{o}, \tau) = p_{F}^{(o)}(s_{o}, \tau_{1}) \cdot \frac{h_{k}(s)}{h_{k}(s_{o})} \cdot p_{F}^{(k)}(s, \tau_{2}) = \frac{1}{h_{k}(s_{o})} \cdot p_{F}^{(o)}(s_{o}, \tau_{1}) \cdot p_{F}^{(k)}(s, \tau_{2}),
\end{equation*}
with $\tau_{1}$ and $\tau_{2}$ being the portions of $\tau$ before and after $s \in \mathcal{S}_{k}$, respectively, and $h_{k}(s) = 1$ since $s \in \mathcal{S}_{k}$, and the Markov chain can only visit one of $\mathcal{S}_{k}$ for $k \in [K]$ since there is no positive-probability trajectory between them. 
Also, by the ATB, 
\begin{equation*}
    p_{F}^{(o)}(s_{o}, \tau_{1}) p_{F}^{(k)}(s, \tau_{2}) = \frac{R(x)}{Z} p_{B}^{(k)}(s_{o}, \tau_{2}) \cdot p_{B}^{(o)}(s_{o}, \tau_{1}).  
\end{equation*}
As before, for $\tau_{1}$, it holds that $p_{B}^{(o)}(x, \tau_{1}) = \tilde{p}_{B}^{(o)}(x, \tau_{1})$. 
For $\tau_{2}$, on the other hand, a telescoping product shows that   
\begin{equation*}
    p_{B}^{(k)}(x, \tau) = \tilde{p}_{B}^{(k)}(x, \tau) \cdot \frac{g_{k}(x)}{g_{k}(s)} = g_{k}(x) \cdot \tilde{p}_{B}^{(k)}(x, \tau). 
\end{equation*}
In particular, 
\begin{equation*}
    \tilde{p}_{F}(s_{o}, \tau) = \frac{R(x)}{Z} \cdot \frac{g_{k}(x)}{h_{k}(s_{o})} \cdot \tilde{p}_{B}^{(k)}(x, \tau), 
\end{equation*}
i.e., defining $Z_{k} = Z \cdot h_{k}(s_{o})$ and $q(k|x) = g_{k}(x)$, which is a valid probability distribution by Lemma~\ref{lemma:harmonicsaaa}, we infer that 
\begin{equation*}
    Z_{k} \tilde{p}_{F}^{(k)}(s_{o}, \tau) = R(x) q(k|x) \tilde{p}_{B}(x, \tau) 
\end{equation*}
for $k \in [K]$, $k = 0$, and every $\tau$ from $s_{o}$ to $x \in \mathcal{X}_{k}$ through $\mathcal{S}_{k}$. 
This shows the solution to SAL is also a solution to the MTB condition in Definition~\ref{def:mtbsa}. 

\subsection{Proof of Proposition~\ref{prop:bgsa}} 

We show that the solution $\{\gfn_{k}\}_{k \ge 1}$ for the Boosted GFlowNets learning problem abides by the MTB under Assumptions~\ref{assumption:ua} and~\ref{assumption:crsa}. 

We first demonstrate that $q(k|x)$ is a probability distribution over $k \ge 1$. 
To see this, we notice that 
\begin{equation*}
    \sum_{1 \le k \le K} q(k|x) = \sum_{1 \le k \le K} \frac{R_{k}(x)}{R(x)} = \sum_{1 \le k \le K} \frac{R^{(k)}(x) - R^{(k - 1)}(x)}{R(x)} = \frac{R^{(K)}(x)}{R(x)},  
\end{equation*}
as $R^{(o)}(x) = 0$ by design.
By Assumption~\ref{assumption:ua}, $q(k|x) \ge 0$; by Assumption~\ref{assumption:ua}, the limit of the above equation as $K \rightarrow \infty$ is 
\begin{equation*}
    \lim_{K \rightarrow \infty} \frac{R^{(K)}}{R(x)} = \frac{R(x)}{R(x)} = 1. 
\end{equation*}
This shows $q(k|x)$ corresponds to a probability distribution over $k \ge 1$. 

We now prove that the learned policies abide by the MTB condition in Definition~\ref{def:mtbsa} having $q(k|x)$ as the posterior mixing distribution. 
By Assumption~\ref{assumption:crsa}, 
\begin{equation*}
    R^{(k)}(x) = R^{(k)}(x, \tau) \coloneqq R^{(k - 1)}(x) + \frac{Z_{k} \cdot p_{F}^{(k)}(s_{o}, \tau)}{p_{B}^{(k)}(x, \tau)}, 
\end{equation*}
i.e., 
\begin{equation*}
    p_{B}^{(k)}(x, \tau) \cdot \left( R^{(k)}(x) - R^{(k - 1)}(x) \right) = Z_{k} \cdot p_{F}^{(k)}(s_{o}, \tau).  
\end{equation*}
Since $R^{(k)}(x) - R^{(k - 1)}(x) = R_{k}(x) = R(x) \cdot q(k|x)$ by definition, 
\begin{equation*}
    Z_{k} \cdot p_{F}^{(k)}(s_{o}, \tau) = R(x) q(k|x) \cdot p_{B}^{(k)}(x, \tau).   
\end{equation*}
This is Proposition~\ref{prop:bgsa}'s second statement.
It should also be clear that, since 
\begin{equation*}
    Z_{k} = \sum_{x \in \mathcal{X}} R_{k}(x) = \sum_{x \in \mathcal{X}} R(x) q(k|x),  
\end{equation*}
our prior verification shows that 
\begin{equation*}
    \sum_{k \ge 1} Z_{k} = \sum_{k \ge 1} \sum_{x \in \mathcal{X}} R(x) q(k|x) = \sum_{x \in \mathcal{X}} R(x) \sum_{k \ge 1} q(k|x) = \sum_{x \in \mathcal{X}} R(x) = Z. 
\end{equation*}
We can interchange the summations due to the fact that $\sum_{k \ge 1} q(k|x) = 1$ and $\mathcal{X}$ is finite. 
These derivations jointly highlight Boosted GFlowNets can be interpreted as DI GFlowNets having countably infinite components.

\section{Related works} \label{sec:relatedsa}  

Generative Flow Networks have been introduced in \citep{bengio2021, foundations} as generative models for discrete, compositional spaces such as graphs and text.
They have since been applied to diverse problems, including phylogenetic inference \citep{zhou2024phylogfn}, robust scheduling \citep{robust}, Bayesian structure learning \citep{deleu2022bayesian, deleu2023joint}, LLM fine-tuning \citep{hu2023amortizing, venkatraman2024amortizingintractableinferencediffusion}, symbolic regression \citep{li2023gfnsrsymbolicregressiongenerative}, quantitative finance \citep{chen2026alphasagestructureawarealphamining}, and drug discovery \citep{bengio2021, laajil2025curriculumaugmentedgflownetsmrnasequence, pandey2024gflownet}, and have been popularized by publicly released software libraries \citep{tiapkin2025gfnxfastscalablelibrary, viviano2026torchgfnpytorchgflownetlibrary}. 
From a theoretical perspective, their connection to variational inference \citep{malkin2023gflownets} and reinforcement learning \citep{tiapkin2024generative} and Markov chains \citep{deleu2023generativeflownetworksmarkov} has been well-established, and variants operating on cyclic graphs \citep{Brunswic_2024, morozov2025revisitingnonacyclicgflownetsdiscrete} have also been recently considered. 
As with SC GFlowNets, recent works have proposed modifying the state space of the GFlowNet for improved learning, e.g., \citep{boussif2024actionabstractionsamortizedsampling, yu2026partialgflownetacceleratingconvergence, silva2025generalization}, however, their focus lies either on improving the exploratory policy, an issue which is orthogonal to our investigations, or on developing strictly distributed algorithms requiring the training of independent models.
Broadly, mixture models have been widely studied in variational and approximate inference, for instance, in semi-implicit variational inference methods \citep{yin2018semi}, path-sampling for Monte Carlo methods \citep{Gelman1998SimulatingNC, liu2001monte}, and in continuously indexed normalizing flows \citep{cornish2021relaxingbijectivityconstraintscontinuously, caterini2021variationalinferencecontinuouslyindexednormalizing}, the terminology of which we borrow for describing our framework.
Correspondingly, recent works have proposed partitioning (binning) the continuous component of the state space of both GFlowNets \citep{zhou2024phylogfn} and diffusion probabilistic models \citep{gupta2026divideanddenoise}, the outcome of which is a coarse mixture approximation to a continuous density. 
\looseness=-1









\end{document}